\documentclass[11pt,reqno,a4paper]{article}
\usepackage{fullpage}

\usepackage[utf8]{inputenc}
\usepackage[T1]{fontenc}
\usepackage{booktabs}
\usepackage{amsfonts}
\usepackage{nicefrac}
\usepackage[protrusion=true,expansion=false]{microtype}
\usepackage{xcolor}
\usepackage{bm}
\usepackage{amsmath,amssymb,amsbsy,amsfonts,amsthm}

\newtheorem{theorem}{Theorem}
\newtheorem{lemma}[theorem]{Lemma}

\newtheorem{proposition}[theorem]{Proposition}

\usepackage{subcaption}
\usepackage{multirow}
\usepackage{makecell}
\usepackage{graphicx}
\usepackage{mathtools}
\usepackage{environ}
\usepackage[round]{natbib}

\usepackage{tablefootnote}
\usepackage{enumitem}
\usepackage{float}
\usepackage{afterpage}
\usepackage{placeins}

\usepackage{hyperref}

\DeclareFontFamily{U}{BOONDOX-cal}{\skewchar\font=45 }
\DeclareFontShape{U}{BOONDOX-cal}{m}{n}{<-> s*[1.05] BOONDOX-r-cal}{}
\DeclareFontShape{U}{BOONDOX-cal}{b}{n}{<-> s*[1.05] BOONDOX-b-cal}{}
\DeclareMathAlphabet{\lcal}{U}{BOONDOX-cal}{m}{n}
\SetMathAlphabet{\lcal}{bold}{U}{BOONDOX-cal}{b}{n}

\usepackage{amsmath,amsfonts,bm}

\def\1{\bm{1}}

\def\va{{\bm{a}}}

\def\vh{{\bm{h}}}

\def\vz{{\bm{z}}}

\def\mA{{\bm{A}}}

\def\mD{{\bm{D}}}

\def\mG{{\bm{G}}}

\def\mI{{\bm{I}}}

\def\mL{{\bm{L}}}

\def\mP{{\bm{P}}}
\def\mQ{{\bm{Q}}}

\def\mS{{\bm{S}}}

\def\mV{{\bm{V}}}
\def\mW{{\bm{W}}}
\def\mX{{\bm{X}}}
\def\mY{{\bm{Y}}}
\def\mZ{{\bm{Z}}}

\DeclareMathAlphabet{\mathsfit}{\encodingdefault}{\sfdefault}{m}{sl}
\SetMathAlphabet{\mathsfit}{bold}{\encodingdefault}{\sfdefault}{bx}{n}
\newcommand{\tens}[1]{\bm{\mathsfit{#1}}}

\def\tH{{\tens{H}}}

\def\tL{{\tens{L}}}

\def\tW{{\tens{W}}}

\def\gA{{\mathcal{A}}}

\def\gV{{\mathcal{V}}}

\def\gX{{\mathcal{X}}}
\def\gY{{\mathcal{Y}}}
\def\gZ{{\mathcal{Z}}}

\usepackage{xparse}

\newcommand{\stepjust}[1]{\tag*{\footnotesize\textit{(by #1)}}}

\newcommand{\eq}[1]{Eq.~\eqref{#1}}
\newcommand{\eqs}{Eqs.~}

\NewDocumentCommand{\mylt}{ O{\ell} O{t} }{%
    ^{\left(#1, #2\right)}%
}

\NewDocumentCommand{\myl}{ O{\ell} }{%
    ^{\left(#1\right)}%
}

\NewDocumentCommand{\myt}{ O{t} }{%
    ^{\left(#1\right)}%
}

\NewDocumentCommand{\mylm}{ O{\ell} O{m} }{%
    ^{\left[#1, #2\right]}%
}
\newcommand{\VarGNDE}{\mX}
\newcommand{\VarGNDEHiddenState}{\mA}
\newcommand{\VarGNDEParameter}{\va}

\newcommand{\VarGraphonNDE}{\gX}
\newcommand{\VarGraphonNDEIV}{\gZ} %
\newcommand{\VarGraphonNDEHiddenState}{\gA}

\newcommand{\VarGraphonNDEParameter}{\lcal{a}}
\newcommand{\VarDTOGradient}{\mG}

\newcommand{\hatVarGraphonNDE}{\widehat{\VarGraphonNDE}}
\newcommand{\tildeVarGraphonNDE}{\widetilde{\VarGraphonNDE}}

\newcommand{\hatVarGNDE}{\widehat{\VarGNDE}}
\newcommand{\tildeVarGNDE}{\widetilde{\VarGNDE}}

\newcommand{\GNN}{\mathrm{GNN}}
\newcommand{\GraphonNN}{\mathrm{WNN}}

\newcommand{\spaceB}{B(I;\mathbb{R}^{1\times F})}
\newcommand{\spaceb}{B(I)}
\newcommand{\spaceC}{C([0,T];\spaceB)}
\newcommand{\spaceltwo}{L^2(I;dP)}
\newcommand{\spaceLTWO}{L^2(I;\mathbb{R}^{1\times F};dP)}
\newcommand{\spaceRnF}{\mathbb{R}^{n\times F}}
\newcommand{\spaceRn}{\mathbb{R}^{n}}
\newcommand{\spaceRFFK}{\mathbb{R}^{F\times F\times K}}

\newcommand{\matL}{\mL}
\newcommand{\matLUn}{\mL_{U_n}}
\newcommand{\opLP}{\mathcal{L}_P}
\newcommand{\opLUn}{\mathcal{L}_{U_n}}

\newcommand{\norm}[1]{\left\lVert #1 \right\rVert}
\newcommand{\normF}[1]{\left\lVert #1 \right\rVert_{\mathrm{F}}}
\newcommand{\normb}[1]{\left\lVert #1 \right\rVert_{\spaceb}}
\newcommand{\normB}[1]{\left\lVert #1 \right\rVert_{\spaceB}}
\newcommand{\normC}[1]{\left\lVert #1 \right\rVert_{\spaceC}}
\newcommand{\normTwo}[1]{\left\lVert #1 \right\rVert_{2}}
\newcommand{\normMax}[1]{\left\lVert #1 \right\rVert_{\max}}

\newcommand{\iprod}[1]{\left\langle #1 \right\rangle}
\newcommand{\braket}[1]{\left[ #1 \right]}
\newcommand{\curlbraket}[1]{\left\{ #1 \right\}}
\newcommand{\abs}[1]{\left| #1 \right|}
\newcommand{\parens}[1]{\left( #1 \right)}

\newcommand{\Or}[1]{\Big(\text{Or}\ #1 \Big)}

\newcommand{\SamplingOperator}{\mathcal{S}_{U_n}}

\newcommand{\MSE}{\mathrm{MSE}_{U_n}}

\NewDocumentCommand{\hltfgk}{ O{t} }{%
\ensuremath{\vh_{fgk}^{(\ell,#1)}}%
}

\newcommand{\ConvFilterLipschitz}{\hyperref[AS0]{AS0}}
\newcommand{\ConvFilterDiff}{\hyperref[AS0diff]{AS0$'$}}
\newcommand{\SigmaLipschitz}{\hyperref[AS1]{AS1}}
\newcommand{\SigmaDiff}{\hyperref[AS1diff]{AS1$'$}}
\newcommand{\SigmaDerivativeLipschitz}{\hyperref[AS1diffLip]{AS1$''$}}
\newcommand{\SigmaTwiceDiff}{\hyperref[AS1diff_diff]{AS1$'''$}}
\newcommand{\dPBoundedBelow}{\hyperref[AS2]{AS2}}
\newcommand{\GraphonLipschitz}{\hyperref[AS3]{AS3}}

\newcommand{\dt}{\frac{d}{dt}}
\newcommand{\ddt}{\frac{d^2}{dt^2}}
\newcommand{\supT}{\sup_{t\in[0,T]}}

\newcommand{\QX}{\mathcal{Q}_{\VarGNDE}}
\newcommand{\QY}{\mathcal{Q}_{\VarGNDEHiddenState}}
\newcommand{\Qh}{\mathcal{Q}_h}

\newcommand{\DX}{\mathcal{D}_{X}}
\newcommand{\Dh}{\mathcal{D}_{h}}

\newcommand{\gd}[2]{\frac{\delta #1}{\delta #2}}

\newcommand{\Liph}{L_h}

\newcommand{\Ch}{C_h}
\newcommand{\Ce}{C_{err}}
\newcommand{\tildeCe}{\widetilde{C}_{err}}

\NewDocumentCommand{\Csys}{ O{} }{
\ensuremath{C_{#1,W}}
}

\newcommand{\CXmax}{C_{\VarGraphonNDE,\max}}
\newcommand{\CYmax}{C_{\VarGraphonNDEHiddenState,\max}}
\newcommand{\CXlip}{C_{\VarGraphonNDE,\mathrm{Lip}}}
\newcommand{\CYlip}{C_{\VarGraphonNDEHiddenState,\mathrm{Lip}}}
\newcommand{\Codotmax}{C_{\odot,\max}}
\newcommand{\Codotlip}{C_{\odot,\mathrm{Lip}}}

\newcommand{\CmXGNDE}{C_\VarGNDE}
\newcommand{\CmXDeriGNDE}{C_{\VarGNDE'}}
\newcommand{\CmXTildeDeriGNDE}{C_{\widetilde{\VarGNDE}'}}
\newcommand{\CmXHatDeriGNDE}{C_{\widehat{\VarGNDE}'}}
\newcommand{\CXdiff}{C_{\Delta \VarGNDE}}
\newcommand{\CXdiffHat}{C_{\Delta \widehat{\VarGNDE}}}

\newcommand{\CmYGNDEmax}{C_{\VarGNDEHiddenState}^{\max}(\alpha_n,\gamma_1,\gamma_2,\gamma_3)}
\newcommand{\CmYGNDE}{C_\VarGNDEHiddenState}
\newcommand{\CmYDeriGNDE}{C_{\VarGNDEHiddenState'}(\alpha_n,\gamma_1,\gamma_2,\gamma_3)}

\newcommand{\CYdiff}{C_{\Delta \VarGNDEHiddenState}(\alpha_n,\gamma_1,\gamma_2,\gamma_3)}

\NewDocumentCommand{\CGdiff}{o}{%
  C_{\Delta \VarDTOGradient}%
  \IfValueTF{#1}{^{(#1)}}{}%
  (\alpha_n, \gamma_1, \gamma_2, \gamma_3)%
}

\newcommand{\CmYDistGNDE}{C_{\VarGNDEHiddenState^{[]}}}
\newcommand{\CmXDistGNDE}{C_{\VarGNDE^{[]}}}

\newcommand{\CXdiffDist}{C_{\Delta \VarGNDE^{[]}}}

\newcommand{\CXdiffDistHat}{C_{\Delta \widehat{\VarGNDE}^{[]}}}

\newcommand{\loss}{\mathrm{Loss}}

\title{\bf Zero-Shot Size Transfer for Neural ODEs on Sparse Random Graphs: Graphon Limits and Adjoint Convergence}
\author{Mingsong Yan\thanks{Department of Mathematics, University of California, Santa Barbara, CA. (\textit{mingsongyan@ucsb.edu})}, \quad
Zhida Wang\thanks{Department of Mathematics, University of California, Santa Barbara, CA. (\textit{zhida@ucsb.edu})},\ \ \ and \ \ 
Sui Tang\thanks{Department of Mathematics, University of California, Santa Barbara, CA. (\textit{suitang@ucsb.edu})}}
\date{}

\begin{document}

\maketitle
\begin{abstract}
Graph Neural Differential Equations (GNDEs) model continuous-time graph dynamics by parameterizing Neural ODE velocity fields with Graph Neural Networks. Their local, size-independent filters suggest a \emph{zero-shot size-transfer} principle: train on a small graph and deploy on larger, similar graphs without retraining. We develop a quantitative theory for this principle on sparse random graphs sampled from graphons. We consider Graphon Neural Differential Equations (Graphon-NDEs) and adjoint Graphon-NDEs as the infinite-node limits of the forward and adjoint GNDE systems, and establish well-posedness. For an $n$-node random graph with sparsity parameter $\alpha_n$, we prove \emph{trajectory-wise} convergence of GNDE solutions to Graphon-NDE solutions at rate $\mathcal{O}((\alpha_n n)^{-1/2})$, up to logarithmic factors, with high probability. We also establish \emph{uniform-in-time} convergence bounds for adjoint systems governing hidden-state and parameter gradients. We further study discretize-then-optimize (DTO) and optimize-then-discretize (OTD) training. Under explicit Euler discretization with $M$ steps, we show that DTO and OTD are asymptotically consistent, with hidden-state and local parameter-gradient discrepancies of orders $\mathcal{O}(1/M)$ and $\mathcal{O}(1/M^2)$, respectively, up to sparsity and logarithmic factors. Experiments on HSBM and tent graphons support the theoretical rates, while zero-shot transfer experiments across four graphon classes demonstrate accurate deployment of learned GNDEs on larger independently sampled graphs.
\end{abstract}

\section{Introduction}\label{sec: introduction}

Neural Ordinary Differential Equations (Neural ODEs)~\citep{chen2018neural}
are continuous-depth neural network models in which the hidden state evolves
according to an ODE whose velocity field is parameterized by a neural network.
They can be viewed as continuous-depth analogues of deep residual networks
\citep{he2016deep}, where discrete residual updates are replaced by continuous
ODE flows. This continuous-depth formulation provides a flexible framework for modeling
temporal dynamics, with applications in generative modeling, time-series
analysis, and scientific computing~\citep{avelin2021neural,sander2022residual}. Neural ODEs are commonly trained through either the
\emph{discretize-then-optimize} (DTO) or \emph{optimize-then-discretize} (OTD)
paradigm. Both approaches require numerical ODE solves during training and can
become computationally expensive in high-dimensional settings
\citep{chen2018neural,onken2020discretize}.

Many dynamical systems of interest, such as social, biological, and physical
interaction networks, inherently have graph structure, where each node's evolution relies on the states of its neighbors. Graph Neural Differential Equations (GNDEs)~\citep{poli2019graph,liu2025graph} encode this relational structure by parameterizing the velocity field with a Graph Neural Network (GNN)~\citep{scarselli2008graph, kipf2016semi}, thereby combining the continuous-depth flexibility of Neural ODEs with graph-based message passing. GNDEs have been applied to diverse tasks, including node classification, traffic forecasting, epidemic modeling, and physical simulation~\citep{chamberlain2021grand,rusch2022graph,fang2021spatial,huang2024causal,luo2023hope,huang2023ggode,choi2022gread}. Despite their versatility, training GNDEs can be expensive: on large graphs, each
ODE solve requires repeated message passing over the entire graph, which can
scale poorly with network size~\citep{finzi2023stable,liu2025graph}. A natural
remedy is \emph{zero-shot size transfer}: train on a small source graph and reuse the learned parameters, without retraining, on larger, structurally similar target graphs.

Such transfer is plausible because graph dynamical systems often exhibit a
natural size transferability of their own: they are governed by \emph{local}
rules---each node evolves from its own state and those of its immediate
neighbors, with the same rule applied at every node regardless of the graph's
size. A heat, epidemic, or reaction--diffusion process, for instance, follows
the same node-level rule on a graph with one hundred nodes as on a graph with
ten thousand; only which nodes are connected changes. The graphon
formalism~\citep{lovasz2012large} makes this scale-invariance precise: it
represents graphs of different sizes as samples of a common continuum kernel, so
one can ask whether the finite-graph dynamics converge to a limiting
\emph{graphon dynamical system}. This has been established for nonlinear heat
equations on dense deterministic graph sequences~\citep{medvedev2014nonlinear},
$\tW$-random graphs~\citep{medvedev2019nonlinear}, sparse sampling
regimes~\citep{kaliuzhnyi2022sparse}, collective dynamics with time-varying
coupling~\citep{ayi2021mean}, and interacting particle
systems~\citep{bayraktar2023graphon}, where the governing rule lives on the
limiting graphon rather than on any finite graph.

As flexible data-driven surrogates for the classical graph dynamical systems~\citep{poli2019graph,berndt2025permutation,chamberlain2021grand,choi2022gread},
GNDEs raise a natural question: does this graphon-based size transfer extend
from classical, prescribed dynamics to learned GNDE surrogates? If the GNDE and
the target dynamics are compatible with the same graphon, parameters
$\hat{\theta}$ learned on a small source graph $\mathcal{G}_{\mathrm{small}}$
should remain accurate when the model is evaluated, without retraining, on a
larger target graph $\mathcal{G}$ from the same family. To organize the
analysis, we apply the triangle-inequality decomposition from our companion
work~\citep{yan2025convergence}:
\begin{equation}\label{eq:triangle_decomposition}
\bigl\| \Phi_{\mathcal{G}}(\hat{\theta}) - S_{\mathcal{G}} \bigr\|
\leq
\underbrace{\bigl\| \Phi_{\mathcal{G}}(\hat{\theta}) - \Phi_{\mathcal{G}_{\mathrm{small}}}(\hat{\theta}) \bigr\|}_{\text{(I) model transfer error}}
+ \underbrace{\bigl\| \Phi_{\mathcal{G}_{\mathrm{small}}}(\hat{\theta}) - S_{\mathcal{G}_{\mathrm{small}}} \bigr\|}_{\text{(II) training error}}
+ \underbrace{\bigl\| S_{\mathcal{G}_{\mathrm{small}}} - S_{\mathcal{G}} \bigr\|}_{\text{(III) graph discretization error}}.
\end{equation}
Here $S_{\mathcal{G}}$ is the exact solution of the target system on
$\mathcal{G}$, $\Phi_{\mathcal{G}}(\hat{\theta})$ is the GNDE surrogate
evaluated on $\mathcal{G}$, and $\hat{\theta}$ is trained on the source graph
$\mathcal{G}_{\mathrm{small}}$. Term~(II) is the source-graph training error,
and Term~(III) is the size-transfer error of the classical graph dynamics,
precisely the error controlled by existing graphon-limit
results~\citep{medvedev2014nonlinear,medvedev2019nonlinear}. This work focuses
on Term~(I), the size-transfer error of the learned GNDE surrogate, together
with its gradient analogue, under random and sparse graph sampling.

\paragraph{Contributions.} 

Building on~\citet{yan2025convergence}, which established trajectory-wise
convergence for \emph{deterministic} graph sequences, we study the size
transferability of GNDEs on \emph{random sparse} graphs. Our random-graph
analysis adopts the sparse sampling setting in
\citet{keriven2020convergence}, developed for size transferability of GNNs. In
this setting, we establish quantitative convergence bounds for finite-graph
GNDEs and their adjoint systems toward their Graphon-NDE limits, and further
analyze the temporal discretization and DTO--OTD consistency underlying
practical GNDE training. The resulting bounds make explicit how transfer and
discretization depend on graph sparsity and architectural quantities, including
activation smoothness and polynomial filter norms. In particular, convergence results for forward trajectories, hidden-state
gradients, and parameter gradients require progressively stronger activation
regularity, in line with recent observations on activation smoothness for Neural
ODEs~\citep{gao2025global}. We itemize our contributions as follows.
\begin{itemize}
    \item \textbf{Graphon limits.} Based on the analysis of \citet{yan2025convergence}, we consider Graphon Neural Differential Equations (Graphon-NDEs) as the infinite-node limit of finite-graph GNDEs and adapt the corresponding well-posedness result under suitable regularity assumptions (Theorem~\ref{theorem: well-posedness}, Section~\ref{sec: Graphon-NDE}). We further derive the infinite-node adjoint system, which serves as the graphon limit of the adjoint GNDEs for hidden-state and parameter gradients, and prove its well-posedness (Theorem~\ref{theorem: Well-posedness of Adjoint System}, Section~\ref{sec: Graphon-NDE}).

    \item \textbf{Forward and adjoint convergence.} We establish a forward convergence rate for GNDE solutions on sparse random
graphs, showing that they converge with high probability to the Graphon-NDE
solution uniformly in time
(Theorem~\ref{theorem: GNDE -> Graphon-NDE},
Section~\ref{sec: GNDE to Graphon-NDE}). In addition, we provide convergence
bounds for the corresponding adjoint systems governing hidden-state and
parameter gradients
(Theorems~\ref{theorem: MSE Y Yn adjoint equation final}--\ref{theorem: z - zn leq TQh},
Section~\ref{sec: GNDE to Graphon-NDE}). Together, these results establish the stability of GNDEs under graph-size
transfer at both the forward-prediction and adjoint-gradient levels, providing
quantitative control of the transfer error of GNDEs (i.e., Term~(I)
in~\eqref{eq:triangle_decomposition}), and its gradient analogue under random
and sparse graph sampling.

    \item \textbf{Temporal discretization and DTO--OTD consistency.}
We establish temporal discretization error bounds showing that residual GNNs
obtained by Euler discretization converge to their continuous-time GNDE limits
for forward trajectories, hidden-state adjoints, and parameter gradients
(Theorems~\ref{theorem: discretized gnde to gnde}--\ref{theorem: discretized adjoint gnde to adjoint gnde parameters},
Section~\ref{sec: ResGNN to GNDE}). Building on these estimates, we prove that
the DTO and OTD training paradigms produce asymptotically consistent gradients
under explicit Euler discretization: the hidden-state gradient discrepancy
vanishes at order $\mathcal{O}(1/M)$, while the local parameter-gradient
discrepancy vanishes at order $\mathcal{O}(1/M^2)$, up to sparsity and logarithmic factors (Theorems~\ref{theorem: hidden states gradients DTO and OTD}--\ref{theorem: parameter gradients DTO and OTD},
Section~\ref{sec: DTO vs OTD}). Our quantitative DTO--OTD consistency rates for continuous-depth graph neural networks complement prior analyses of
DTO and OTD training for Neural ODEs~\citep{onken2020discretize}. In particular,
the results show that, under explicit Euler discretization, the gradients
obtained by differentiating through the discrete solver are asymptotically
consistent with those obtained from the continuous adjoint as $M\to\infty$.

    \item \textbf{Numerical experiments.} We validate the theoretical convergence rates on the hierarchical stochastic block model (HSBM) and tent graphons in Section~\ref{sec: validation of theory}, including graph-size convergence of the forward trajectory and adjoint gradients, temporal discretization error, and DTO--OTD gradient discrepancy. We then test the size transferability of GNDEs as learned surrogates in Section~\ref{subsec: exp transfer}: a GNDE trained on a small source graph is deployed zero-shot on independently sampled larger target
graphs. The experiments cover four systems with qualitatively distinct
dynamical behaviors: hierarchical smoothing (linear heat on HSBM graphon), traveling wavefronts (Fisher--KPP on a circular threshold graphon), degree-stratified endemic states (SIS on a rank-1 power-law graphon), and consensus contraction (bounded-confidence consensus on the tent graphon). Across all four systems, the trained GNDEs accurately fit the source dynamics and achieve small transfer error on independently sampled larger graphs.
\end{itemize}

\paragraph{Related work.}

\emph{1) Size transferability of GNNs.} Size transferability of discrete-layer GNNs has been studied across related
graph-limit settings~\citep{ruiz2020graphon,levie2021transferability,
maskey2023transferability,keriven2020convergence,le2024limits,herbst2025higher}. These works provide \emph{layer-wise} guarantees: each hidden layer's output converges as graph size grows. Our forward convergence result (Theorem~\ref{theorem: GNDE -> Graphon-NDE}) generalizes~\citet{keriven2020convergence} from discrete-layer GNNs to continuous-depth GNDEs, where a stronger notion of \emph{trajectory-wise} convergence is required: the
entire solution path must converge uniformly in time, rather than only the outputs of finitely many layers. A key technical step is to extend the spatial chaining argument of~\citet{keriven2020convergence} to a joint node--time domain, which yields \emph{uniform-in-time} control of the random-graph error along the continuous GNDE trajectory. We additionally establish adjoint convergence and quantify the DTO--OTD discrepancy, which do not arise in the discrete-layer setting and are essential for
understanding the transferability of gradient information in training.

\emph{2) Graphon limits of graph dynamical systems.}
A related line of work studies continuum limits of prescribed dynamics on graph
sequences, including nonlinear heat equations, collective dynamics, interacting
particle systems, and microscopic-to-macroscopic limits
\citep{medvedev2014nonlinear,medvedev2019nonlinear,kaliuzhnyi2022sparse,
ayi2021mean,bayraktar2023graphon,paul2022microscopic}. Our analysis shares
standard tools with this literature, such as Banach fixed-point arguments for
well-posedness and Gr\"{o}nwall-type estimates for stability and convergence.
The main difference is that classical graph dynamical systems are driven by
fixed or prescribed interaction rules, whereas the vector field of our GNDE
surrogate is learned through a GNN parameterization. Consequently, the
convergence bounds must track architectural quantities such as activation
regularity, convolutional filter norms, depth, and width. Moreover, while
classical graphon-limit results mainly concern forward trajectories, we also analyze the associated adjoint systems, which are needed to understand gradient computation and gradient-level stability under graph-size transfer.

\emph{3) Discretization and adjoints for Neural ODEs.}
The numerical treatment of adjoints is a central issue in Neural ODE training.
The original Neural ODE framework computes gradients using the adjoint
sensitivity method~\citep{chen2018neural}. Subsequent work has shown that
discretization choices can substantially affect gradient accuracy.
\citet{gholami2019anode} highlighted numerical issues in adjoint-based training
and proposed ANODE to improve gradient accuracy with controlled memory cost.
\citet{onken2020discretize} compared DTO and OTD training strategies for Neural
ODEs in non-graph settings. \citet{xu2023correcting} showed that certain
higher-order discretizations can produce gradient oscillations under automatic
differentiation. Recent work further highlights the role of activation
smoothness in forward and backward Neural ODE analysis~\citep{gao2025global}.
Our work studies these issues for GNDEs: under explicit Euler discretization, we
prove quantitative DTO--OTD consistency rates showing that gradients obtained by backpropagating through the Euler-discretized GNDE asymptotically agree with those obtained by first deriving the continuous adjoint equations and then discretizing them.

\paragraph{Organization.}
Section~\ref{sec: preliminaries} introduces the notation and preliminaries,
including graphons, function spaces, and G\^ateaux derivatives.
Section~\ref{sec: GNDE} formulates spectral GNNs, GNDEs, and Graphon-NNs.
Section~\ref{sec: training of GNDEs} reviews the training paradigms for GNDEs, i.e., DTO and OTD.
Section~\ref{sec: Graphon-NDE} introduces Graphon-NDEs and adjoint Graphon-NDEs
and establishes their well-posedness. Section~\ref{sec: GNDE to Graphon-NDE}
proves uniform-in-time convergence of the forward trajectories and of the
backward gradients with respect to hidden states and parameters.
Section~\ref{sec: ResGNN to GNDE} analyzes temporal discretization errors for
forward trajectories, hidden-state gradients, and parameter gradients.
Section~\ref{sec: DTO vs OTD} establishes the asymptotic consistency of DTO and OTD under explicit Euler discretization. Section~\ref{sec: experiments}
presents numerical experiments.

\section{Notation and Preliminaries}\label{sec: preliminaries}

\subsection{Basic notations}
Let $\mathbb{N}:=\{1,2,\ldots\}$ and $\mathbb{R}^+:=[0,\infty)$. For $n\in\mathbb{N}$, let $[n]:=\{1,2,\ldots,n\}$ and $\mathbb{Z}_n:=\{0,1,\ldots,n-1\}$. For vectors, $\|\cdot\|_2$ denotes the Euclidean norm. For matrices, $\|\cdot\|_2$, $\|\cdot\|_{\mathrm{F}}$, and $\|\cdot\|_{\max}$ denote the spectral, Frobenius, and max (entry-wise) norms, respectively. For tensors, $\|\cdot\|_{\max}$ similarly denotes the maximum absolute entry. We write $I := [0, 1]$ for the unit interval and $I^2 := I \times I$. We write \(A\lesssim B\) if there exists an absolute constant \(c>0\) such that
\(A\leq cB\).

\subsection{Graphs, graphons, and graph features}\label{subsec: graphs and graphons}

A graph of $n$ nodes is denoted by $\mathcal{G} = \langle V, E, \mW \rangle$, where $V$ is the set of $n$ nodes, $E \subseteq V \times V$ is the set of edges, and $\mW = [\mW_{ij} \in I : i, j \in [n]]$ is the symmetric adjacency matrix with $\mW_{ij} = \mW_{ji} \neq 0$ if and only if $(i,j) \in E$. The degree matrix $\mD$ is the diagonal matrix with entries $\mD_{ii} := \sum_j \mW_{ij}$, and the \emph{symmetric normalized adjacency} is $\matL := \mD^{-1/2}\mW\mD^{-1/2}$, with the convention $[\mD^{-1/2}]_{ii} := 0$ when $\mD_{ii} = 0$. (The associated normalized Laplacian is $\mI - \matL$; we work with $\matL$ directly because the spectral filters introduced below are polynomials in $\matL$.) We denote the graph node feature matrix as $\mZ\in \mathbb{R}^{n\times F}$, where $F$ is the number of features.

A \emph{graphon} is a bounded, symmetric, measurable function $\tW: I^2 \to I$. Graphons serve both as limit objects of convergent graph sequences in the sense of homomorphism density~\citep{lovasz2012large} and as generative models for random graphs. Given a graphon $\tW$, a probability distribution $P$ on $I$ and sparsity parameter $\alpha_n\in(0,1]$, the \emph{$\tW$-random graph} $\mathcal{G}$ of $n$ nodes is generated as follows: nodes $u_i\in I$, $i\in[n]$, are independently sampled from the distribution $P$;
conditional on these nodes, edges are sampled independently from a Bernoulli distribution, i.e., nodes $i$ and $j$ are connected with probability
$\alpha_n \tW(u_i,u_j)$ for $i<j$. A $\tW$-random graph provides a general model for random graphs, including important examples such as Erdős–Rényi graphs and stochastic block models.

\subsection{Function spaces}\label{subsec: function spaces}

We denote by $\spaceb$ the Banach space of bounded scalar-valued functions on $I$ with norm $\|Z\|_{\spaceb} := \sup_{u \in I} |Z(u)|$, and by $\spaceB$ the space of bounded vector-valued functions $\mathcal{Z} = [\mathcal{Z}_f : f \in [F]]: I \to \mathbb{R}^{1 \times F}$ with norm $\|\mathcal{Z}\|_{\spaceB} := (\sum_{f \in [F]} \|\mathcal{Z}_f\|_{\spaceb}^2)^{1/2}$. Given a probability measure $P$ on $I$, we define $\spaceltwo$ as the space of square-integrable scalar functions with norm defined by $\|Z\|_{\spaceltwo} := (\int_I |Z(u)|^2\, dP(u))^{1/2}$ and inner product $\langle Z, \widetilde{Z} \rangle_{\spaceltwo} := \int_I Z(u)\widetilde{Z}(u)\, dP(u)$. The vector-valued counterpart is $\spaceLTWO$ with norm $\|\mathcal{Z}\|_{\spaceLTWO} := (\sum_{f \in [F]} \|\mathcal{Z}_f\|_{\spaceltwo}^2)^{1/2}$. 

For an interval $\Omega\subseteq \mathbb{R}^+$, we write $C(\Omega;\spaceB)$ for the space of continuous vector-valued functions $\mathcal{X}=[\mathcal{X}_f:f\in[F]]: I \times \Omega \to \mathbb{R}^{1 \times F}$ satisfying that for each $t\in\Omega$, $\mathcal{X}(\cdot, t) \in \spaceB$; for each $f\in[F]$ and $u\in I$, $\mathcal{X}_f(u,\cdot)$ is continuous on $\Omega$. The norm is defined by $\|\mathcal{X}\|_{C(\Omega;\spaceB)} := \sup_{t\in\Omega} \|\mathcal{X}(\cdot,t)\|_{\spaceB}$. By $C^1(\Omega;\spaceB)$ we denote a subspace of $C(\Omega;\spaceB)$, in which the vector-valued function $\mathcal{X}$ additionally satisfies that for each $f\in[F]$ and $u\in I$, $\mathcal{X}_f(u,\cdot)$ is continuously differentiable. 

\subsection{G\^ateaux derivatives}\label{subsec: gateaux}

Let $\mathcal{H}_1$ and $\mathcal{H}_2$ be Hilbert spaces. An operator $\mathcal{F}: \mathcal{H}_1 \to \mathcal{H}_2$ is G\^ateaux differentiable at $u \in \mathcal{H}_1$ if the limit
\[
\gd{\mathcal{F}}{u}(h) := \lim_{\epsilon \to 0} \frac{\mathcal{F}(u + \epsilon h) - \mathcal{F}(u)}{\epsilon},
\]
exists for all $h \in \mathcal{H}_1$ and $h \mapsto \gd{\mathcal{F}}{u}(h)$ is a bounded linear operator \citep{ekeland1999convex}. Its Hilbert-space adjoint $\big(\gd{\mathcal{F}}{u}\big)^*: \mathcal{H}_2 \to \mathcal{H}_1$ satisfies
\begin{equation*}%
\iprod{\gd{\mathcal{F}}{u}(h),\, v}_{\mathcal{H}_2} = \iprod{h,\, \parens{\gd{\mathcal{F}}{u}}^*(v)}_{\mathcal{H}_1}, \quad \forall\, h \in \mathcal{H}_1,\ v \in \mathcal{H}_2.
\end{equation*}
When $\mathcal{H}_2 = \mathbb{R}$, the operator $\gd{\mathcal{F}}{u}$ is a continuous linear functional on $\mathcal{H}_1$, and the Riesz representation theorem gives
$%
\gd{\mathcal{F}}{u}(h) = \langle \nabla_u \mathcal{F},\, h \rangle_{\mathcal{H}_1}$, for all $h\in\mathcal{H}_1$, where $\nabla_u \mathcal{F} \in \mathcal{H}_1$ is the gradient of $\mathcal{F}$ at $u$. For a composition $\mathcal{F} \circ \widetilde{\mathcal{F}}$, it follows from the chain rule that
\begin{equation}\label{eq: chain rule gradient}
\nabla_u(\mathcal{F} \circ \widetilde{\mathcal{F}}) = \parens{\gd{\widetilde{\mathcal{F}}}{u}}^*(\nabla_{\widetilde{u}} \mathcal{F}),
\end{equation}
where $\widetilde{u} := \widetilde{\mathcal{F}}(u)$. This identity is the foundation of gradient backpropagation in both discrete and continuous architectures.

\section{Graph Neural Differential Equations}\label{sec: GNDE}

We define the architecture of spectral Graph Neural Networks (GNNs) and their continuous-depth extension, the Graph Neural Differential Equations (GNDEs).

\subsection{Spectral graph convolutions}\label{subsec: spectral gcn}

Graph neural networks operate by alternating graph convolution and pointwise nonlinear activation. We focus on \emph{spectral} (polynomial filter) convolutions. Given a graph $\mathcal{G}$ with symmetric normalized adjacency $\matL$, the graph convolution is the linear operator $\phi: \spaceRFFK \times \spaceRnF \to \spaceRnF$ defined by
\begin{equation}\label{def: operator phi}
\phi(\vh, \VarGNDE) := \left[\sum_{g=1}^{F} \sum_{k=0}^{K-1} \vh_{fgk}\, \matL^k \VarGNDE_g : f \in [F]\right],\quad  \vh \in \spaceRFFK,\VarGNDE \in \spaceRnF,
\end{equation}
which aggregates information from the node itself and its neighbors up to $K-1$ hops away. The detailed properties of the graph convolutional operators can be found in Appendix~\ref{appendix: graph convolutional operators}. The output of the $\ell$-th GNN layer is
\begin{equation}\label{def: updating formula gnn}
\tildeVarGNDE\myl := \phi(\vh\myl, \VarGNDE\myl[\ell-1]), \qquad \VarGNDE\myl := \sigma(\tildeVarGNDE\myl),
\end{equation}
where $\sigma$ is an element-wise activation function. A GNN with $L$ layers is compactly written as
\begin{equation}\label{def: GNN multi layer}
\VarGNDE^{(L)} := \GNN(\VarGNDE^{(0)}; \matL, \tH),
\end{equation}
where $\tH = [\vh_{fgk}\myl : f, g \in [F],\, k \in \mathbb{Z}_K,\, \ell \in [L]]$ collects all trainable filter coefficients.

\subsection{Formulation of GNDEs}

In GNDEs, the GNN output serves as the velocity field of an ODE, i.e.,
\begin{align}\label{GNDE}
\begin{split}
\dt \VarGNDE(t) &= \GNN\big(\VarGNDE(t);\matL,\tH(t)\big), \\
\VarGNDE(0) &= \mZ \in \spaceRnF,
\end{split}
\end{align}
where $\tH(t)$ denotes the collection of \emph{time-varying} trainable parameters
\begin{equation}\label{parameterized H(t) filters}
\tH(t) := \left\{\vh_{fgk}\mylt : f, g \in [F],\ k \in \mathbb{Z}_K,\ \ell \in [L] \right\}, \quad t \in [0,T].
\end{equation}
Time-varying parameters enable GNDEs to capture nonautonomous dynamics, a capability absent in standard discrete-layer GNNs. If $\GNN(\VarGNDE(t);\matL,\tH(t))$ is Lipschitz continuous in $\VarGNDE$ uniformly over $t \in [0,T]$ and varies continuously in time, then GNDE~\eqref{GNDE} admits a unique solution on $[0,T]$ for every initial feature matrix $\mZ \in \spaceRnF$.

\subsection{Graphon convolutions and Graphon Neural Networks}\label{subsec: graphon nn}

The continuum analogue of the spectral GNN is the \emph{Graphon Neural Network} (Graphon-NN). Given a graphon $\tW$ with probability distribution $P$, define the degree function $d_P(u) := \int_I \tW(u,v)\, dP(v)$ and the symmetric normalized kernel $\tL_P(u,v) := \tW(u,v)/\sqrt{d_P(u)\, d_P(v)}$. The associated integral operator $\opLP: \spaceltwo \to \spaceltwo$ (the graphon analogue of $\matL$) acts as
\[
(\opLP X)(u) := \int_I \tL_P(u,v)\, X(v)\, dP(v),\quad X\in \spaceltwo.
\]
The layer-wise graphon convolution operator $\Phi: \spaceRFFK \times \spaceLTWO \to \spaceLTWO$ is defined by
\begin{equation}\label{def: operator Phi}
\Phi(\vh, \VarGraphonNDE) := \left[\sum_{g=1}^{F} \sum_{k=0}^{K-1} \vh_{fgk} \opLP^k \VarGraphonNDE_g : f \in [F]\right],\quad \vh\in\mathbb{R}^{F\times F\times K},\VarGraphonNDE\in\spaceLTWO, 
\end{equation}
with detailed properties in Appendix~\ref{appendix: graphon convolutional operators}. The output of the $\ell$-th Graphon-NN layer is
\begin{equation}\label{def: updating formula Graphon-NN}
\tildeVarGraphonNDE\myl := \Phi(\vh\myl, \VarGraphonNDE\myl[\ell-1]), \qquad \VarGraphonNDE\myl = \sigma(\tildeVarGraphonNDE\myl),
\end{equation}
and the Graphon-NN with $L$ layers maps input features $\VarGraphonNDE^{(0)}$ to
\begin{equation}\label{def: Graphon-NN multi layer}
\VarGraphonNDE^{(L)} := \GraphonNN(\VarGraphonNDE^{(0)}; \opLP, \tH).
\end{equation}

\section{Training of GNDEs}\label{sec: training of GNDEs}
As in Neural ODEs, GNDEs are used in supervised learning settings where the loss is defined on the terminal state $\VarGNDE(T)$ of the dynamics. Training amounts to finding
time-dependent parameters $\tH(\cdot)$ that minimize the loss, where the
dependence of $\VarGNDE(T)$ on $\tH(\cdot)$ is governed by the GNDE \eqref{GNDE}. Since the parameters $\tH(t)$ influence the loss only through the hidden-state trajectory $\VarGNDE(t)$, the gradient computation decomposes into two stages: 
\begin{enumerate}
\item[\textbf{(S1)}] compute the gradient of the loss with respect to the hidden-state trajectory;
\item[\textbf{(S2)}] use hidden-state gradients to compute the gradient with respect to the parameters.
\end{enumerate}
We describe two paradigms, discretize-then-optimize (DTO) and optimize-then-discretize (OTD), that both follow this structure but differ in the order of discretization and differentiation. Throughout the paper, we use the superscript $[\cdot]$ for quantities on a discrete time grid and $(\cdot)$ for their continuous-time counterparts.

\subsection{Discretize-Then-Optimize (DTO).}
In DTO, one first discretizes the continuous-time GNDE \eqref{GNDE} and then differentiates the resulting discrete objective via backpropagation. We employ the explicit Euler method with step size $\kappa := T/M$ on the uniform time grid $t_m := m\kappa$, $m \in \mathbb{Z}_{M+1}$. The discretized forward dynamics is
\begin{equation}\label{eq: Euler's method for Xn}
\begin{aligned}
\VarGNDE^{[m]} &:= \VarGNDE^{[m-1]} + \kappa\,
\GNN\parens{\VarGNDE^{[m-1]};\matL,\tH(t_{m-1})},\quad m\in[M],\\
\VarGNDE^{[0]}&:=\mZ,
\end{aligned}
\end{equation}
where $\VarGNDE^{[m]}$ denotes the approximation to $\VarGNDE(t_m)$ at the $m$-th time step. The recursion \eqref{eq: Euler's method for Xn} defines a residual GNN with $M$ blocks.
 
\paragraph{(S1) Hidden-state gradients.} Denote by $\VarDTOGradient^{[m]} := \nabla_{\VarGNDE^{[m]}} \loss \in \spaceRnF$ the gradient of the loss with respect to the $m$-th hidden state. We obtain the backward recursions of $\VarDTOGradient^{[m]}$ by differentiating \eqref{eq: Euler's method for Xn} and applying \eqref{eq: chain rule gradient} with respect to $\VarGNDE^{[m-1]}$, which gives
\begin{equation}\label{eq: discretized Gn}
\begin{aligned}
\VarDTOGradient^{[M]}&=\nabla_{\mX(T)}\loss,\\
\VarDTOGradient^{[m-1]} &= \VarDTOGradient^{[m]}
+ \kappa\,\DX^{[m-1]}\parens{\VarDTOGradient^{[m]}},\quad m=M,\ldots,1,
\end{aligned}
\end{equation}
where the adjoint operator $\DX^{[m]}: \spaceRnF \to \spaceRnF$ is defined by 
\begin{equation}\label{def: operator DX[m]}
\DX^{[m]}:=\parens{\frac{\delta \GNN\parens{\VarGNDE^{[m]};\matL,\tH(t_{m})}}
{\delta \VarGNDE^{[m]}}}^*.
\end{equation}
 
\paragraph{(S2) Parameter gradients.} Similarly, differentiating \eqref{eq: Euler's method for Xn} and applying \eqref{eq: chain rule gradient} with respect to $\vh\mylt[\ell][t_{m-1}]$ gives the parameter gradient. Given $\{\VarDTOGradient^{[m]}\}_{m=0}^M$ from \eqref{eq: discretized Gn}, we have
\begin{equation}\label{grad of h dicretize-then-optimize}
\kappa\,\Dh^{[\ell,m]}\parens{\VarDTOGradient^{[m+1]}}
= \text{DTO-Gradient of the loss w.r.t. }\vh\mylt[\ell][t_m],
\end{equation}
where the adjoint operator $\Dh^{[\ell,m]}: \spaceRnF \to \spaceRFFK$ is defined by 
\begin{equation}\label{def: operator Dh[ell,m]}
\Dh^{[\ell,m]}:=\parens{\frac{\delta \GNN\parens{\VarGNDE^{[m]};\matL,\tH(t_{m})}}
{\delta \vh\mylt[\ell][t_m]}}^*. 
\end{equation}
 
DTO computes \emph{exact} gradients of the discretized objective and is straightforward
to implement via automatic differentiation, but it typically requires storing
all intermediate states $\{\VarGNDE^{[m]}\}_{m=0}^M$ during backpropagation.

\subsection{Optimize-Then-Discretize (OTD).}
In OTD, one first derives the gradient equations at the continuous level and then discretizes them numerically \citep{kidger2022neural}.

\paragraph{(S1) Hidden-state gradients.} Define the continuous adjoint state $\VarGNDEHiddenState(t)
:= \nabla_{\VarGNDE(t)}\loss \in \spaceRnF$, the gradient of the loss with respect to the hidden state at time $t$. The GNDE \eqref{GNDE} defines a continuous-time flow $\VarGNDE(t)$ whose infinitesimal change
is governed by the GNN vector field. Applying \eqref{eq: chain rule gradient} to this flow, as in the DTO derivation, yields the following backward-in-time adjoint equation
\begin{equation}\label{adjoint GNDE: Y}
\begin{aligned}
\VarGNDEHiddenState(T) &= \nabla_{\VarGNDE(T)}\loss,\\
\dt \VarGNDEHiddenState(t) &=
-\mathcal{D}_X^{(t)}\parens{
\VarGNDEHiddenState(t)},\quad \text{where}\quad \mathcal{D}_X^{(t)}:=\parens{\frac{\delta \GNN(\VarGNDE(t);\matL,\tH(t))}{\delta \VarGNDE(t)}}^*. 
\end{aligned}
\end{equation}

\paragraph{(S2) Parameter gradients.} Similarly, applying \eqref{eq: chain rule gradient} with respect to $\vh\mylt$, expresses the parameter gradient directly in terms of the hidden-state adjoint as
\begin{equation}\label{adjoint GNDE: z}
\VarGNDEParameter\mylt = 
\mathcal{D}_h\mylt
(\VarGNDEHiddenState(t)),\quad\text{where}\quad \mathcal{D}_h\mylt:=\parens{\frac{\delta \GNN(\VarGNDE(t);\matL,\tH(t))}{\delta \vh\mylt}}^*. 
\end{equation}
In practice, \eqs \eqref{adjoint GNDE: Y} and \eqref{adjoint GNDE: z} must be discretized. Using the explicit Euler method with the same step size $\kappa = T/M$ and time grid $\{t_m\}$ as in DTO, and approximating the continuous operators $\DX^{(t_m)}$ and $\Dh^{(\ell,t_m)}$ by their discrete counterparts $\DX^{[m]}$ and $\Dh^{[\ell,m]}$ yields
\paragraph{(S1, discretized).}
\begin{equation}\label{eq: discretized Yn}
\begin{aligned}
\VarGNDEHiddenState^{[M]}&:=\nabla_{\VarGNDE(T)}\loss,\\
\VarGNDEHiddenState^{[m-1]} &:= \VarGNDEHiddenState^{[m]}
+ \kappa\,\DX^{[m]}\!\big(\VarGNDEHiddenState^{[m]}\big),\quad m=M,\ldots,1.
\end{aligned}
\end{equation}

\paragraph{(S2, discretized).}
\begin{equation}\label{grad of h optimize-then-dicretize}
\begin{split}
\VarGNDEParameter\mylm&:=\mathcal{D}_h^{[\ell,m]}\parens{\VarGNDEHiddenState^{[m]}},\\
\kappa\VarGNDEParameter\mylm &= \text{OTD-Gradient of the loss w.r.t. }\vh\mylt[\ell][t_m].
\end{split}
\end{equation}
OTD can avoid storing the full forward trajectory
$\{\VarGNDE^{[m]}\}_{m=0}^M$ by recomputing or checkpointing the forward states,
but it incurs additional computational cost from backward adjoint integration
and produces gradients that depend on the chosen time discretization.

\subsection{OTD versus DTO.}
The DTO recursion \eqref{eq: discretized Gn} and the OTD recursion \eqref{eq: discretized Yn} share the same algebraic form but evaluate the adjoint operators at different points: DTO at the left endpoint $(\VarGNDE^{[m-1]},t_{m-1})$, OTD at the right endpoint $(\VarGNDE^{[m]},t_{m})$; Table~\ref{tab:DTO_vs_OTD} summarizes the adjoint operators used in DTO and OTD. The explicit forms of these operators can be found in Appendix~\ref{appendix: graph convolutional operators} (specifically, \eq{adjoint operator: GNDE to filters}). Under the backward-in-time Euler discretization with the same step size and time grid, these two recursions become directly comparable, though for higher-order solvers such coincidence generally fails~\citep{onken2020discretize}. Later in the paper, we analyze how these discrepancies behave under graph refinement and in the graphon limit, and establish conditions under which DTO and OTD become asymptotically equivalent.
 
\begin{table}[h]
    \centering
    \renewcommand{\arraystretch}{2.2}
    \resizebox{\textwidth}{!}{%
    \begin{tabular}{l|c|c}
    \toprule
     & \textbf{DTO} & \textbf{OTD} \\
    \midrule

    Hidden states &
    $\displaystyle \DX^{[m]}:=\parens{\frac{\delta \GNN\parens{\VarGNDE^{[m]};\matL,\tH(t_{m})}}{\delta \VarGNDE^{[m]}}}^*$ &
    $\displaystyle \DX\myt:=\parens{\frac{\delta \GNN\parens{\VarGNDE(t);\matL,\tH(t)}}{\delta \VarGNDE(t)}}^*$ \\

    \midrule

    Parameters &
    $\displaystyle \Dh^{[\ell,m]}:=\parens{\frac{\delta \GNN\parens{\VarGNDE^{[m]};\matL,\tH(t_{m})}}{\delta \vh\mylt[\ell][t_m]}}^*$ &
    $\displaystyle \Dh\mylt:=\parens{\frac{\delta \GNN\parens{\VarGNDE(t);\matL,\tH(t)}}{\delta \vh\mylt}}^*$ \\

    \midrule

    Evaluation & Left endpoint: $(\VarGNDE^{[m-1]}, t_{m-1})$ & Right endpoint: $(\VarGNDE^{[m]}, t_m)$ \\

    \bottomrule
    \end{tabular}%
    }
    \vspace{0.3cm}
    \caption{Adjoint operators for DTO and OTD.}
    \label{tab:DTO_vs_OTD}
\end{table}

\section{Graphon Neural Differential Equations}\label{sec: Graphon-NDE}

A central question motivating this work is whether GNDEs exhibit \emph{size transferability}: can a GNDE trained on a moderate-sized graph be deployed on a larger, structurally similar graph without retraining? To formalize this, we need an infinite-node reference object against which finite-graph GNDEs can be compared. In this section, we consider \emph{Graphon Neural Differential Equations} (Graphon-NDEs) as the natural infinite-node limit of GNDEs and establish their well-posedness. 

Recall that a graphon $\tW\in B(I^2)$ serves as a generative model for sequences of structurally similar graphs of increasing size~\citep{lovasz2012large}. Replacing the discrete normalized adjacency matrix $\matL$ in the GNDE \eqref{GNDE} by its graphon analogue $\opLP$, and letting node features become functions $\VarGraphonNDE(u,t)$ for the continuum nodes $u\in I$, we define the Graphon-NDE as
\begin{equation}\label{Graphon-NDE}
    \begin{split}
\frac{\partial}{\partial t} \VarGraphonNDE(u,t) &= \GraphonNN(\VarGraphonNDE(u,t);\opLP,\tH(t)),\\
\VarGraphonNDE(u,0) &= \VarGraphonNDEIV(u) \in \spaceB.
    \end{split}
\end{equation}
The Graphon-NDE \eqref{Graphon-NDE} is a nonlocal integro-differential equation on the graphon space, where the graphon neural network $\GraphonNN$ plays the same architectural role as the GNN in \eqref{GNDE} but operates on function-valued features rather than finite-dimensional node vectors.

For training, the same two-stage gradient structure \textbf{(S1)--(S2)} from Section~\ref{sec: training of GNDEs} carries over to the infinite-node setting. Let $\VarGraphonNDEHiddenState(u,t)$ and $\VarGraphonNDEParameter\mylt$ denote the gradients of the loss with respect to the hidden state $\VarGraphonNDE(u,t)$ and parameters $\vh\mylt$, respectively. In parallel to the adjoint GNDE equations \eqref{adjoint GNDE: Y}--\eqref{adjoint GNDE: z}, we introduce the following infinite-node adjoint equations. 

\paragraph{(S1) Hidden-state gradients.}
\begin{align}\label{adjoint Grahpon-NDE: Y}
\begin{split}
\frac{\partial}{\partial t} \VarGraphonNDEHiddenState(u,t) &= - \parens{\frac{\delta \GraphonNN}{\delta \VarGraphonNDE}}^*(\VarGraphonNDEHiddenState(u,t)),\\
\VarGraphonNDEHiddenState(u,T) &= \frac{\delta L}{\delta \VarGraphonNDE^{(L,T)}}(u) \in \spaceB.
\end{split}
\end{align}

\paragraph{(S2) Parameter gradients.}
\begin{align}\label{adjoint Graphon-NDE: H}
\VarGraphonNDEParameter\mylt &=\parens{\frac{\delta \GraphonNN}{\delta \vh\mylt}}^*(\VarGraphonNDEHiddenState(u,t))\in\spaceRFFK.
\end{align}
We note that the explicit form of adjoint operator appearing in \eqref{adjoint Grahpon-NDE: Y} and \eqref{adjoint Graphon-NDE: H} can be found in Appendix \ref{appendix: graphon convolutional operators} (specifically, \eq{adjoint operator: Graphon-NDE to filters}). We now state the assumptions needed for well-posedness. Unlike discrete-layer GNN architectures, which generate only finitely many hidden states during forward propagation, the continuous-depth structure of GNDEs and Graphon-NDEs evolves features through infinitely many intermediate states, forming a trajectory over a continuous time horizon. Ensuring the well-posedness of this trajectory, for both the forward dynamics and the backward adjoint equations, requires regularity conditions on the architecture, which we formalize as follows.

\begin{itemize}
    \item\phantomsection\label{AS0} \textbf{AS0.} The convolutional filters are \(\Liph\)-Lipschitz continuous in time, i.e., for all \(f, g \in [F]\), \(\ell \in [L]\), and \(k \in \mathbb{Z}_K\), $\left|\hltfgk[t_1] - \hltfgk[t_2]\right| \leq \Liph |t_1 - t_2|$, for all $t_1, t_2 \in \mathbb{R}^+$.
    
    \item\phantomsection\label{AS1} \textbf{AS1.} The activation function \(\sigma\) is $L_\sigma$-Lipschitz continuous with \(\sigma(0) = 0\).

    \item\phantomsection\label{AS2} \textbf{AS2.} The graphon degree function $d_P$ is bounded below, i.e., there exists $c_{\mathrm{min}}>0$ such that $d_P(u)>c_{\mathrm{min}}$ for all $u\in I$. 
\end{itemize}

Assumptions \ConvFilterLipschitz\ and \SigmaLipschitz\ impose regularity on the neural architecture: \ConvFilterLipschitz\ ensures that the time-varying filters evolve smoothly, and \SigmaLipschitz\ guarantees that the nonlinearity of the activation function does not amplify perturbations, and this property is satisfied by common activations such as ReLU, sigmoid, and tanh. As noted in \cite{gao2025global}, the smoothness of the activation function plays a critical role in ensuring well-posedness of both the forward and backward dynamics. Assumption \dPBoundedBelow\ is the nondegenerate-degree condition used in
\citet{keriven2020convergence}; it ensures that the graphon degree function do not approach
zero and that the symmetric normalized kernel $\tL_P$ is well
defined and bounded. The following result directly follows from \cite{yan2025convergence}. 

\begin{theorem}[Well-posedness of Graphon-NDE]\label{theorem: well-posedness}
Suppose that \ConvFilterLipschitz, \SigmaLipschitz, and \dPBoundedBelow\ hold. If \(\tW \in B(I^2)\) and \(\VarGraphonNDEIV \in \spaceB\), then for any \(T > 0\), there exists a unique solution \(\VarGraphonNDE \in C^1\parens{[0,T];\spaceB}\) to the Graphon-NDE \eqref{Graphon-NDE}.
\end{theorem}

For the adjoint system, we require a slightly stronger regularity condition on the activation function.
\begin{itemize}
\item\phantomsection\label{AS1diff} \textbf{AS1$'$.} The activation function $\sigma$ is differentiable.
\end{itemize}

\begin{theorem}[Well-posedness of Adjoint Graphon-NDE, proof in Appendix \ref{appendix: Boundedness and Well-posedness}]\label{theorem: Well-posedness of Adjoint System}
    Suppose that \ConvFilterLipschitz, \SigmaDiff, and \dPBoundedBelow\ hold. Let $T>0$ and $\VarGraphonNDE$ be the solution of the Graphon-NDE \eqref{Graphon-NDE}. There exists a unique solution \(\VarGraphonNDEHiddenState \in C^1\parens{[0,T];\spaceB}\) to the adjoint system \eqref{adjoint Grahpon-NDE: Y}.
\end{theorem}

The well-posedness of both the forward Graphon-NDE and its adjoint system ensures that the infinite-node limit provides a well-defined reference for analyzing size transferability. In the following sections, we quantify how closely finite-graph GNDEs approximate this limit.

\section{Convergence Analysis: from GNDEs to Graphon-NDEs}\label{sec: GNDE to Graphon-NDE}

We now turn to the central question: as the number of nodes $n$ grows, do GNDE solutions converge to the Graphon-NDE solution? Crucially, because GNDEs are continuous-depth models, we require
\emph{trajectory-wise} convergence, which means convergence of the entire
solution trajectory over $[0,T]$, not merely convergence at finitely many layers. As discussed in \cite{yan2025convergence}, trajectory-wise convergence is important for both forward and backward propagation:
it ensures that the continuous-time evolution of node-features is consistent across graph sizes and is necessary
for controlling the accumulation of gradient errors along adjoint trajectories. 

To quantify the trajectory-wise discrepancy, we extend the sampled mean-squared error metric of \citet{keriven2020convergence} to the continuous-time GNDE setting. Given a set of $n$ distinct samples $U_n=\{u_j:j\in[n]\}$ in $I$, a sampling operator $\SamplingOperator:\spaceB\to\spaceRnF$ is defined by $[\SamplingOperator \mathcal{Z}]_{i,:}:=\mathcal{Z}(u_i)$, $i\in[n]$. For Graphon-NDE solution $\VarGraphonNDE$ and GNDE solution $\VarGNDE$, we define
\begin{equation*}
\MSE(\VarGraphonNDE(\cdot,t),\VarGNDE(t)):=\frac{1}{\sqrt{n}}\normF{\SamplingOperator(\VarGraphonNDE(\cdot,t))-\VarGNDE(t)}. 
\end{equation*}
In addition, we adopt the piecewise-Lipschitz graphon regularity assumption from
\citet{keriven2020convergence}.
\begin{itemize}
\item\phantomsection\label{AS3} \textbf{AS3.} The graphon $\tW$ is piecewise Lipschitz, namely there exist a positive constant $c_{\mathrm{Lip}}$, a positive integer $n_I$, and a partition $I_s$, $s\in[n_I]$ of $I$, such that for any $v\in I$ and $s\in[n_I]$, $\abs{\tW(u,v)-\tW(u',v)}\leq c_{\mathrm{Lip}}|u-u'|$, for all $u,u'\in I_s$.
\end{itemize}

\subsection{Convergence of forward trajectory}

\begin{theorem}[Proof in Appendix \ref{appendix: Infinite-node Convergence of GNDEs}]\label{theorem: GNDE -> Graphon-NDE}
    Suppose that \ConvFilterLipschitz, \SigmaLipschitz, \dPBoundedBelow, and \GraphonLipschitz\ hold. Let $\gamma_1,\gamma_2\in(0,1)$ with $2\gamma_1+\gamma_2<1$. Suppose that $n$ is large enough satisfying \eqref{eq: n large enough for matrix L - LUn} and sparsity level $\alpha_n$ satisfies \eqref{assumption: sparsity alpha n large enough}. Let $\VarGNDE$ and $\VarGraphonNDE$ be solutions of GNDE \eqref{GNDE} and Graphon-NDE \eqref{Graphon-NDE}, respectively, with initial values satisfying 
    \begin{equation}\label{eq: initial values of GNDE and Graphon-NDE equal}
        \mX(0) = \SamplingOperator(\VarGraphonNDE(\cdot,0)). 
    \end{equation}
   Then, with probability at least $1-2\gamma_1-\gamma_2$, 
    \begin{equation}\label{eq: GNDE -> Graphon-NDE}
        \supT\MSE(\VarGraphonNDE(\cdot,t),\VarGNDE(t)) = \mathcal{O}\parens{\frac{1}{\sqrt{\alpha_n n}}+\frac{\sqrt{\log(n_ILFK/\gamma_2)}+\sqrt{\log(n_I/\gamma_1)}}{\sqrt{n}}}. 
    \end{equation}
\end{theorem}

We highlight that the error bound established in 
Theorem~\ref{theorem: GNDE -> Graphon-NDE} is uniform over the entire
continuous time horizon $t \in [0,T]$. While the resulting convergence rates are
structurally similar to those obtained for finite-depth GNNs
\citep{keriven2020convergence}, our result establishes a stronger dynamic
property. The key distinction is that a standard multi-layer GNN involves only
finitely many layers, whereas a GNDE defines a trajectory of hidden states indexed by a continuum of times. Consequently, fixed-time concentration estimates alone are not
sufficient to obtain trajectory-wise guarantees uniformly over $[0,T]$. To overcome this difficulty, we extend the spatial chaining technique for
non-normalized kernels in \citep[Lemma~4]{keriven2020convergence} to a
joint spatial-temporal setting. More precisely, Lemma~\ref{lemma: concentration ineq Yut}
controls the relevant empirical process over the augmented domain
(the Cartesian product of the node and time domains) by combining Hoeffding-type concentration with Dudley's chaining inequality. This allows the sampling error to be controlled
simultaneously over both the node and time variables. We finally remark that the randomness in our GNDE framework is localized strictly to the initial graph generation (the node samples and the Bernoulli edges). Once the graph is drawn, the subsequent ODE dynamics are deterministic, allowing the bound to hold simultaneously for all $t \in [0,T]$ on a single high-probability event.

\subsection{Convergence of hidden-state gradients}

With the forward trajectory convergence established, we turn to the adjoint equations, which govern the backward propagation of gradients. A stronger regularity condition on the activation function is needed.
\begin{itemize}
\item\phantomsection\label{AS1diffLip} \textbf{AS1$''$.} The derivative \(\sigma'\) is $L_{\sigma'}$-Lipschitz continuous.
\end{itemize}

\begin{theorem}[Proof in Appendix \ref{appendix: Infinite-node Convergence of Adjoint GNDEs}]\label{theorem: MSE Y Yn adjoint equation final}
Suppose that \ConvFilterLipschitz, \SigmaDerivativeLipschitz, \dPBoundedBelow, and \GraphonLipschitz\ hold. Let $\gamma_1,\gamma_2,\gamma_3\in(0,1)$ with $2\gamma_1+\gamma_2+\gamma_3<1$. Suppose that $n$ is large enough satisfying \eqref{eq: n large enough for matrix L - LUn} and sparsity level $\alpha_n$ satisfies \eqref{assumption: sparsity alpha n large enough}. Let $\VarGNDE$, $\VarGraphonNDE$, $\VarGNDEHiddenState$ and $\VarGraphonNDEHiddenState$ be the solutions of GNDE \eqref{GNDE}, Graphon-NDE \eqref{Graphon-NDE}, adjoint GNDE \eqref{adjoint GNDE: Y}, and adjoint Graphon-NDE \eqref{adjoint Grahpon-NDE: Y}, respectively. Suppose that initial value condition \eqref{eq: initial values of GNDE and Graphon-NDE equal} holds, and the terminal values of the adjoint systems satisfy 
\begin{equation}\label{eq: initial condition adjoint Hidden State}
\VarGNDEHiddenState(T)=\SamplingOperator(\VarGraphonNDEHiddenState(\cdot,T)).
\end{equation}
Then, with probability at least $1-2\gamma_1-\gamma_2-\gamma_3$, \eq{eq: GNDE -> Graphon-NDE} holds and 
\begin{equation}\label{eq:adjoint GNDE -> adjoint Graphon-NDE}
\supT\mathrm{MSE}_{U_n}(\VarGraphonNDEHiddenState(\cdot,t),\VarGNDEHiddenState(t)) = \mathcal{O}\parens{\frac{1}{\sqrt{\alpha_n n}}+\frac{\sqrt{\log(n_ILFK/\gamma_3)}+\sqrt{\log(n_ILFK/\gamma_2)}+\sqrt{\log(n_I/\gamma_1)}}{\sqrt{n}}}.
\end{equation}
\end{theorem}

The stronger assumption \SigmaDerivativeLipschitz\ (compared to \SigmaLipschitz\ in Theorem~\ref{theorem: GNDE -> Graphon-NDE}) is needed because the adjoint equation \eqref{adjoint GNDE: Y} involves the derivative of the GNN with respect to the hidden state, which depends on $\sigma'$. This mirrors a general principle in continuous-depth models: backward propagation requires higher regularity than forward propagation, as the adjoint operators involve derivatives of the velocity field~\citep{gao2025global}.

\subsection{Convergence of parameter gradients}

We next study the convergence of the parameter gradients, under the assumption that the filter parameters are differentiable in time.
\begin{itemize}
\item\phantomsection\label{AS0diff} \textbf{AS0$'$.} The parameters $\hltfgk$, $\ell\in[L]$, $f,g\in[F]$, $k\in\mathbb{Z}_K$, are differentiable in $t$.
\end{itemize}
Additionally, we require a stronger smoothness condition on the activation function.
\begin{itemize}
\item\phantomsection\label{AS1diff_diff} \textbf{AS1$'''$.} The activation function $\sigma$ is twice differentiable. 
\end{itemize}

\begin{theorem}[Proof in Appendix \ref{Appendix: Infinite-node Convergence of Adjoint GNDEs (Parameter Gradients)}]\label{theorem: z - zn leq TQh}
Suppose that \ConvFilterDiff, \SigmaTwiceDiff, \dPBoundedBelow, and \GraphonLipschitz\ hold. Let $\gamma_1,\gamma_2,\gamma_3,\gamma_4\in(0,1)$ satisfying $2\gamma_1+\gamma_2+\gamma_3+\gamma_4<1$. Suppose that $n$ is large enough satisfying \eqref{eq: n large enough for matrix L - LUn} and $\alpha_n$ satisfies \eqref{assumption: sparsity alpha n large enough}. Let $\VarGNDE$, $\VarGraphonNDE$, $\VarGNDEHiddenState$ and $\VarGraphonNDEHiddenState$ be the solutions of GNDE \eqref{GNDE}, Graphon-NDE \eqref{Graphon-NDE}, adjoint GNDE \eqref{adjoint GNDE: Y}, and adjoint Graphon-NDE \eqref{adjoint Grahpon-NDE: Y}, respectively. Suppose that initial value conditions satisfy \eqref{eq: initial values of GNDE and Graphon-NDE equal} and \eqref{eq: initial condition adjoint Hidden State}. Let $\VarGNDEParameter\mylt$ and $\VarGraphonNDEParameter\mylt$ be defined in \eqref{adjoint GNDE: z} and \eqref{adjoint Graphon-NDE: H}, respectively. Then with probability at least $1-2\gamma_1-\gamma_2-\gamma_3-\gamma_4$, \eqs \eqref{eq: GNDE -> Graphon-NDE}, \eqref{eq:adjoint GNDE -> adjoint Graphon-NDE} hold, and 
\begin{equation}\label{eq:adjoint GNDE parameter -> adjoint Graphon-NDE parameter}
  \begin{aligned}
\max_{\ell\in[L]}\supT\Big\|\VarGraphonNDEParameter\mylt&-\frac{1}{n}\VarGNDEParameter\mylt\Big\|_{\mathrm{max}}=\mathcal{O}\parens{\frac{\sqrt{\log(LF^2 K/\gamma_4)}}{\sqrt{n}}}+\\
&\mathcal{O}\Big(\parens{\frac{1}{\sqrt{\alpha_n n}}+\frac{\sqrt{\log(n_ILFK/\gamma_2)}+\sqrt{\log(n_I/\gamma_1)}}{\sqrt{n}}}\\
&\times\parens{\frac{1}{\sqrt{\alpha_n}}+\sqrt{\log(n_ILFK/\gamma_3)}+\sqrt{\log(n_ILFK/\gamma_2)}+\sqrt{\log(n_I/\gamma_1)}}\Big).
\end{aligned}  
\end{equation}

\end{theorem}
Theorems~\ref{theorem: GNDE -> Graphon-NDE}, \ref{theorem: MSE Y Yn adjoint equation final} and \ref{theorem: z - zn leq TQh} together establish size transferability at three levels: the forward trajectory (Theorem~\ref{theorem: GNDE -> Graphon-NDE}), the hidden-state gradients (Theorem~\ref{theorem: MSE Y Yn adjoint equation final}), and the parameter gradients (Theorem~\ref{theorem: z - zn leq TQh}). The progression of regularity assumptions on the activation function $\sigma$---from a Lipschitz continuous $\sigma$ (for the forward pass), to a Lipschitz continuous derivative $\sigma'$ (for the hidden-state gradients), and finally to a twice-differentiable $\sigma$ (for the parameter gradients)---reflects the increasing order of smoothness required at each analytical stage.

We remark that the bounds in
Theorems~\ref{theorem: GNDE -> Graphon-NDE}, \ref{theorem: MSE Y Yn adjoint equation final}, and \ref{theorem: z - zn leq TQh} should be interpreted as worst-case upper bounds. The numerical results in Sections~\ref{subsec: exp forward},
\ref{subsec: exp hidden state gradient}, and~\ref{subsec: exp parameter gradient} suggest that the forward-trajectory estimate is relatively sharp, whereas the hidden-state gradient estimate and the
parameter-gradient estimate can be conservative, especially for sparse graphs. One possible reason is that the adjoint and parameter-gradient analyses are obtained through successive stability estimates. In particular, the
hidden-state adjoint bound depends on the forward-trajectory error, and the parameter-gradient bound further depends on both the forward and hidden-state adjoint errors. At each stage, triangle inequalities and worst-case operator-norm estimates are used to control all possible inputs uniformly. These successive worst-case estimates may amplify the apparent sparsity dependence of the bounds, even though faster convergence is observed in the
numerical experiments. We leave the derivation of sharper transfer-error bounds for future work.

\section{Convergence Analysis: from Residual GNNs to GNDEs}\label{sec: ResGNN to GNDE}

The previous section established that GNDEs converge to Graphon-NDEs as the number of nodes grows. In practice, however, GNDEs are never solved exactly; they are discretized into residual GNNs via numerical methods such as Euler's scheme \eqref{eq: Euler's method for Xn}. A natural question thus arises: how well do the discrete approximations $\VarGNDE^{[m]}$ approximate the continuous GNDE solution $\VarGNDE(t)$ as the step size $\kappa$ decreases?

The importance of this question is twofold. First, it completes the approximation chain needed for a \emph{two-scale convergence} result. The previous section establishes the spatial limit, namely the convergence of GNDEs to Graphon-NDEs as the number of nodes grows, while this section establishes the temporal limit, namely the convergence of residual GNNs to GNDEs as the step size tends to zero. Coupling these two results yields an end-to-end convergence guarantee from finite-node, discrete residual GNNs to the infinite-node, time-continuous Graphon-NDE. Second, this result provides the foundation for comparing the DTO and OTD training paradigms in Section~\ref{sec: DTO vs OTD}, as both produce discrete gradient sequences whose difference can be controlled using the discretization-error estimates developed here.

\subsection{Discretization error of forward trajectory}

\begin{theorem}[Proof in Appendix \ref{Appendix: Infinite-node Convergence of Discretized GNDEs}]\label{theorem: discretized gnde to gnde}
    Suppose that \ConvFilterDiff, \SigmaDiff, \dPBoundedBelow, and \GraphonLipschitz\ hold. Let $\gamma_1,\gamma_2\in(0,1)$ with $2\gamma_1+\gamma_2<1$. Suppose that $n$ is large enough satisfying \eqref{eq: n large enough for matrix L - LUn} and \eqref{eq: n is big enough 2}; sparsity level $\alpha_n$ satisfies \eqref{assumption: sparsity alpha n large enough}. Let $\VarGNDE(t)$ be the solution of GNDE \eqref{GNDE} and $\left\{\VarGNDE^{[m]} : m \in [M]\right\}$ be computed via Euler's method \eqref{eq: Euler's method for Xn} with step size $\kappa = T/M$ and initial value 
    \begin{equation}\label{eq: initial value condition on GNDE and discretized GNDE}
\VarGNDE^{[0]}=\VarGNDE(0). 
    \end{equation}
    Then, with probability at least $1-2\gamma_1-\gamma_2$, \eq{eq: GNDE -> Graphon-NDE} holds and 
    \begin{equation}\label{eq: error bound X(t) - Xm temperol discretization error}
    \frac{1}{\sqrt{n}}\max_{m\in[M]}\normF{\VarGNDE(t_m)- \VarGNDE^{[m]}}=\mathcal{O}\parens{\frac{1}{M}}.
    \end{equation}
\end{theorem}

The $\mathcal{O}(1/M)$ rate is the classical first-order convergence of the explicit Euler method, confirming that the discretization error vanishes as the number of time steps $M$ increases, independently of the graph size $n$. We numerically validate the predicted $\mathcal{O}(1/M)$ convergence rate in
Section~\ref{subsec: exp discretization}.

\subsection{Discretization error of hidden-state gradients}

\begin{theorem}[Proof in Appendix \ref{Appendix: Infinite-node Convergence of Discretized Adjoint GNDEs}]\label{theorem: discretized adjoint gnde to adjoint gnde Hidden state}
Suppose that \ConvFilterDiff, \SigmaTwiceDiff, \dPBoundedBelow, and \GraphonLipschitz\ hold. Let $\gamma_1,\gamma_2,\gamma_3\in(0,1)$ with $2\gamma_1+\gamma_2+\gamma_3<1$. Suppose that $n$ is large enough satisfying \eqref{eq: n is big enough 2}, \eqref{eq: n is big enough 3} and \eqref{eq: n large enough for matrix L - LUn}; sparsity level $\alpha_n$ satisfies \eqref{assumption: sparsity alpha n large enough}. Suppose that initial conditions \eqref{eq: initial values of GNDE and Graphon-NDE equal}, \eqref{eq: initial condition adjoint Hidden State} and \eqref{eq: initial value condition on GNDE and discretized GNDE} hold. Let $\VarGNDEHiddenState(t)$ be the solution of adjoint GNDE \eqref{adjoint GNDE: Y} and $\left\{\VarGNDEHiddenState^{[m]} : m \in [M]\right\}$ be computed via backward-in-time Euler's method \eqref{eq: discretized Yn} with step size $\kappa = T/M$ and terminal value 
\begin{equation}\label{eq: initial condition GNDE discretized GNDE Adjoint A}
\VarGNDEHiddenState^{[M]}=\VarGNDEHiddenState(T).
\end{equation} 
Then, with probability at least $1-2\gamma_1-\gamma_2-\gamma_3$, \eq{eq:adjoint GNDE -> adjoint Graphon-NDE} holds and 
    \begin{align*}
        \frac{1}{\sqrt{n}}\max_{m\in[M]}\normF{\VarGNDEHiddenState^{[m]}-\VarGNDEHiddenState(t_m)}=\mathcal{O}\parens{\frac{\frac{1}{\sqrt{\alpha_n}}+\sqrt{\log(n_ILFK/\gamma_3)}+\sqrt{\log(n_ILFK/\gamma_2)}+\sqrt{\log(n_I/\gamma_1)}}{M}}.
    \end{align*}
\end{theorem}

The adjoint discretization error retains the $\mathcal{O}(1/M)$ dependence on the temporal step size. However, the prefactor now depends explicitly on the sparsity parameter $\alpha_n$, through a term of order $1/\sqrt{\alpha_n}$. Thus, in the sparse-graph regime where $\alpha_n \to 0$, the discretization error bound deteriorates unless the number of time steps $M$ is chosen large enough relative to the sparsity level.

\subsection{Discretization error of parameter gradients}

\begin{theorem}[Proof in Appendix \ref{Appendix: Infinite-node Convergence of Discretized Adjoint GNDEs (Parameter Gradients)}]\label{theorem: discretized adjoint gnde to adjoint gnde parameters}
    Suppose that \ConvFilterDiff, \SigmaTwiceDiff, \dPBoundedBelow, and \GraphonLipschitz\ hold. Let $\gamma_1,\gamma_2,\gamma_3\in(0,1)$ with $2\gamma_1+\gamma_2+\gamma_3<1$. Suppose that $n$ is large enough satisfying \eqref{eq: n is big enough 2}, \eqref{eq: n is big enough 3} and \eqref{eq: n large enough for matrix L - LUn}; sparsity level $\alpha_n$ satisfies \eqref{assumption: sparsity alpha n large enough}; $M$ is large enough such that \eqref{eq: M is large enough} holds. Suppose that initial conditions \eqref{eq: initial values of GNDE and Graphon-NDE equal}, \eqref{eq: initial condition adjoint Hidden State}, \eqref{eq: initial value condition on GNDE and discretized GNDE} and \eqref{eq: initial condition GNDE discretized GNDE Adjoint A} hold. Let $\VarGNDEParameter\mylt$ and $\VarGNDEParameter\mylm$ be defined in \eqref{adjoint GNDE: z} and \eqref{grad of h optimize-then-dicretize}, respectively. Then with probability at least $1-2\gamma_1-\gamma_2-\gamma_3$, 
    \begin{equation*}
\frac{1}{n}\max_{m\in[M],\ell\in[L]}\normMax{\VarGNDEParameter\mylt[\ell][t_m] - \VarGNDEParameter\mylm} =\mathcal{O}\parens{\frac{\frac{1}{\sqrt{\alpha_n}}+\sqrt{\log(n_ILFK/\gamma_3)}+\sqrt{\log(n_ILFK/\gamma_2)}+\sqrt{\log(n_I/\gamma_1)}}{M}}.
\end{equation*}
\end{theorem}

\section{DTO versus OTD}\label{sec: DTO vs OTD}

We address the question raised in Section~\ref{sec: training of GNDEs}: under what conditions do the DTO and OTD training paradigms produce comparable gradients? Recall from the comparison in Section~\ref{sec: training of GNDEs} that the DTO gradient and the OTD gradient differ at finite step size $\kappa$ because their discrete recursions evaluate the adjoint operators at different points (left versus right endpoints). The results of Section~\ref{sec: ResGNN to GNDE} show that the OTD gradients, both for hidden states and for parameters, converge to their continuous-time adjoint GNDE counterparts. An analogous convergence result can also be established for DTO gradients. Therefore, by comparing both DTO and OTD gradients against the same continuous-time GNDE limit, we can control their discrepancy in terms of their respective convergence errors.

\begin{theorem}[DTO versus OTD for Hidden-State Gradients, proof in Appendix \ref{Appendix: DTO versus OTD for Gradients of Hidden States}]\label{theorem: hidden states gradients DTO and OTD}
Let $\VarDTOGradient^{[m]}$ and $\VarGNDEHiddenState^{[m]}$ be generated from \eqref{eq: discretized Gn} and \eqref{eq: discretized Yn}, respectively. Under the same assumptions as Theorem \ref{theorem: discretized adjoint gnde to adjoint gnde Hidden state}, with probability at least $1-2\gamma_1-\gamma_2-\gamma_3$, 
    \begin{align}\label{eq: theoretical dto otd state}
\frac{1}{\sqrt{n}}\max_{m\in[M]}\normF{\VarDTOGradient^{[m]}-\VarGNDEHiddenState^{[m]}}= \mathcal{O}\parens{\frac{\frac{1}{\sqrt{\alpha_n}}+\sqrt{\log(n_ILFK/\gamma_3)}+\sqrt{\log(n_ILFK/\gamma_2)}+\sqrt{\log(n_I/\gamma_1)}}{M}}.
    \end{align}
\end{theorem}

For the parameter gradients, the estimate concerns the \emph{local} gradient contribution on each time interval. This contribution contains an explicit factor $\kappa=T/M$ from the time discretization. As a result, the local difference is multiplied by this extra factor of $\kappa$, yielding the rate $\mathcal{O}(1/M^2)$ for each parameter-gradient contribution.

\begin{theorem}[DTO versus OTD for Parameter Gradients, proof in Appendix \ref{Appendix: DTO versus OTD for Gradients of Parameters}]\label{theorem: parameter gradients DTO and OTD}
Let $\VarDTOGradient^{[m]}$ and $\VarGNDEHiddenState^{[m]}$ be generated from \eqref{eq: discretized Gn} and \eqref{eq: discretized Yn}, respectively. Under the same assumptions as Theorem \ref{theorem: discretized adjoint gnde to adjoint gnde Hidden state}, with probability at least $1-2\gamma_1-\gamma_2-\gamma_3$, 
\begin{equation}\label{eq: theoretical dto otd param}
\begin{aligned}
&\frac{1}{n}\max_{m\in[M],\ell\in[L]}\normMax{\kappa\,\Dh^{[\ell,m]}\parens{\VarDTOGradient^{[m+1]}} - \kappa\,\mathcal{D}_h^{[\ell,m]}\parens{\VarGNDEHiddenState^{[m]}}}\\
=\ & \mathcal{O}\parens{\frac{\frac{1}{\sqrt{\alpha_n}}+\sqrt{\log(n_ILFK/\gamma_3)}+\sqrt{\log(n_ILFK/\gamma_2)}+\sqrt{\log(n_I/\gamma_1)}}{M^2}} .
    \end{aligned}
\end{equation}

\end{theorem}

Theorems~\ref{theorem: hidden states gradients DTO and OTD} and
\ref{theorem: parameter gradients DTO and OTD} show that DTO and OTD become
asymptotically equivalent as $M\to\infty$. The hidden-state gradient
discrepancy vanishes at order $\mathcal{O}(1/M)$, while each local
parameter-gradient difference vanishes at order $\mathcal{O}(1/M^2)$. Thus,
for sufficiently fine time discretizations, the gradient discrepancy between
DTO and OTD becomes small. The choice between them is then mainly determined by
implementation costs: DTO differentiates through the discrete solver and
typically requires storing the forward states, whereas OTD computes gradients
through the continuous adjoint equation solved backward in time. We numerically
assess the two predicted rates in Section~\ref{subsec: exp dto otd}.

\section{Numerical Experiments}\label{sec: experiments}

We organize the numerical experiments into two parts. First, we validate the
convergence rates predicted by our analysis in
Section~\ref{sec: validation of theory}. Second, we test the practical
size-transferability of learned GNDE dynamics by training on small source graphs
and evaluating on independently sampled larger target graphs in
Section~\ref{subsec: exp transfer}.

\subsection{Validation of theoretical rates}\label{sec: validation of theory}

We validate our theoretical results through five groups of experiments, organized into two categories. First, we study graph-size transfer errors, denoted by $\mathcal{TE}$. These experiments support the convergence results for forward trajectories, hidden-state gradients, and parameter gradients in Theorems~\ref{theorem: GNDE -> Graphon-NDE}, \ref{theorem: MSE Y Yn adjoint equation final}, and~\ref{theorem: z - zn leq TQh}, respectively. They are presented in Sections~\ref{subsec: exp forward}, \ref{subsec: exp hidden state gradient}, and~\ref{subsec: exp parameter gradient}. Second, we study temporal discretization errors and DTO--OTD consistency, denoted by $\mathcal{DE}$. These experiments validate the discretization and gradient-consistency results in Theorems~\ref{theorem: discretized gnde to gnde}, \ref{theorem: hidden states gradients DTO and OTD}, and~\ref{theorem: parameter gradients DTO and OTD}. They are presented in Sections~\ref{subsec: exp discretization} and~\ref{subsec: exp dto otd}.

\paragraph{Graphons.} We consider two graphons: (i) an hierarchical stochastic block model (HSBM) graphon
$\tW:[0,1]^2\to{0,1}$ with two hierarchical levels; and (ii) the tent graphon
$\tW(x,y)=\max(0,1-|x-y|)$. These graphons are shown in
Figure~\ref{fig:graphons}.

\paragraph{Graph sampling.}
Given $n$ nodes sampled independently and uniformly from the unit interval $I$,
denoted by $U_n=\{u_i\}_{i=1}^n$, edges are sampled independently as Bernoulli
random variables with probabilities $\alpha_n \tW(u_i,u_j)$, where
$\alpha_n\in(0,1]$ controls the sparsity level.

\paragraph{GNDE architecture.}
The vector field of the GNDE is parameterized by a GNN with three hidden layers and one output layer (\texttt{fc1}, \texttt{mid1}, \texttt{mid2}, \texttt{fc2}). The width of each hidden layer is $16$. Each hidden layer applies a polynomial graph convolution filter
with $K=3$ using the symmetric normalized adjacency matrix $\matL$, followed by the Softplus activation $\sigma(x)=\log(1+e^x)$. Node features of
dimension $F=4$ are initialized using random Fourier polynomials of degree $10$, i.e., for $i$-th node $u_i$ and $f$-th feature, the initial feature is $[\mZ]_{i,f} = \sum_{k=1}^{10} a_{f,k}\sin(2\pi k\, u_i) + b_{f,k}\cos(2\pi k\, u_i)$,
where coefficients $a_{f,k}$ and $b_{f,k}$ are drawn independently from normal distribution $\mathcal{N}(0,1)$.

\subsubsection{Transfer Error for Forward trajectories}\label{subsec: exp forward}

We numerically validate the convergence rate of finite-node GNDE solutions
toward the Graphon-NDE solution established in
Theorem~\ref{theorem: GNDE -> Graphon-NDE}. The Graphon-NDE solution
$\VarGraphonNDE(u,t)$ has a continuum node variable
$u\in [0,1]$ and time $t\in[0,T]$, and is not available in closed form.
Therefore, we approximate the Graphon-NDE trajectory by a high-resolution GNDE
trajectory $\VarGNDE_{n_{\mathrm{ref}}}(t)$ on a large reference graph
$\mathcal{G}_{n_{\mathrm{ref}}}$ with $n_{\mathrm{ref}}=10000$ nodes placed at
evenly spaced latent positions. This reference GNDE is then solved using
Euler's method with $M_{\mathrm{ref}}=256$ time steps. We denote the resulting
numerical solution by
$\VarGNDE_{n_{\mathrm{ref}},M_{\mathrm{ref}}}^{[m]}$, which serves as a
surrogate for $\VarGNDE_{n_{\mathrm{ref}}}(t_m)$, with $t_m=mT/M_{\mathrm{ref}}$. 

For each smaller graph size $n$, we compute the corresponding approximation $\VarGNDE_{n,M_{\mathrm{ref}}}^{[m]}$ using the same time
discretization. We slightly abuse notation and use $\SamplingOperator$ to denote
the numerical sampling map from the reference grid to the node set $U_n$ of the
smaller randomly sampled graph by using nearest neighbors. We define the empirical transfer error for forward trajectory by
\begin{equation}\label{eq: E forward}
    \mathcal{TE}_{\mathrm{fwd}}(n)
    :=
    \max_{m\in[M_{\mathrm{ref}}]}
    \frac{1}{\sqrt{n}}
    \normF{
    \VarGNDE_{n,M_{\mathrm{ref}}}^{[m]}
    -
    \SamplingOperator\left(
    \VarGNDE_{n_{\mathrm{ref}},M_{\mathrm{ref}}}^{[m]}
    \right)
    } .
\end{equation}
Here, $\mathcal{TE}_{\mathrm{fwd}}(n)$ serves as a numerical surrogate for the
discrepancy between the GNDE and Graphon-NDE solutions, namely the left-hand
side of \eqref{eq: GNDE -> Graphon-NDE}. We choose $n_{\mathrm{ref}}$ and
$M_{\mathrm{ref}}$ sufficiently large so that $\mathcal{TE}_{\mathrm{fwd}}(n)$
accurately approximates this exact error and the reference approximation error
is negligible. We therefore expect $\mathcal{TE}_{\mathrm{fwd}}(n)$ to exhibit
the rate $\mathcal{O}(1/\sqrt{\alpha_n n})$, up to logarithmic factors, predicted by
Theorem~\ref{theorem: GNDE -> Graphon-NDE}.

Concretely, we consider graph sizes
$n \in \{100,150,200,300,500,700,1000,1500,2000,3000\}$ and run 100 independent
trials. In each trial, we sample a set of random Fourier coefficients, and initialize the GNDE, whose weights are shared across
all graph sizes in that trial. For each $n$, we consider three sparsity levels
$\alpha_n \in \{1, n^{-0.25}, n^{-0.5}\}$. Since the theoretical rate is $\mathcal{O}(1/\sqrt{\alpha_n n})$, the expected
log-log slopes are $-0.50$, $-0.375$, and $-0.25$, respectively. 
Figure~\ref{fig:forward_convergence} shows clear power-law decay across all
three regimes for both the tent and HSBM graphons, and the fitted slopes remain
close to these theoretical values.

\begin{figure}[t]
    \centering
    \includegraphics[width=\textwidth]{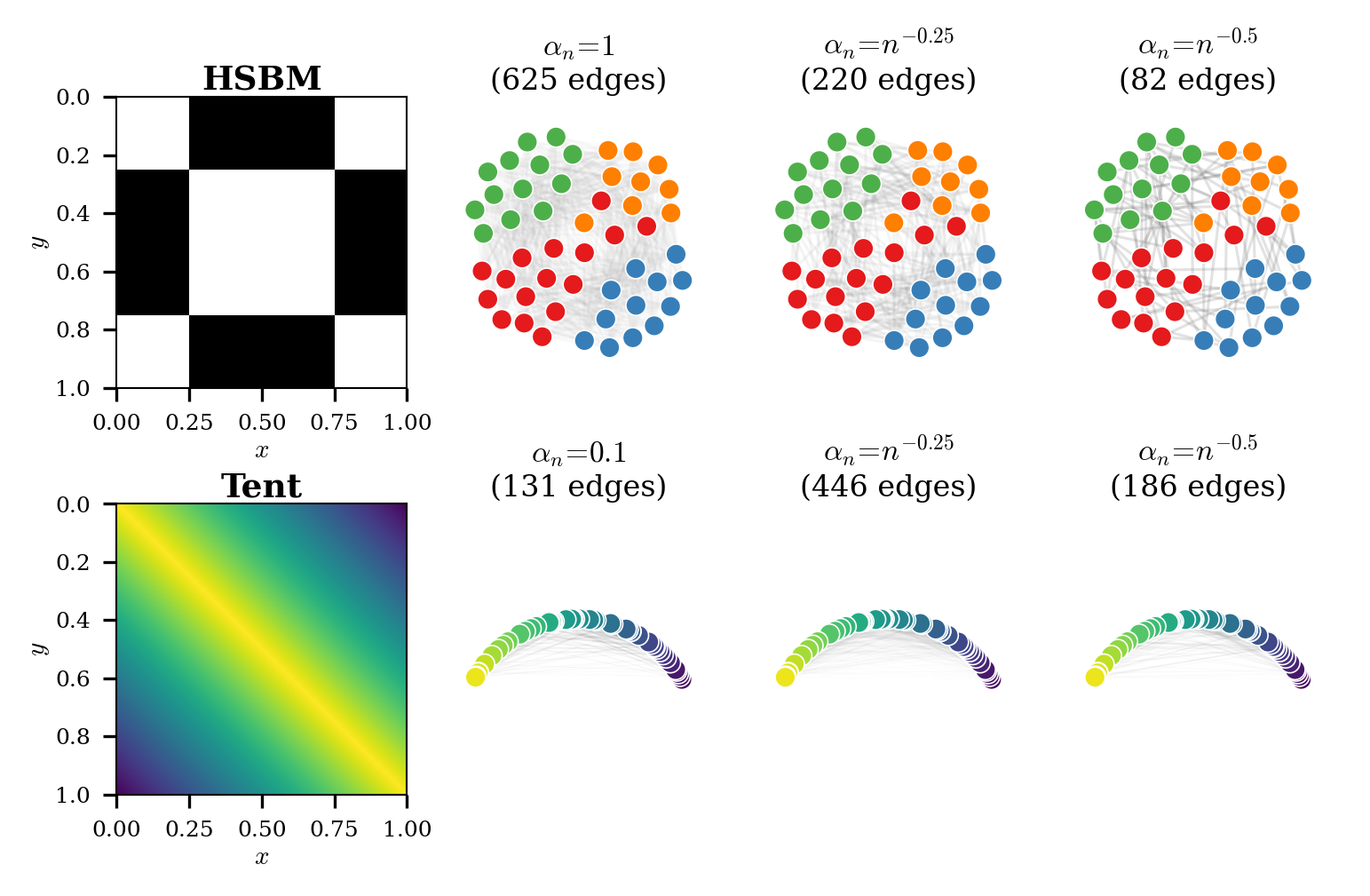}
    \caption{Illustration of the tent and HSBM graphons and representative sampled graphs. The graphs are sampled from each graphon with edge probabilities
$\alpha_n \tW(u_i,u_j)$, where
$\alpha_n \in \{1,\,n^{-0.25},\,n^{-0.5}\}$.}
    \label{fig:graphons}
\end{figure}

\begin{figure}[t]
    \centering
    \includegraphics[width=\textwidth]{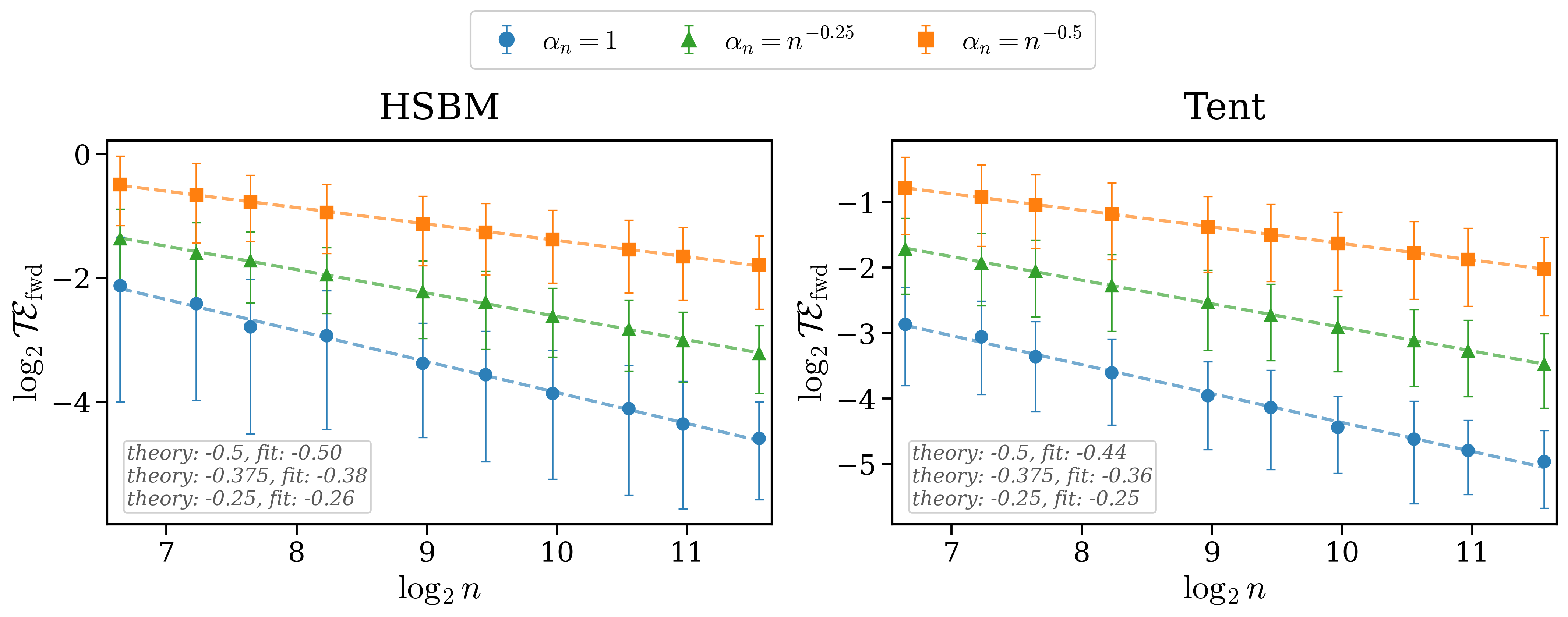}
    \caption{Transfer error for forward trajectory: log-log plot of $\mathcal{TE}_{\mathrm{fwd}}(n)$ versus the number of nodes $n$ for the HSBM graphon (left) and the tent graphon (right). Three sparsity levels $\alpha_n\in\{1,n^{-0.25},n^{-0.5}\}$ are considered. Markers denote empirical means, error bars show standard deviations, and dashed lines indicate fitted power-law trends. The fitted slopes are consistent with the theoretical rate $\mathcal{O}(1/\sqrt{\alpha_n n})$.}
    \label{fig:forward_convergence}
\end{figure}

\subsubsection{Transfer Error for Hidden-state Gradients}\label{subsec: exp hidden state gradient}

We next validate the hidden-state adjoint convergence rate from
Theorem~\ref{theorem: MSE Y Yn adjoint equation final}. We follow the same reference-graph protocol
as in Section~\ref{subsec: exp forward}. In each trial,
we additionally draw an independent random Fourier polynomial as the terminal
adjoint state and solve the adjoint equation backward in time on both the
reference graph and each sampled graph, using the same Euler time discretization. We define the empirical transfer error for hidden-state gradients as
\begin{equation}\label{eq: E adjoint}
\mathcal{TE}_{\mathrm{state}}(n)
:=
\max_{m\in[M_{\mathrm{ref}}]}
\frac{1}{\sqrt{n}}
\normF{
\VarGNDEHiddenState_{n,M_{\mathrm{ref}}}^{[m]}
-
\SamplingOperator\left(
\VarGNDEHiddenState_{n_{\mathrm{ref}},M_{\mathrm{ref}}}^{[m]}
\right)
} .
\end{equation}
This quantity is the discretized analogue of the left-hand side
in~\eqref{eq:adjoint GNDE -> adjoint Graphon-NDE}. Theorem~\ref{theorem:
MSE Y Yn adjoint equation final} predicts the worst-case rate
$\mathcal{O}(1/\sqrt{\alpha_n n})$, up to logarithmic factors.

Figure~\ref{fig:adjoint_convergence} reports the results over the same setting of graph
sizes, sparsity levels, and independent trials as in Section \ref{subsec: exp forward}. The empirical
decay rates are consistent with the theoretical rate
$\mathcal{O}(1/\sqrt{\alpha_n n})$ for both the HSBM and tent graphons; in the sparsest regime
$\alpha_n=n^{-1/2}$, the observed convergence is slightly faster than the worst-case prediction.

\begin{figure}[t]
    \centering
    \includegraphics[width=\textwidth]{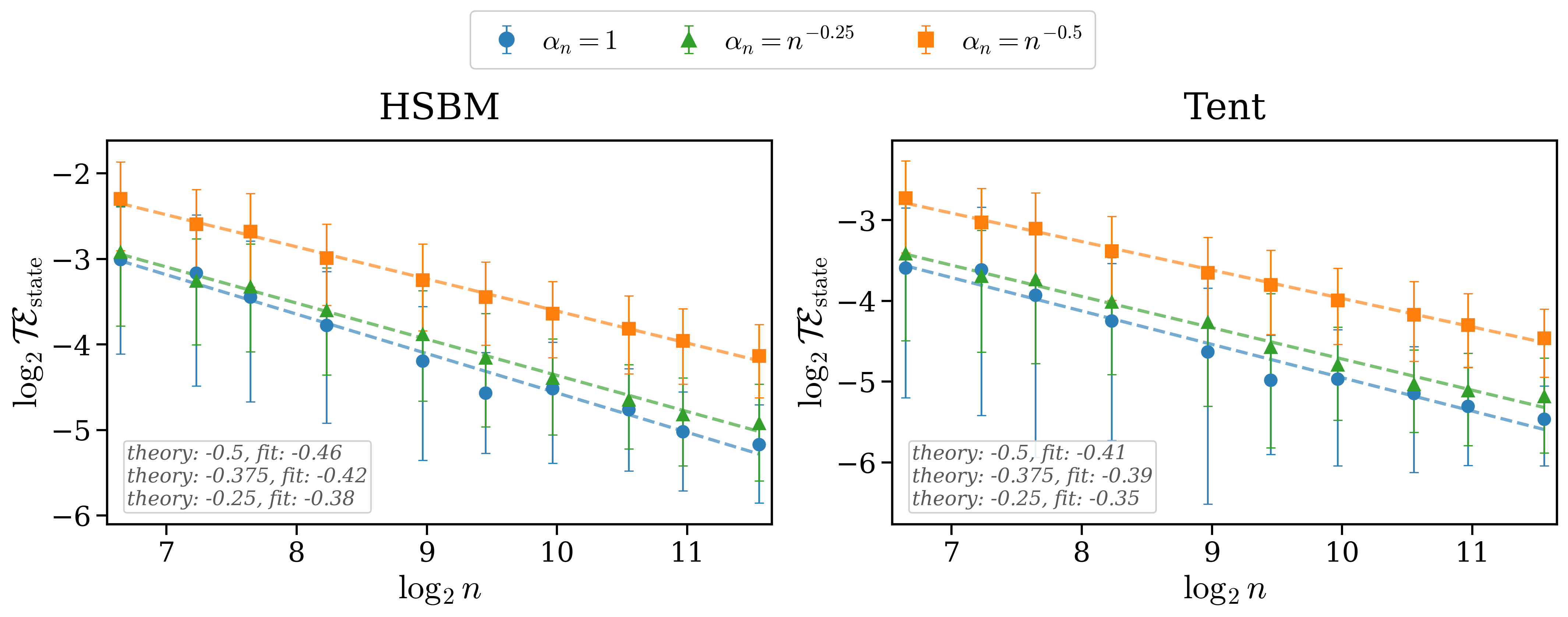}
    \caption{Transfer error for hidden-state gradients: log-log plot of $\mathcal{TE}_{\mathrm{state}}(n)$ versus the number of nodes $n$ for the HSBM graphon (left) and the tent graphon (right), using the same graph sizes and sparsity levels as in Figure~\ref{fig:forward_convergence}. The adjoint trajectories are initialized from a terminal random Fourier state and integrated backward in time. The empirical decay rates are consistent with the theoretical worst-case bound $\mathcal{O}(1/\sqrt{\alpha_n n})$. For sparse graphs, the observed convergence rate is slightly faster than the worst-case theoretical rates.
}

    \label{fig:adjoint_convergence}
\end{figure}

\subsubsection{Transfer Error for Parameter Gradients}\label{subsec: exp parameter gradient}

We validate the parameter-gradient convergence rate from
Theorem~\ref{theorem: z - zn leq TQh}. We apply the same transfer protocol as in Sections~\ref{subsec: exp forward} and \ref{subsec: exp hidden state gradient}. At each backward-in-time Euler step, we
additionally extract the gradients of parameters in four layers, and define the transfer error for parameter gradients as 
\begin{equation}\label{eq: E param grad}
\mathcal{TE}_{\mathrm{param}}(n) := \max_{\ell\in[4]}\,\max_{m\in[M_{\mathrm{ref}}]}
\Big\|\tfrac{1}{n}\,\VarGNDEParameter_{n,M_{\mathrm{ref}}}^{[\ell,m]} - \tfrac{1}{N_{\mathrm{ref}}}\,
\VarGNDEParameter_{N_{\mathrm{ref},M_{\mathrm{ref}}}}^{[\ell,m]}\Big\|_{\max},
\end{equation}
which is a discrete analogue of the left-hand side of \eq{eq:adjoint GNDE parameter -> adjoint Graphon-NDE parameter}. It follows from Theorem~\ref{theorem: z - zn leq TQh} that the leading
worst-case sparsity-dependent rate is
$\mathcal{O}(1/(\alpha_n\sqrt{n}))$, up to logarithmic factors.

As before, we consider three sparsity regimes
$\alpha_n\in\{1,n^{-0.25},n^{-0.5}\}$. The corresponding theoretical
worst-case rates are
$\mathcal{O}(n^{-0.5})$, $\mathcal{O}(n^{-0.25})$, and
$\mathcal{O}(1)$, respectively. In particular, for the sparsest regime
$\alpha_n=n^{-0.5}$, the theorem does not guarantee convergence of the parameter-gradient transfer  error. However, in the numerical experiments, we observe substantially faster decay than the worst-case rate predicts. We report the results for $\mathcal{TE}_{\mathrm{param}}(n)$ in Figure~\ref{fig:param_grad_convergence}. For both the HSBM and tent graphons, the empirical decay is close to $\mathcal{O}(n^{-0.5})$ across all three sparsity levels. Thus, the dense case is consistent with the theoretical worst-case rate, while the sparse cases exhibit faster convergence than the
worst-case bound suggests. This indicates that the theoretical bound for parameter-gradient transfer on sparse graphs is conservative in these
experiments.

\begin{figure}[t]
    \centering
    \includegraphics[width=\textwidth]{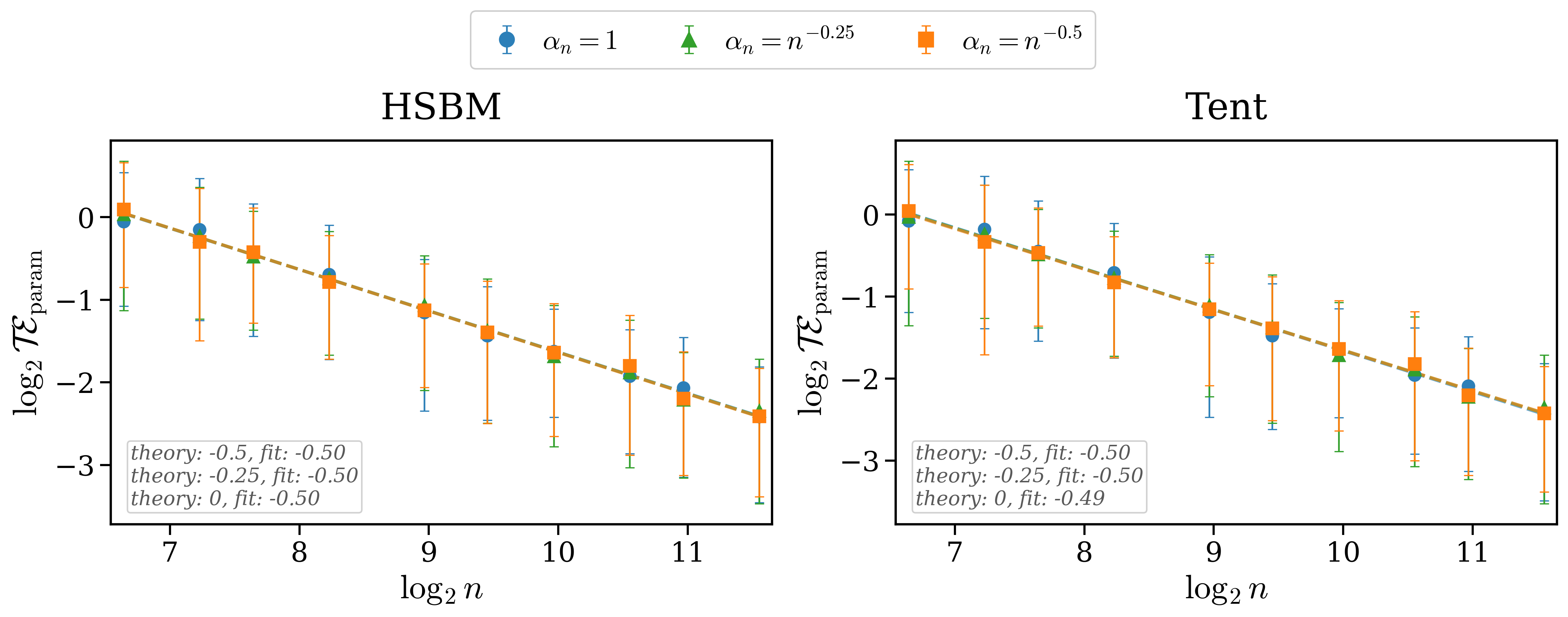}
    \caption{Transfer error for parameter gradients: log-log plot of
    $\mathcal{TE}_{\mathrm{param}}(n)$ versus the number of nodes $n$ for the HSBM graphon (left) and tent graphon
    (right). Three sparsity levels $\alpha_n\in\{1,n^{-0.25},n^{-0.5}\}$ are considered. The observed empirical rates for both graphons across three sparsity levels are all around $\mathcal{O}(n^{-0.5})$. The empirical rate matches the dense-graph theoretical rate and is much faster than the worst-case theoretical rates predicted for sparse graphs.
}
    \label{fig:param_grad_convergence}
\end{figure}

\subsubsection{Discretization Error for Forward Trajectories}\label{subsec: exp discretization}

We numerically validate the theoretical error rate established in
Theorem~\ref{theorem: discretized gnde to gnde} for Euler discretizations of
GNDEs. The theorem
bounds the discrepancy between the exact GNDE solution
$\VarGNDE_n(t_m)$ and its Euler approximation $\VarGNDE_n^{[m]}$. Again, in practice, the exact trajectory $\VarGNDE_n(t)$ is not available in closed form.
Therefore, we use a fine-resolution Euler solution as a reference approximation.

Specifically, we compute a reference trajectory
$\VarGNDE_{n,M_{\mathrm{ref}}}^{[k]}$ using Euler's method with
$M_{\mathrm{ref}}=256$ time steps. For each coarse resolution
$M \in \{2,4,8,16,32,64\}$, we compute the Euler iterates
$\VarGNDE_{n,M}^{[m]}$ on the time grid $t_m=mT/M$, $m\in[M]$. Since each $M$ divides $M_{\mathrm{ref}}$, the coarse time grid is nested in the reference
grid, and the matching fine-grid index is $k_m = mM_{\mathrm{ref}}/M$. We define
\begin{equation}\label{eq: E disc}
\mathcal{DE}_{\mathrm{fwd}}(M)
    :=
    \frac{1}{\sqrt{n}}
    \max_{m\in[M]}
    \normF{
    \VarGNDE_{n,M}^{[m]}
    -
    \VarGNDE_{n,M_{\mathrm{ref}}}^{[k_m]}
    } .
\end{equation}
Here $\VarGNDE_{n,M_{\mathrm{ref}}}^{[k_m]}$ serves as a numerical
surrogate for the exact solution $\VarGNDE_n(t_m)$ at time $t_m$. We expect
$\mathcal{DE}_{\mathrm{fwd}}(M)$ to exhibit the first-order rate
$\mathcal{O}(1/M)$ predicted by Theorem~\ref{theorem: discretized gnde to gnde}.

For each graphon, we fix a single random graph $\mathcal{G}_n$ with
$n=2000$ nodes, sampled from the $512\times512$ graphon discretization without edge drop, i.e., in the dense-graph regime. We then vary only the
random Fourier initialization and model weights across 50 random seeds.

Figure~\ref{fig:discretization} shows that $\mathcal{DE}_{\mathrm{fwd}}(M)$
decays approximately as a power law in $M$ for both graphons. The fitted
log-log slopes are close to $-1$, consistent with the theoretical first-order rate $\mathcal{O}(1/M)$, and this behavior is stable across the 50 random trials.

\begin{figure}[t]
    \centering
    \includegraphics[width=\textwidth]{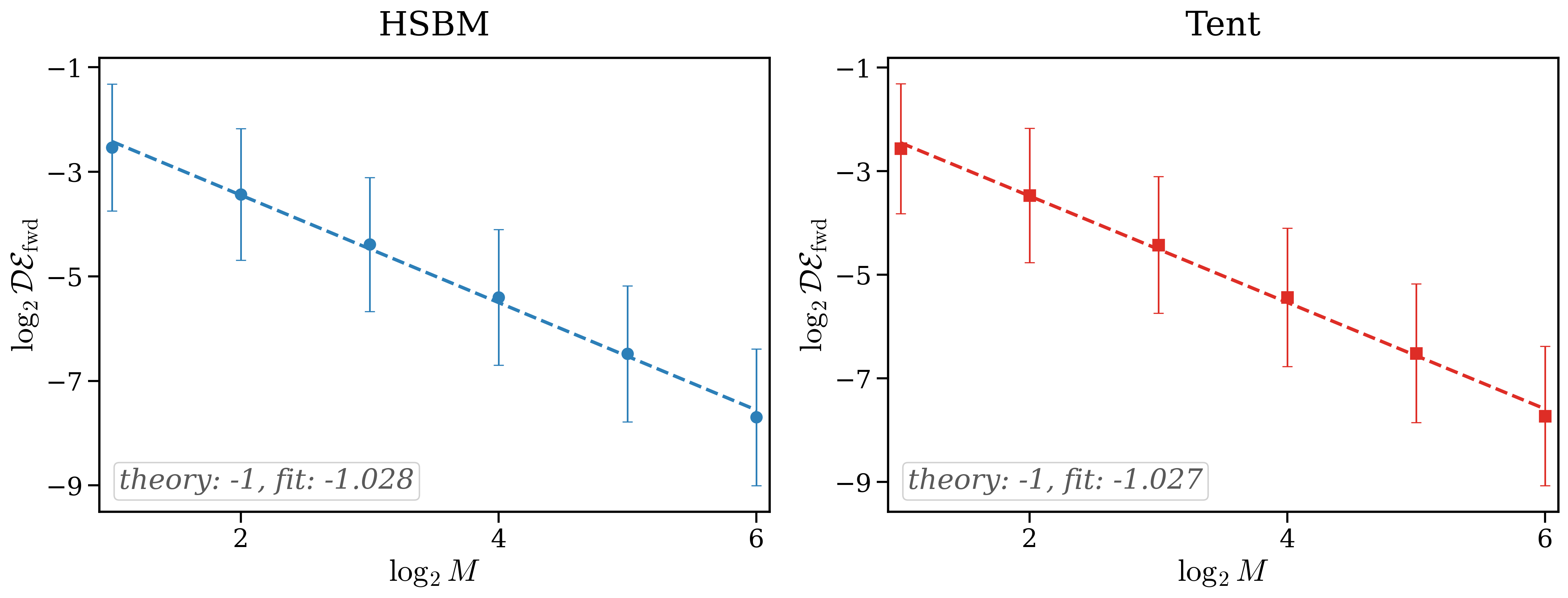}
    \caption{Temporal discretization error: log-log plot of $\mathcal{DE}_{\mathrm{fwd}}(M)$ versus the number of Euler steps $M$ for the HSBM graphon (left) and the tent graphon (right). The setup uses $n=2000$ nodes in the dense-graph setting ($\alpha_n=1$). The fitted log-log slopes are close to $-1$, consistent with the theoretical rate $\mathcal{O}(1/M)$.}
    \label{fig:discretization}
\end{figure}

\subsubsection{DTO--OTD Gradient Discrepancy}\label{subsec: exp dto otd}
We numerically verify the theoretical DTO--OTD error rates for hidden-state
gradients and parameter gradients established in
Theorems~\ref{theorem: hidden states gradients DTO and OTD}
and~\ref{theorem: parameter gradients DTO and OTD}. To this end, we define hidden-state gradient discrepancy (i.e., the left side of \eqref{eq: theoretical dto otd state}) by 
\begin{equation*}
\mathcal{DE}_{\mathrm{state}}(M) := \frac{1}{\sqrt{n}}\max_{m\in[M]} \normF{\VarDTOGradient^{[m]} - \VarGNDEHiddenState^{[m]}},
\end{equation*}
and parameter gradient discrepancy (i.e., the left side of \eqref{eq: theoretical dto otd param}) by 
\begin{equation*}
\mathcal{DE}_{\mathrm{param}}^{(\ell)}(M) := \frac{1}{n}\max_{m\in[M]} \normMax{\kappa\,\Dh^{[\ell,m]}\!\parens{\VarDTOGradient^{[m+1]}} - \kappa\,\Dh^{[\ell,m]}\!\parens{\VarGNDEHiddenState^{[m]}}},\quad \ell \in [L]. 
\end{equation*}
Note that Theorems~\ref{theorem: hidden states gradients DTO and OTD}
and~\ref{theorem: parameter gradients DTO and OTD} establish the theoretical
rates $\mathcal{DE}_{\mathrm{state}}(M)=\mathcal{O}(1/M)$ and
$\mathcal{DE}_{\mathrm{param}}^{(\ell)}(M)=\mathcal{O}(1/M^2)$, up to sparsity and logarithmic factors.

We test on a graph $\mathcal{G}_n$ with $n=200$ nodes. To obtain meaningful
gradient quantities, we fit the GNDE to a prescribed target dynamics and define
the loss with respect to the corresponding target trajectory. Recall that the GNDE vector field is parameterized by the three-hidden-layer GNN
described above. For the parameter-gradient discrepancy, we report the errors
for four parameter groups, denoted by \texttt{fc1}, \texttt{mid1},
\texttt{mid2}, and \texttt{fc2}. We consider the Allen--Cahn equation on a graph $\frac{d}{dt}\mX(t)=-\frac{\varepsilon^2}{n}\bigl(\mD\mX(t)-\mW\mX(t)\bigr)+\mX(t) - \mX(t)^{\odot 3}$, where $\mW$ is the adjacency matrix, $\mD$ is the degree matrix,
$\varepsilon=0.5$, and $\mX(t)^{\odot 3}$ denotes the element-wise cube of
$\mX(t)$. We approximate the graph Allen--Cahn target trajectory using the adaptive
Dormand--Prince solver with tolerance $10^{-8}$ and treat the resulting
terminal state as the reference terminal state, denoted by
$\VarGNDE_n^{\star}(T)$. When training the GNDE, we define the loss as the
error at the terminal time, i.e., $\normF{\VarGNDE_n(T)-\VarGNDE_n^{\star}(T)}^2$. Both the DTO and OTD paradigms use the Euler scheme with step size
$\kappa=T/M$, where
$M \in \{5,10,20,40,80,160,320,640\}$.

Figure~\ref{fig:dto_otd} presents the numerical results for both tent and HSBM graphons. The hidden-state discrepancy $\mathcal{DE}_{\mathrm{state}}(M)$ decays with slope essentially $-1$, in agreement with the $\mathcal{O}(1/M)$ rate from Theorem~\ref{theorem: hidden states gradients DTO and OTD}. The parameter gradient discrepancy $\mathcal{DE}_{\mathrm{param}}^{(\ell)}(M)$ decays with slopes approximately $-2$ across all considered parameter blocks, supporting the $\mathcal{O}(1/M^2)$ rate predicted by Theorem~\ref{theorem: parameter gradients DTO and OTD}.

\begin{figure}[t]
    \centering
    \includegraphics[width=\textwidth]{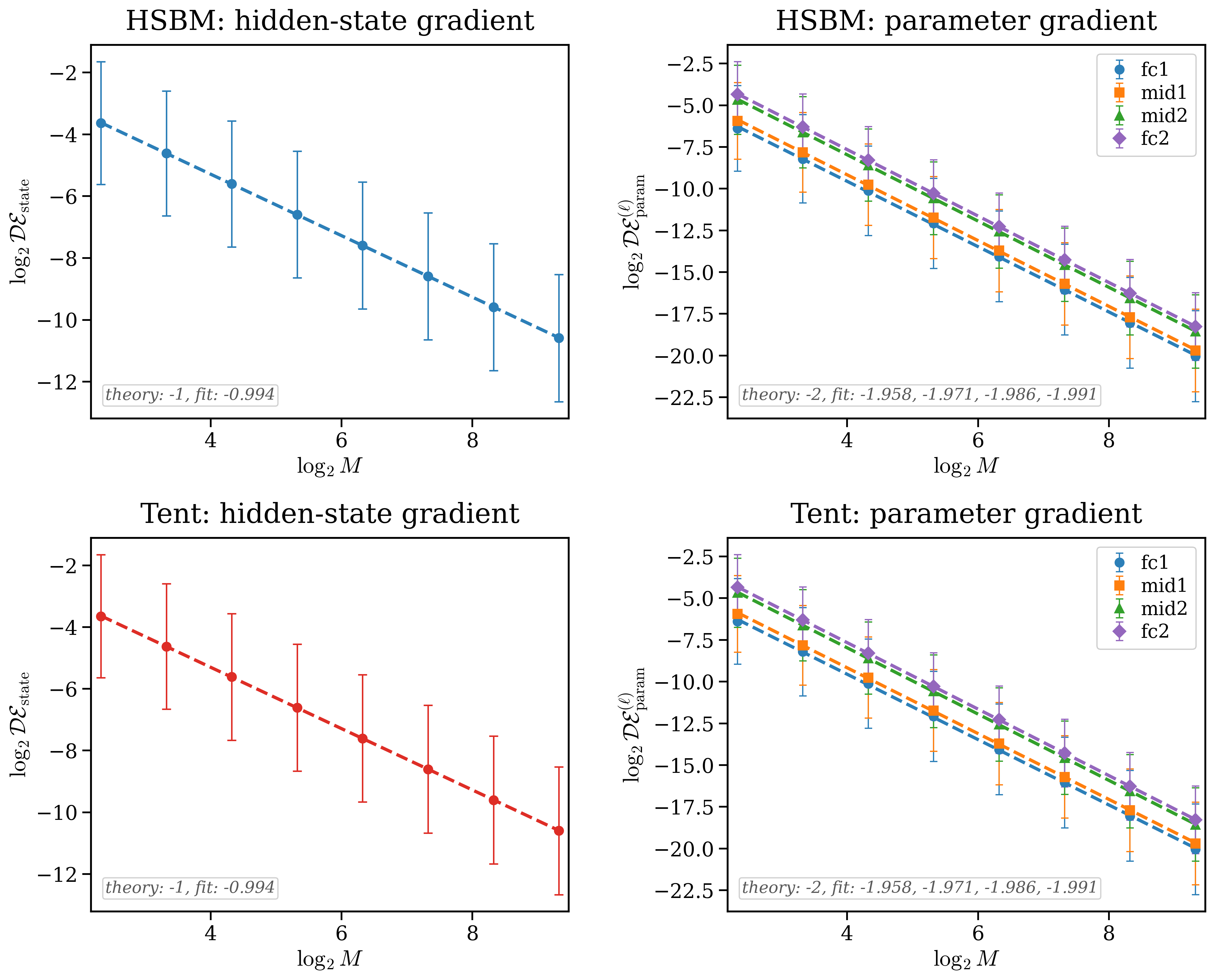}
    \caption{DTO--OTD gradient discrepancy: log-log plots of the DTO--OTD gradient discrepancies versus the number of Euler steps $M$ for the HSBM graphon
(top row) and the tent graphon (bottom row). Left column: hidden-state
gradient discrepancy $\mathcal{DE}_{\mathrm{state}}(M)$. The fitted log-log
slopes are close to $-1$, consistent with the theoretical rate
$\mathcal{O}(1/M)$. Right column: parameter-gradient discrepancy
$\mathcal{DE}_{\mathrm{param}}^{(\ell)}(M)$ for the four parameter blocks
\texttt{fc1}, \texttt{mid1}, \texttt{mid2}, and \texttt{fc2}. The fitted
log-log slopes are close to $-2$, consistent with the theoretical rate
$\mathcal{O}(1/M^2)$.}
    \label{fig:dto_otd}
\end{figure}

\subsection{Size transferability for learned dynamics}
\label{subsec: exp transfer}

We next test the practical implication of our theory: a GNDE trained on a small \emph{source} graph ($n_{\mathrm{src}}=16$ nodes) should transfer to much larger \emph{target} graphs (up to $n=1024$) sampled independently from the same graphon. Unlike the forward-convergence study of Section~\ref{subsec: exp forward}, which compares each sampled graph against a single large reference graph serving as a proxy for the graphon limit, here the source and target are two \emph{independent} random graphs of different sizes: they share neither latent positions nor edge realizations. Successful prediction therefore requires the trained GNDE to generalize to a genuinely new graph, not merely to reproduce its behavior on the (small) graph it was trained on.

\paragraph{Dynamics.}
We consider four graph-based dynamical systems, summarized in Table~\ref{tab:dynamics_transfer}. Three of them take the nonlocal diffusion form
\begin{equation}
\label{eq: dynamics transfer}
    \frac{d}{dt}[\mX]_i
    =
    c \sum_{j=1}^n
    \matL_{ij}\,
    f\big([\mX]_j-[\mX]_i\big)
    +
    g\big([\mX]_i\big),
\end{equation}
where $\matL_{ij} = \mW_{ij}/\sqrt{\mD_{ii}\mD_{jj}}$ is the symmetric normalized adjacency weight ($\mW$ is the adjacency, $\mD$ is the degree matrix, as in Section~\ref{subsec: graphs and graphons}), $f$ controls how a difference between neighbors contributes to the local change, and $g$ is an on-node reaction term. When $f$ is the identity, the sum is a standard graph-Laplacian diffusion: $\sum_j \matL_{ij}\big([\mX]_j-[\mX]_i\big)=-\big[(\mD_{\matL}-\matL)\mX\big]_i$, where $\mD_{\matL}:=\operatorname{diag}(\matL\mathbf{1})$ is the diagonal matrix of row sums of $\matL$. These models can be viewed as graph discretizations of nonlocal diffusion equations on graphons~\citep{medvedev2014nonlinear}; the four systems span standard application regimes and cover different kinds of nonlinearity in the GNDE velocity field: pure diffusion, diffusion with reaction, multiplicative coupling, and saturating influence. We describe each system together with its paired graphon and visualized pattern below.

\begin{table}[t]
\centering
\resizebox{\textwidth}{!}{%
\begin{tabular}{@{}lll@{}}
\toprule
\textbf{Dynamics} & \textbf{Equation on node $i$} & \textbf{Description} \\
\midrule
\textbf{Linear Heat} &
$\displaystyle
\frac{d}{dt} [\mX]_i = c
\sum_{j=1}^{n} \matL_{ij}\,\big([\mX]_j - [\mX]_i\big)
$ & \makecell[l]{Classical diffusion (graph Laplacian flow); \\[4pt] models linear heat propagation on graphs. \\[4pt] Coupling constant $c = 0.5$.} \\[4pt]

\midrule 

\textbf{Fisher--KPP} &
$\displaystyle
\frac{d}{dt} [\mX]_i =
c \sum_{j=1}^{n} \matL_{ij}\,\big([\mX]_j - [\mX]_i\big)
+ [\mX]_i\,(1 - [\mX]_i)
$ & \makecell[l]{Diffusion with logistic reaction; \\[4pt] models population growth with spatial spreading. \\[4pt] Coupling constant $c = 1.0$.} \\[4pt]

\midrule 

\textbf{SIS Epidemic} &
$\displaystyle
\frac{d}{dt} [\mX]_i =
-\beta\,[\mX]_i + \gamma\,(1-[\mX]_i) \sum_{j=1}^{n} \matL_{ij}\,[\mX]_j
$ & \makecell[l]{Epidemic spreading with recovery; \\ [4pt] models disease spread through infected neighbors. \\[4pt] Parameters $\beta = 1$, $\gamma = 5$.} \\[4pt]

\midrule 

\textbf{Consensus} &
$\displaystyle
\frac{d}{dt} [\mX]_i =
c \sum_{j=1}^{n}
\matL_{ij}\,
\frac{[\mX]_j - [\mX]_i}{1 + ([\mX]_j - [\mX]_i)^2}
$ & \makecell[l]{Nonlinear averaging with saturating influence; \\[4pt] models bounded-confidence consensus. \\[4pt] Coupling constant $c = 10$.} \\
\bottomrule
\end{tabular}%
}
\caption{Graph-based dynamical systems used in the size-transfer experiment. Here $\matL_{ij} = \mW_{ij}/\sqrt{\mD_{ii}\mD_{jj}}$ is the symmetric normalized adjacency weight and $[\mX]_i$ denotes the node feature at node $i$.}
\label{tab:dynamics_transfer}
\end{table}

Throughout, each node $i$ carries a latent position $u_i\in[0,1]$ (its graphon coordinate), and $u$ denotes a generic such position. We use four graphons: the HSBM and tent graphons considered in  Section \ref{sec: validation of theory}, together with two additional ones, the circular threshold graphon and the rank-1 power-law graphon (illustrated in Figure~\ref{fig:graphons_natural}). We pair each dynamics with a graphon whose structure makes its characteristic dynamical behavior clearly visible (diffusion, reaction--diffusion growth, epidemic propagation, or consensus formation), and visualize the resulting trajectories on a source graph with $n_{\mathrm{src}}=16$ nodes in Figure~\ref{fig:dynamics_comparison}.
\begin{itemize}
    \item \emph{Linear heat flow on graphs}~\citep{chung1997spectral,chung2007diffusion}. On the HSBM graphon (Figure~\ref{fig:dynamics_comparison}, first column), heat smooths within each community quickly and equalizes between communities more slowly, producing a hierarchical relaxation with two distinct timescales.
    \item \emph{Fisher--KPP reaction--diffusion}~\citep{fisher1937wave,kolmogorov1937study,du2020fisher,hoffman2019invasion} models population growth with spatial spreading. On the circular threshold graphon (Figure~\ref{fig:dynamics_comparison}, second column), a small population initially confined near latent position $u=0.125$ grows and spreads outward as a finite-speed traveling wavefront, illustrating how a species invades new spatial regions.
    \item \emph{Network SIS epidemic dynamics}~\citep{lajmanovich1976deterministic,vanmieghem2009virus,vanmieghem2011nintertwined} is the only system not in the diffusion form above; its multiplicative term $(1-[\mX]_i)\sum_j \matL_{ij}[\mX]_j$ pairs susceptibility with infection pressure from neighbors. On the rank-1 power-law graphon (Figure~\ref{fig:dynamics_comparison}, third column), well-connected hub nodes sustain higher endemic infection levels than peripheral ones, mirroring how epidemics concentrate in highly connected sub-populations.
    \item \emph{Nonlinear consensus with saturating influence}~\citep{hegselmann2002opinion,motsch2014heterophilious,brooks2024sigmoidal} is a smooth bounded-confidence variant of classical opinion dynamics. On the Tent graphon (Figure~\ref{fig:dynamics_comparison}, fourth column), trajectories starting from a spread of initial opinions contract toward a single common value, modeling how a group reaches agreement when mutual influence saturates with disagreement size.
\end{itemize}

\begin{figure}[t]
    \centering
    \includegraphics[width=\textwidth]{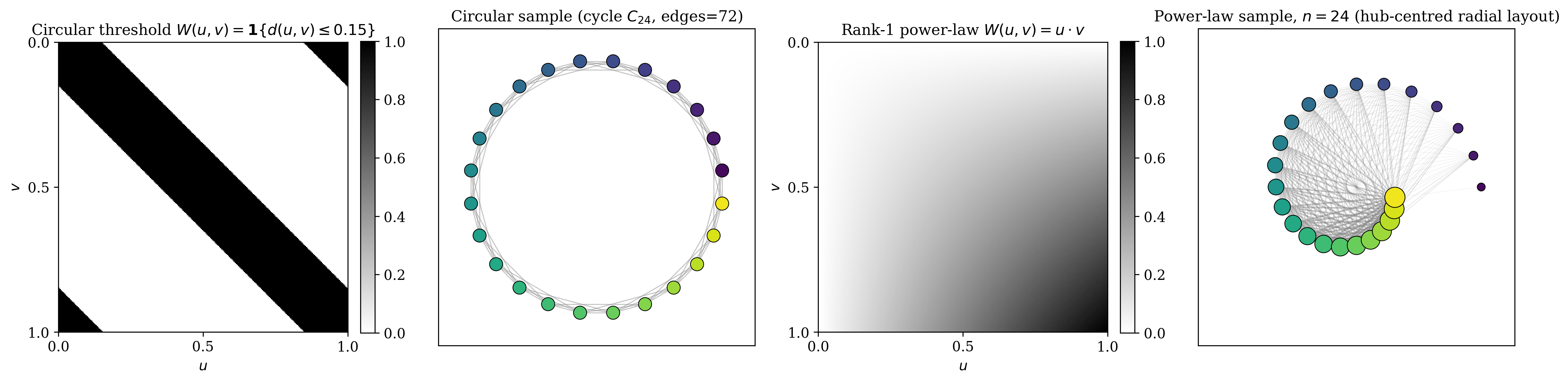}
    \caption{Two graphons used in the pattern-formation experiment, with representative sample graphs at $n=24$ drawn under the same Bernoulli sampling protocol as in Section \ref{sec: validation of theory}. \emph{Left:} circular threshold graphon $\tW(u,v)=\mathbf{1}\{d(u,v)\le 0.15\}$, where $d(u,v)=\min(|u-v|,1-|u-v|)$ is the periodic distance. \emph{Right:} rank-1 power-law graphon $\tW(u,v)=uv$, whose continuous graphon yields a weighted graph with edge weights $\tW(u_i, u_j)$.}
    \label{fig:graphons_natural}
\end{figure}

\paragraph{Initial conditions.}
In the transfer experiments, we take the node feature dimension $F$ to be $1$. Thus, the initial condition is a scalar field obtained by evaluating a continuous function $\mZ:I\to\mathbb{R}$ at the sampled node positions, i.e., $[\mX(0)]_i=\mZ(u_i)$, for $i\in[n]$. We use two families of initial conditions, each chosen for a different purpose. For the quantitative transfer experiment in Table~\ref{tab:transfer_error}, we use random Fourier polynomials $\mZ(u) = \sum_{k=1}^{10} a_k \sin(2\pi k u) + b_k \cos(2\pi k u)$, with i.i.d.\ Gaussian coefficients $(a_k, b_k)$, where errors are averaged over random draws. The same realization of $\mZ$ is reused across all graph sizes within a trial. For the qualitative visualization in Figure~\ref{fig:dynamics_comparison}, we instead pick a single initial condition tailored to each dynamics: a Fourier polynomial for linear heat and consensus; a sigmoidal localized seed for Fisher--KPP, $\mZ(u) = V\,\sigma(-s(u-u_{\rm thresh}))$ with $V=1$, $s=20$, $u_{\rm thresh}=0.125$,
which models a small population initially confined near $u=0.125$; and the constant $\mZ(u)=0.5$ for SIS, so that the graphon's row-sum profile alone shapes the resulting spatial structure.

\paragraph{GNDE architecture.}
In the theoretical analysis, we primarily take the graph shift operator in the graph convolutional operation to be the symmetric normalized adjacency matrix $\matL$. In the transfer experiments below, however, we choose the graph shift operator $\mS$ to match the linear graph operator appearing in the target dynamics. For the linear heat equation, Fisher--KPP dynamics, and nonlinear consensus dynamics, the linear graph-dependent term has the form $-c\widetilde{\matL}\mX$, where $\widetilde{\matL}:=\mD_{\matL}-\matL$. We therefore set $\mS=\widetilde{\matL}$ and use a GNN of one hidden layer with width $16$ and Tanh activation. For the SIS dynamics, the interaction is multiplicative rather than driven by pairwise differences, so we instead set $\mS=\matL$ and use a GNN of two hidden layers with width $16$ and the same Tanh activation. For all experiments, we use a polynomial filter with $K=2$.

\paragraph{Training protocol.}
For training, we use velocity regression rather than trajectory matching, in order to isolate approximation and transfer performance from the cost of differentiating through an ODE solver. Ground-truth position--velocity pairs $\{(\mX(t_m),\frac{d}{dt}\mX(t_m))\}_{m=0}^{2}$ are generated by approximating the target dynamics with an adaptive Dormand--Prince solver with tolerance $10^{-8}$, sampled at three uniformly spaced times $\{t_0,t_1,t_2\}$ in $[0,T]$. Let $\mX^{\langle q\rangle}$ denote the target trajectory generated from the $q$-th initial condition. We minimize the following loss
\begin{equation}
\label{eq: simfree loss}
    \mathcal{L}(\tH)
    =
    \frac{1}{3Q}
    \sum_{q=1}^{Q}
    \sum_{m=0}^{2}
    \left\|
    \mathrm{GNN}(\mX^{\langle q\rangle}(t_m);\mS,\tH)
    -
    \frac{d}{dt}\mX^{\langle q\rangle}(t_m)
    \right\|^2,
\end{equation}
with $Q=500$ initial conditions. We optimize the model using Adam~\citep{kingma2014adam} with learning rate $10^{-3}$ for up to $500$ epochs, with early stopping. Since each update only requires GNN forward evaluations and does not involve an inner ODE solve, this velocity-regression training procedure is substantially cheaper than trajectory-based training through a differentiable ODE solver such as \texttt{torchdiffeq}~\citep{chen2018neural}. %

\paragraph{Transfer protocol.}
With training in place, we evaluate transfer as follows. For each graphon--dynamics pair, we train on a source graph with $n_{\mathrm{src}}=16$ nodes and apply the learned parameters $\widehat{\tH}$ to independently sampled target graphs of size $n=32$, $64$, $128$, $256$, $512$, $1024$, between $2\times$ and $64\times$ larger than the source.
Both the source graph $\mathcal{G}_{n_{\mathrm{src}}}$ and the target graph $\mathcal{G}_n$ are drawn \emph{independently} from the graphon, by the same Bernoulli mechanism used in Section~\ref{sec: validation of theory}: latent positions $u_i\in[0,1]$ together with random edges generated from the graphon $\tW$. For the $\{0,1\}$-valued graphons (HSBM and circular threshold), the edges are unweighted with edge probability $\alpha_n\tW(u_i,u_j)$; for the weighted graphons (tent and rank-1 power-law), we retain the value $\tW(u_i,u_j)$ as an edge weight and impose sparsity through $\alpha_n$. We use a \emph{dense} regime ($\alpha_n=1$) and a \emph{sparse} regime ($\alpha_n=n^{-1/2}$ for the Tent, Circular, and Power-law graphons, $\alpha_n=0.1$ for the binary HSBM graphon), with surviving edges rescaled by $1/\alpha_n$. We define the transfer error to be
\begin{equation}
\label{eq: E transfer new}
    \mathcal{TE}(n)
    \;:=\;
    \frac{
    \sup_{t\in[0,T]}
    \big\|\VarGNDE_n(t;\widehat{\tH})-\mX_n^\star(t)\big\|
    }{
    \sup_{t\in[0,T]}
    \big\|\mX_n^\star(t)\big\|
    },
\end{equation}
where $\mX_n^\star(t)$ is the ground-truth solution on the target graph. Since $\mX_n^\star$ has no closed form, in our experiments we approximate it with the adaptive Dormand--Prince solver with tolerance $10^{-6}$; the GNDE solution $\VarGNDE_n(\cdot;\widehat{\tH})$ is numerically solved by a Runge-Kutta method of order four with step size $0.02$, and each supremum of $t\in[0,T]$ is taken as the maximum over $50$ uniformly spaced output times $t_m\in[0,T]$.

\paragraph{Results.} We report $\mathcal{TE}(n)$ in Table~\ref{tab:transfer_error}
as mean $\pm$ std over 10 trials, each redrawing the target graph, initial
condition, and model initialization. Its numerator is the absolute prediction
error bounded by (I)+(II)+(III) in \eqref{eq:triangle_decomposition}, and the
denominator $\sup_t\|\mX_n^\star(t)\|=\mathcal{O}(\sqrt{n})$ matches that scaling, so
rows are comparable across $n$: the source row ($n_{\mathrm{src}}=16$) is the
relative training error~(II), each transfer row is the relative total, and
$\mathcal{TE}(n)$ comparable to the source value means the size-induced terms
(I) and (III) do not dominate~(II)

\emph{Dense regime ($\alpha_n=1$): transfer succeeds across graphons.} For every dynamics graphon pair, the transfer error on the larger target graph stays comparable to the source error at $n_{\mathrm{src}}=16$ (e.g., Fisher--KPP on Tent: $0.018$ at $n=1024$ versus $0.022$ at $n_{\mathrm{src}}$). The comparable source and target errors indicate that transfer to larger dense graphs does not produce a noticeable size-induced degradation. The main exception is SIS on the rank-1 power-law graphon, where the multiplicative interaction $(\mathbf{1}-\mX)\matL\mX$ amplifies the effect of smooth degree heterogeneity and the transfer error rises modestly from $0.011$ at $n_{\mathrm{src}}$ to $0.024$ at $n=1024$.

\emph{Sparse regime: sparsification effects depend on the graphon.} On the Tent, circular threshold, and rank-1 power-law graphons, using $\alpha_n=n^{-1/2}$ generally increases the transfer error compared with the dense setting $\alpha_n=1$, especially at smaller transfer sizes. As $n$ grows, however, the sparse-regime errors move closer to the corresponding dense-regime errors, suggesting that the effect of edge subsampling weakens at larger graph sizes. For the HSBM graphon, the dense graph ($\alpha_n=1$) is the deterministic two-level block pattern and carries no edge-sampling randomness. We therefore probe sparsity using a fixed keep-probability $0.1$ (discarding around $90\%$ of edges); the transfer error is elevated at small $n$ but returns toward the dense value as $n$ grows.

\begin{table}[ht]
\centering\scriptsize
\caption{Transfer error of GNDEs trained on $n_{\mathrm{src}}=16$ and transferred to $n \in \{128, 256, 512, 1024\}$ across four graphon families. Dense regime $\alpha_n=1$; sparse regime $\alpha_n=n^{-1/2}$ for Tent, Circular, and Power-law, and a fixed $\alpha_n=0.1$ for the HSBM graphon.}
\label{tab:transfer_error}
\setlength{\tabcolsep}{2.5pt}
\resizebox{\textwidth}{!}{%
\begin{tabular}{llcccccccc}
\toprule
& & \multicolumn{2}{c}{\textbf{Tent}} & \multicolumn{2}{c}{\textbf{HSBM}} & \multicolumn{2}{c}{\textbf{Circular}} & \multicolumn{2}{c}{\textbf{Power-law}} \\
\cmidrule(lr){3-4}\cmidrule(lr){5-6}\cmidrule(lr){7-8}\cmidrule(lr){9-10}
Dynamics & $n$ & $\alpha{=}1$ & $\alpha{=}n^{-1/2}$ & $\alpha{=}1$ & $\alpha{=}0.1$ & $\alpha{=}1$ & $\alpha{=}n^{-1/2}$ & $\alpha{=}1$ & $\alpha{=}n^{-1/2}$ \\
\midrule
\textbf{Linear Heat} & 16 (src) & \multicolumn{2}{c}{$0.004 \pm 0.002$} & \multicolumn{2}{c}{$0.004 \pm 0.002$} & \multicolumn{2}{c}{$0.004 \pm 0.002$} & \multicolumn{2}{c}{$0.004 \pm 0.002$} \\
 & 128 & $0.003_{\pm 0.001}$ & $0.006_{\pm 0.004}$ & $0.004_{\pm 0.004}$ & $0.006_{\pm 0.003}$ & $0.003_{\pm 0.002}$ & $0.005_{\pm 0.003}$ & $0.004_{\pm 0.002}$ & $0.005_{\pm 0.003}$ \\
 & 256 & $0.003_{\pm 0.001}$ & $0.005_{\pm 0.004}$ & $0.004_{\pm 0.004}$ & $0.005_{\pm 0.004}$ & $0.003_{\pm 0.002}$ & $0.005_{\pm 0.003}$ & $0.004_{\pm 0.002}$ & $0.004_{\pm 0.003}$ \\
 & 512 & $0.003_{\pm 0.001}$ & $0.004_{\pm 0.003}$ & $0.004_{\pm 0.004}$ & $0.004_{\pm 0.004}$ & $0.003_{\pm 0.002}$ & $0.004_{\pm 0.002}$ & $0.004_{\pm 0.002}$ & $0.004_{\pm 0.002}$ \\
 & 1024 & $0.003_{\pm 0.001}$ & $0.004_{\pm 0.003}$ & $0.004_{\pm 0.004}$ & $0.004_{\pm 0.004}$ & $0.003_{\pm 0.002}$ & $0.004_{\pm 0.002}$ & $0.004_{\pm 0.002}$ & $0.004_{\pm 0.002}$ \\
\midrule
\textbf{Fisher--KPP} & 16 (src) & \multicolumn{2}{c}{$0.022 \pm 0.011$} & \multicolumn{2}{c}{$0.026 \pm 0.013$} & \multicolumn{2}{c}{$0.032 \pm 0.014$} & \multicolumn{2}{c}{$0.029 \pm 0.012$} \\
 & 128 & $0.019_{\pm 0.009}$ & $0.023_{\pm 0.009}$ & $0.022_{\pm 0.010}$ & $0.027_{\pm 0.010}$ & $0.026_{\pm 0.010}$ & $0.039_{\pm 0.011}$ & $0.027_{\pm 0.009}$ & $0.029_{\pm 0.009}$ \\
 & 256 & $0.018_{\pm 0.008}$ & $0.021_{\pm 0.008}$ & $0.021_{\pm 0.009}$ & $0.023_{\pm 0.009}$ & $0.025_{\pm 0.010}$ & $0.033_{\pm 0.010}$ & $0.027_{\pm 0.009}$ & $0.028_{\pm 0.009}$ \\
 & 512 & $0.018_{\pm 0.008}$ & $0.020_{\pm 0.008}$ & $0.021_{\pm 0.009}$ & $0.022_{\pm 0.009}$ & $0.024_{\pm 0.010}$ & $0.029_{\pm 0.010}$ & $0.027_{\pm 0.009}$ & $0.027_{\pm 0.009}$ \\
 & 1024 & $0.018_{\pm 0.008}$ & $0.019_{\pm 0.008}$ & $0.021_{\pm 0.009}$ & $0.021_{\pm 0.009}$ & $0.024_{\pm 0.010}$ & $0.027_{\pm 0.010}$ & $0.027_{\pm 0.009}$ & $0.027_{\pm 0.009}$ \\
\midrule
\textbf{SIS} & 16 (src) & \multicolumn{2}{c}{$0.016 \pm 0.010$} & \multicolumn{2}{c}{$0.012 \pm 0.010$} & \multicolumn{2}{c}{$0.011 \pm 0.009$} & \multicolumn{2}{c}{$0.011 \pm 0.007$} \\
 & 128 & $0.016_{\pm 0.010}$ & $0.025_{\pm 0.006}$ & $0.022_{\pm 0.017}$ & $0.053_{\pm 0.031}$ & $0.008_{\pm 0.008}$ & $0.053_{\pm 0.025}$ & $0.023_{\pm 0.016}$ & $0.065_{\pm 0.028}$ \\
 & 256 & $0.016_{\pm 0.010}$ & $0.022_{\pm 0.007}$ & $0.021_{\pm 0.016}$ & $0.041_{\pm 0.022}$ & $0.008_{\pm 0.008}$ & $0.029_{\pm 0.012}$ & $0.023_{\pm 0.015}$ & $0.061_{\pm 0.025}$ \\
 & 512 & $0.016_{\pm 0.010}$ & $0.021_{\pm 0.007}$ & $0.019_{\pm 0.016}$ & $0.032_{\pm 0.016}$ & $0.008_{\pm 0.008}$ & $0.019_{\pm 0.010}$ & $0.023_{\pm 0.014}$ & $0.054_{\pm 0.022}$ \\
 & 1024 & $0.017_{\pm 0.010}$ & $0.020_{\pm 0.008}$ & $0.019_{\pm 0.017}$ & $0.027_{\pm 0.015}$ & $0.008_{\pm 0.008}$ & $0.013_{\pm 0.008}$ & $0.024_{\pm 0.015}$ & $0.050_{\pm 0.020}$ \\
\midrule
\textbf{Consensus} & 16 (src) & \multicolumn{2}{c}{$0.213 \pm 0.055$} & \multicolumn{2}{c}{$0.201 \pm 0.059$} & \multicolumn{2}{c}{$0.213 \pm 0.041$} & \multicolumn{2}{c}{$0.193 \pm 0.055$} \\
 & 128 & $0.165_{\pm 0.051}$ & $0.199_{\pm 0.035}$ & $0.183_{\pm 0.066}$ & $0.214_{\pm 0.041}$ & $0.206_{\pm 0.030}$ & $0.208_{\pm 0.021}$ & $0.157_{\pm 0.047}$ & $0.182_{\pm 0.032}$ \\
 & 256 & $0.156_{\pm 0.049}$ & $0.187_{\pm 0.035}$ & $0.174_{\pm 0.064}$ & $0.194_{\pm 0.049}$ & $0.204_{\pm 0.027}$ & $0.210_{\pm 0.016}$ & $0.147_{\pm 0.047}$ & $0.171_{\pm 0.035}$ \\
 & 512 & $0.154_{\pm 0.049}$ & $0.177_{\pm 0.038}$ & $0.172_{\pm 0.065}$ & $0.183_{\pm 0.056}$ & $0.203_{\pm 0.028}$ & $0.211_{\pm 0.016}$ & $0.147_{\pm 0.043}$ & $0.163_{\pm 0.036}$ \\
 & 1024 & $0.151_{\pm 0.048}$ & $0.171_{\pm 0.040}$ & $0.170_{\pm 0.066}$ & $0.177_{\pm 0.060}$ & $0.199_{\pm 0.028}$ & $0.208_{\pm 0.018}$ & $0.145_{\pm 0.043}$ & $0.158_{\pm 0.037}$ \\
\bottomrule
\end{tabular}%
}
\end{table}

Finally, Figure~\ref{fig:dynamics_comparison} illustrates that, on the source graph ($n_{\mathrm{src}}=16$), the GNDE prediction visually matches the ground truth for each of the four dynamics. The training and transfer protocol is identical to the experiment above; the only difference is that we rescale the per-dynamics coupling in~\eqref{eq: dynamics transfer} ($c=3$ for linear heat and SIS, $c=15$ for consensus, $d=0.6$, $r=0.5$ for Fisher--KPP) so that each dynamics' characteristic behavior develops clearly within the simulated time window.

\paragraph{Computational platform.}
All experiments were run on NVIDIA RTX A4000 GPUs with 16\,GB of memory.

\begin{figure}[t]
    \centering
    \includegraphics[width=\textwidth]{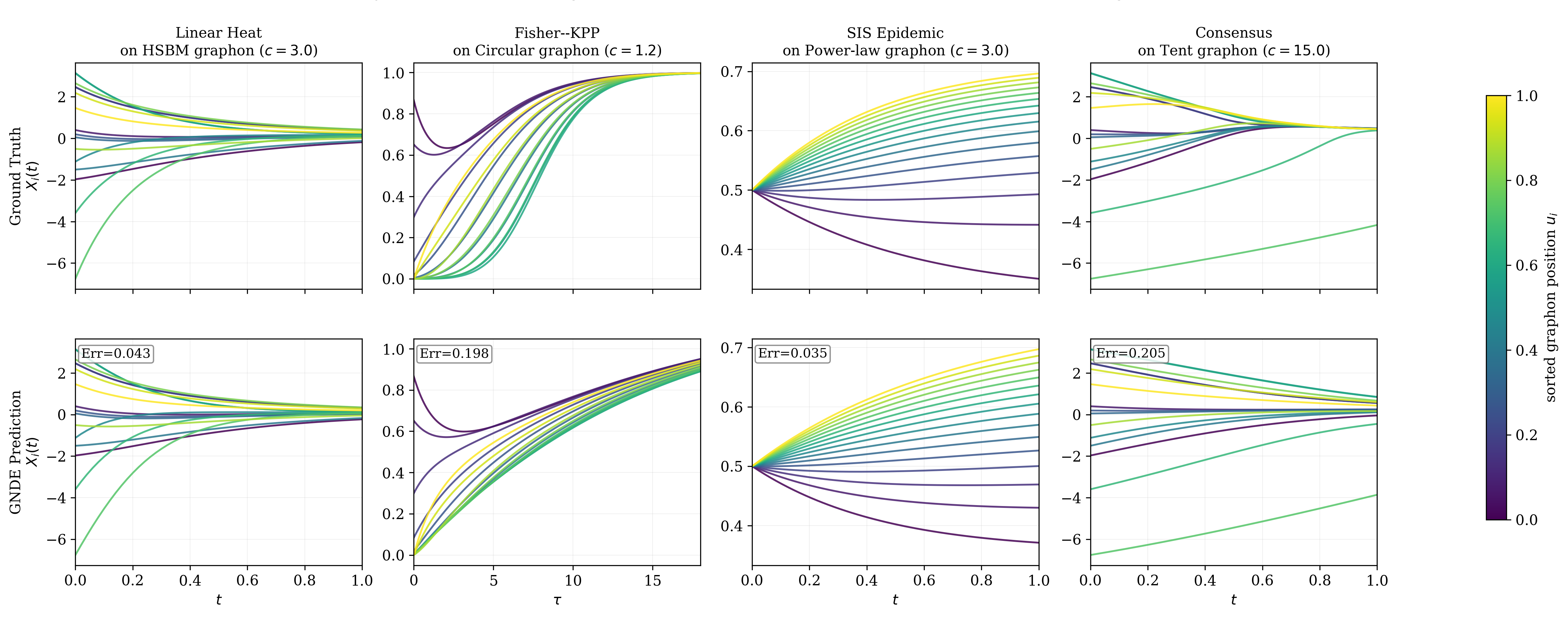}
    \caption{Ground truth versus GNDE prediction on source graphs with $n_{\mathrm{src}}=16$. Each dynamics is paired with a graphon that highlights its characteristic dynamical behavior. Top row: ground-truth trajectories. Bottom row: GNDE predictions, with relative trajectory error annotated. Heat, SIS, and consensus use $t\in[0,1]$; Fisher--KPP on the circular graphon uses rescaled time $\tau\in[0,18]$.}
    \label{fig:dynamics_comparison}
\end{figure}

\section*{Acknowledgment} 
M.~Yan acknowledges support from an AMS-Simons travel grant. S.~Tang is supported by NSF DMS CAREER Grant No.~2340631. The authors thank Mr.\ Fangqian Zhang for his careful proofreading of the manuscript
and helpful discussions.

\bibliography{utils/ref}
\bibliographystyle{plainnat}

\appendix

\section{Notation Conventions}
Because the vector fields of our GNDEs and Graphon-NDEs are parameterized by multi-layer architectures, we explicitly track their hidden states using the following conventions. Superscripts denote the layer and time indices. Parentheses $(\ell, t)$ indicate continuous time $t$, while square brackets $[\ell, m]$ indicate discrete steps $t_m$. Base inputs ($\ell=0$) frequently omit the zero for brevity (e.g., $\VarGNDE^{(0,t)}\equiv\VarGNDE\myt\equiv\VarGNDE(t)$, $\VarGraphonNDE^{(0,t)}\equiv \VarGraphonNDE(\cdot,t)$, and $\VarGNDE^{[0,m]}\equiv\VarGNDE^{[m]}$). The $\ell$-th layer outputs of the forward passes $\GNN\parens{\VarGNDE(t);\matL,\tH(t)}$, $\GraphonNN(\VarGraphonNDE(u,t);\opLP,\tH(t))$, and $\GNN(\VarGNDE^{[m]};\matL,\tH(t_{m}))$ are denoted by $\VarGNDE\mylt$, $\VarGraphonNDE\mylt$, and $\VarGNDE^{[\ell,m]}$, respectively.

For the layer-wise operations, the two-variable graph and graphon convolutional operators, $\phi(\vh, \mX)$ and $\Phi(\vh, \mathcal{X})$, are introduced in the main text (i.e., \eqs\eqref{def: operator phi} and \eqref{def: operator Phi}). By fixing one variable, we define their associated linear operators $\phi_{\vh}$ and $\phi_{\mX}$ (defined in \eqref{def: graph convolutional operator phi h} and \eqref{def: graph convolutional operator phi X}), as well as the continuous counterparts $\Phi_{\vh}$ and $\Phi_{\mathcal{X}}$ (defined in \eqref{def: Phih} and \eqref{def: frakx}). Detailed properties of these single-variable operators, including their adjoints, are provided in Appendices~\ref{appendix: graph convolutional operators} and \ref{appendix: graphon convolutional operators}.

To layer-wisely notate the backward gradients (i.e., the layer-wise output of the adjoint operator in the adjoint systems \eqref{adjoint GNDE: Y}, \eqref{adjoint GNDE: z}, \eqref{adjoint Grahpon-NDE: Y}, \eqref{adjoint Graphon-NDE: H}), we reverse the layer indexing: $\ell=L$ denotes the initial input to the backward pass, and $\ell=0$ denotes the final output. Analogous to the forward states, we omit the $L$ index for brevity: $\VarGNDEHiddenState^{(L,t)}\equiv\VarGNDEHiddenState\myt\equiv\VarGNDEHiddenState(t)$, $\VarGraphonNDEHiddenState^{(L,t)}\equiv \VarGraphonNDEHiddenState(\cdot,t)$, and $\VarGNDEHiddenState^{[L,m]}\equiv\VarGNDEHiddenState^{[m]}$. These represent the gradients of the forward pass output (i.e., the inputs to the adjoint operators). Under this convention, applying the adjoint of the network's Gateaux derivative to the terminal state yields the gradient with respect to the initial forward state:
$$
\parens{\gd{\GNN}{\VarGNDE^{(0,t)}}}^*(\VarGNDEHiddenState^{(L,t)})=\VarGNDEHiddenState^{(0,t)},\ \parens{\gd{\GraphonNN}{\VarGraphonNDE^{(0,t)}}}^*(\VarGraphonNDEHiddenState^{(L,t)})=\VarGraphonNDEHiddenState^{(0,t)},\ \parens{\gd{\GNN}{\VarGNDE^{[0,m]}}}^*(\VarGNDEHiddenState^{[L,m]})=\VarGNDEHiddenState^{[0,m]}.
$$
The explicit recursive chain-rule derivations for these backward states, along with the gradients with respect to the filter weights, are detailed in \eqref{adjoint operator: GNDE to filters} of Appendix~\ref{appendix: graph convolutional operators} and \eqref{adjoint operator: Graphon-NDE to filters} of Appendix~\ref{appendix: graphon convolutional operators}.

\section{Boundedness and Temporally Lipschitz of Graphon-NDE solutions} 
\begin{proposition}[Boundedness of Graphon-NDE solutions]\label{proposition: norm of X ell is bounded by norm of X}
Suppose that \ConvFilterLipschitz, \SigmaLipschitz, and \dPBoundedBelow\  hold. Let $\VarGraphonNDE$ be the solution of the Graphon-NDE \eqref{Graphon-NDE}. Let $\tildeVarGraphonNDE\mylt,\VarGraphonNDE\mylt$ be defined in \eq{def: updating formula Graphon-NN}. Then for each $\ell\in\mathbb{Z}_{L+1}$, it holds that
\begin{align}\label{eq: norm of X ell is bounded by norm of X}
    &\norm{ \VarGraphonNDE\myl}_{\spaceC} \leq  \CXmax\myl:=(F\Csys[0]\Ch {L_\sigma})^{\ell}\norm{\VarGraphonNDE}_{\spaceC},\\
    \Or{&\norm{\tildeVarGraphonNDE\myl}_{\spaceC}}\nonumber
\end{align} 
where constants $\Ch$ and $\Csys[0]$ are defined in \eqref{def: constant Ch} and \eqref{def: C0W}, respectively.

\end{proposition}
\begin{proof}
    It follows from updating formula \eqref{def: updating formula Graphon-NN} that $\VarGraphonNDE\mylt=\sigma(\Phi_{\vh\mylt}(\VarGraphonNDE^{(\ell-1,t)}))$, $\ell\in[L]$, $t\in[0,T]$, where we adopt the notations of graphon convolutional operators in Appendix \ref{appendix: graphon convolutional operators}. Then we obtain from \SigmaLipschitz\ and Lemma \ref{lemma: norm of operator mathfrak H and adjoint operator widehat mathfrak H} that $$\norm{\VarGraphonNDE\mylt }_{\spaceB}\leq  \parens{F\Csys[0]\Ch {L_\sigma}}\|\VarGraphonNDE^{(\ell-1,t)}\|_{\spaceB},$$ which with a recursion and the norm defined in $\spaceC$ yields
\begin{equation*}
    \norm{ \VarGraphonNDE\mylt }_{\spaceB} \leq (F\Csys[0]\Ch{L_\sigma})^{\ell}\|\VarGraphonNDE(\cdot, t)\|_{\spaceB}\leq (F\Csys[0]\Ch{L_\sigma})^{\ell}\|\VarGraphonNDE\|_{\spaceC},
\end{equation*}
Since $\VarGraphonNDE^{(0)}=\VarGraphonNDE$, \eq{eq: norm of X ell is bounded by norm of X} trivially holds when $\ell=0$. Therefore, \eq{eq: norm of X ell is bounded by norm of X} holds for all $\ell\in\mathbb{Z}_{L+1}$. Noting that \eq{def: updating formula Graphon-NN} can be rewritten as  $\tildeVarGraphonNDE\mylt=\Phi_{\vh\mylt}(\sigma(\tildeVarGraphonNDE^{(\ell-1,t)}))$, $\ell\in[L]$, the result for $\tildeVarGraphonNDE\mylt$ can be derived with a similar argument as above. 
\end{proof}

\begin{proposition}[Temporally Lipschitz of Graphon-NDE solutions]\label{proposition: solution Xfell is piecewise Lipschitz continuous about t}
    Suppose that \ConvFilterLipschitz, \SigmaLipschitz, and \dPBoundedBelow\  hold. Let $\VarGraphonNDE=[\VarGraphonNDE_f:f\in[F]]$ be the solution of the Graphon-NDE \eqref{Graphon-NDE}. Let $\tildeVarGraphonNDE\mylt,\VarGraphonNDE\mylt$ be defined in \eq{def: updating formula Graphon-NN}. Then for each $\ell\in\mathbb{Z}_{L+1}$, it holds that 
    \begin{align}\label{eq: Xfell is Lipschitz continuous}
        &\norm{\VarGraphonNDE\mylt[\ell][t_1]_f-\VarGraphonNDE\mylt[\ell][t_2]_f}_{\spaceb}\leq \CXlip\myl|t_1-t_2|,\quad \forall t_1,t_2\in[0,T],\ \forall f\in[F], \\
        \Or{&\norm{\tildeVarGraphonNDE\mylt[\ell][t_1]_f-\tildeVarGraphonNDE\mylt[\ell][t_2]_f}_{\spaceb}}\nonumber
    \end{align}
    where constant $\CXlip\myl$ is defined recursively as $\CXlip\myl:=\Csys[0]{L_\sigma}\parens{\sqrt{F}\Liph\CXmax\myl[\ell-1]+F\Ch\CXlip\myl[\ell-1]}$, 
    and $\CXlip^{(0)}:=\CXmax^{(L)}$.
\end{proposition}
\begin{proof}
We begin with the case of $\ell=0$. Since $\VarGraphonNDE$ is the solution of Graphon-NDE \eqref{Graphon-NDE}, it holds that 
\begin{equation}\label{in the proof grahon-nde}
    \frac{\partial}{\partial t} \VarGraphonNDE(u,t) = \GraphonNN(\VarGraphonNDE(u,t);\opLP,\tH(t)).
\end{equation}
According to Theorem \ref{theorem: well-posedness}, the unique solution $\VarGraphonNDE$ belongs to $C^1\parens{[0,T];\spaceB}$, which implies that $\frac{\partial}{\partial t} \VarGraphonNDE(u,t)$ is a continuous function about $t$. Therefore, it holds that 
\begin{equation}\label{in the proof Xf is general Lipschitz continuous}
    \left|\VarGraphonNDE_f(u,t_1)-\VarGraphonNDE_f(u,t_2)\right|\leq \supT\left|\frac{\partial }{\partial t}\VarGraphonNDE_f(u,t)\right||t_1-t_2|.
\end{equation}
We obtain from \eq{in the proof grahon-nde} that
\begin{equation}\label{in the proof sup max sup partial Xf partial t}
    \sup_{u\in I}\max_{f\in[F]}\supT\left|\frac{\partial }{\partial t}\VarGraphonNDE_f(u,t)\right|\leq \supT\norm{\GraphonNN(\VarGraphonNDE(\cdot,t);\opLP,\tH(t))}_{\spaceB}.
\end{equation}
It follows from Proposition \ref{proposition: norm of X ell is bounded by norm of X} that 
\begin{equation}\label{in the proof sup Pho leq H infinity T L X C}
\supT\norm{\GraphonNN(\VarGraphonNDE(\cdot,t);\opLP,\tH(t))}_{\spaceB}\leq \CXmax^{(L)}.
\end{equation}
Collecting estimates \eqref{in the proof Xf is general Lipschitz continuous}, \eqref{in the proof sup max sup partial Xf partial t} and \eqref{in the proof sup Pho leq H infinity T L X C} yields $\left|\VarGraphonNDE_f(u,t_1)-\VarGraphonNDE_f(u,t_2)\right|\leq \CXlip^{(0)}|t_1-t_2|$, $u\in I$. This verifies \eq{eq: Xfell is Lipschitz continuous} when $\ell=0$. 

We next assume that, for $\ell\in[L]$, $f\in[F]$, 
\begin{equation}\label{in the proof Xfell-1 is Lipschitz continuous}
\norm{\VarGraphonNDE\mylt[\ell-1][t_1]_f-\VarGraphonNDE\mylt[\ell-1][t_2]_f}_{\spaceb}\leq \CXlip\myl[\ell-1]|t_1-t_2|,\quad u\in I, t_1,t_2\in [0,T], 
\end{equation}
holds with a constant $\CXlip\myl[\ell-1]$ and will prove \eq{eq: Xfell is Lipschitz continuous}. Note that
\begin{align*}
    &\norm{\VarGraphonNDE\mylt[\ell][t_1]_f-\VarGraphonNDE\mylt[\ell][t_2]_f}_{\spaceb} =\norm{\left[\sigma\parens{ \Phi_{\vh\mylt[\ell][t_1]}(\VarGraphonNDE^{(\ell-1,t_1)})}\right]_f-\left[\sigma\parens{ \Phi_{\vh\mylt[\ell][t_2]}(\VarGraphonNDE^{(\ell-1,t_2)})}\right]_f}_{\spaceb}\stepjust{\eq{def: updating formula Graphon-NN}}\\
    &\leq {L_\sigma}\Csys[0]\parens{\Liph\sqrt{F} \norm{\VarGraphonNDE\myl[\ell-1]}_{\spaceC}+F\Ch \mathrm{Lip}(\VarGraphonNDE\myl[\ell-1])}|t_1-t_2| \stepjust{\SigmaLipschitz\ and Lemma \ref{lemma: operator frakH is lipschitz continuous about t}}\\
    &\leq \CXlip\myl\abs{t_1-t_2}\stepjust{Proposition \ref{proposition: norm of X ell is bounded by norm of X}, \eq{in the proof Xfell-1 is Lipschitz continuous} and definition of $\CXlip\myl$}
\end{align*}
This proves \eq{eq: Xfell is Lipschitz continuous}. The result for $\tildeVarGraphonNDE\mylt_f$ can be shown in a similar manner. 

\end{proof}

\section{Infinite-node Convergence of GNDEs}\label{appendix: Infinite-node Convergence of GNDEs}

\begin{proposition}\label{proposition: epsilon ell+1 leq epsilon ell + Delta}
Suppose that \ConvFilterLipschitz, \SigmaLipschitz, and \dPBoundedBelow\  hold. Let $\VarGNDE$ and $\VarGraphonNDE$ be the solutions of GNDE \eqref{GNDE} and Graphon-NDE \eqref{Graphon-NDE}, respectively. For each $\ell\in\mathbb{Z}_{L+1}$, define 
\begin{align}
\epsilon\mylt &:= \MSE(\VarGraphonNDE\mylt,\VarGNDE\mylt)=\frac{1}{\sqrt{n}}\norm{\SamplingOperator(\VarGraphonNDE\mylt)-\VarGNDE\mylt}_{\mathrm{F}},\label{def: epsilon ell}\\
\widetilde{\epsilon}\mylt &:=\MSE(\tildeVarGraphonNDE\mylt,\tildeVarGNDE\mylt)=\frac{1}{\sqrt{n}}\normF{\SamplingOperator(\tildeVarGraphonNDE\mylt)-\tildeVarGNDE\mylt}.\label{def: widetilde epsilon ell}
\end{align}
Then for $\ell\in\mathbb{Z}_{L}$ and $t\in[0,T]$, there holds
   \begin{equation*}
       \epsilon^{(\ell+1,t)}\leq {L_\sigma}\widetilde{\epsilon}^{(\ell+1,t)}\leq \parens{{L_\sigma}FK\Ch}\epsilon\mylt + \Delta_{1}\mylt + \Delta_{2}\mylt,
   \end{equation*} 
   where 
   \begin{align}
    \Delta_{1}\mylt&:={L_\sigma}FK^2\Ch\norm{\matL-\matLUn}_2\norm{\VarGraphonNDE\mylt}_{\spaceB},\label{def: Delta1 ell t}\\
    \Delta_{2}\mylt&:={L_\sigma}FK\Ch\norm{\opLUn}_{\spaceb\to\spaceb}^K\sqrt{ \sum_{g=1}^F\sum_{k=1}^{K}\sum_{s=0}^{k-1}\norm{(\opLUn-\mathcal{L}_{P})\parens{\opLP^{k-1-s}\VarGraphonNDE\mylt_g}}_{\spaceb}^2}.\label{def: Delta2 ell t}
   \end{align}
\end{proposition}

\begin{proof}
Note that
\begin{align*}
    \epsilon^{(\ell+1,t)}&=\MSE\parens{\VarGraphonNDE^{(\ell+1,t)},\VarGNDE\mylt[\ell+1]}\\
    &=\MSE\parens{\sigma(\Phi_{\vh\mylt[\ell+1]}(\VarGraphonNDE\mylt),\sigma(\phi_{\vh\mylt[\ell+1]}(\VarGNDE\mylt))}\stepjust{\eqs\eqref{def: updating formula Graphon-NN} and \eqref{def: updating formula gnn}}\\
    &\leq {L_\sigma}\MSE\parens{\Phi_{\vh^{(\ell+1,t)}}(\VarGraphonNDE\mylt),\phi_{\vh\mylt[\ell+1]}(\VarGNDE\mylt)}={L_\sigma}\widetilde{\epsilon}^{(\ell+1,t)}\stepjust{\SigmaLipschitz}\\
    &\leq \parens{{L_\sigma}FK\Ch} \MSE\parens{\VarGraphonNDE\mylt,\VarGNDE\mylt} + \Delta_1\mylt + \Delta_2\mylt\stepjust{Lemma \ref{lemma: MSE hatfrakH function Z - hatfrakh matrix Z}}\\
    &= \parens{{L_\sigma}FK\Ch} \epsilon\mylt + \Delta_1\mylt + \Delta_2\mylt
\end{align*}    
\end{proof}

\begin{lemma}\label{lemma: upper bound of LUn - LP X for all}
    Suppose that \ConvFilterLipschitz, \SigmaLipschitz, \dPBoundedBelow, and \GraphonLipschitz\ hold. Let $u_j$, $j\in[n]$ be independent random variables following a distribution $P$. Let $\VarGraphonNDE$ be the solution of Graphon-NDE \eqref{Graphon-NDE}. If \eq{eq: sup dP dUn} holds and $n$ satisfies \eqref{eq: n large enough for integral operator LP - LUn}, then with probability at least $1-\gamma_2$, for all $\ell\in\mathbb{Z}_{L}$, $g\in[F]$, $k\in[K]$ and $s\in\mathbb{Z}_k$, 
    \begin{equation}\label{eq: upper bound of LUn - LP Xf}
    \begin{aligned}
    &\supT\normb{\parens{\opLUn-\mathcal{L}_{P}}\parens{\opLP^{k-1-s}\VarGraphonNDE\mylt_g}} \\
    &\lesssim\parens{\frac{\Csys[1]\sqrt{\log(4n_I/\gamma_1)} + \Csys[2]\sqrt{\log(4n_ILFK/\gamma_2)}}{\sqrt{n}}}\parens{\frac{c_{\max}}{c_{\min}}}^{k-1-s}\max\curlbraket{\CXmax,\CXlip}
    \end{aligned}
    \end{equation} 
    where $\Csys[1]$ and $\Csys[2]$ are defined in \eqref{def: C1W} and \eqref{def: C2W}, respectively; and \begin{equation}\label{def: C X max and lip}
\CXmax:=\max\left\{\CXmax\myl:\ell\in\mathbb{Z}_{L+1}\right\},\quad \CXlip:=\max\left\{\CXlip\myl:\ell\in\mathbb{Z}_{L+1}\right\}.
\end{equation}
    As a result, for all $\ell\in\mathbb{Z}_{L}$, 
\begin{equation}\label{eq: upper bound of LUn - LP X for all}
\begin{aligned}
    &\supT\sqrt{ \sum_{g=1}^F\sum_{k=1}^{K}\sum_{s=0}^{k-1}\norm{(\opLUn-\mathcal{L}_{P})\parens{\opLP^{k-1-s}\VarGraphonNDE\mylt_g}}_{\spaceb}^2}\\
    &\lesssim\parens{\frac{\Csys[1]\sqrt{\log(4n_I/\gamma_1)} + \Csys[2]\sqrt{\log(4n_ILFK/\gamma_2)}}{\sqrt{n}}}\Csys[3]\max\curlbraket{\CXmax,\CXlip}
\end{aligned}
\end{equation}
where 
\begin{equation}\label{def: C3W}
\Csys[3]:=\sqrt{F\sum_{k=1}^{K}\sum_{s=0}^{k-1}\parens{\frac{c_{\max}}{c_{\min}}}^{2(k-1-s)}}. 
\end{equation}
\end{lemma}

\begin{proof}
Let $m$ be a nonnegative integer. For each $\ell\in\mathbb{Z}_{L}$ and $g\in[F]$, by \eq{TW operator infinity norm is bounded} and Proposition \ref{proposition: norm of X ell is bounded by norm of X}, we have 
\begin{align*}
\supT\norm{\opLP^m(\VarGraphonNDE\mylt_g)}_{\spaceb}\leq \parens{\frac{c_{\max}}{c_{\min}}}^m\supT\norm{\VarGraphonNDE\mylt_g}_{\spaceb}\leq \parens{\frac{c_{\max}}{c_{\min}}}^m \CXmax,
\end{align*}
and, by \eq{TW operator infinity norm is bounded} and Proposition \ref{proposition: solution Xfell is piecewise Lipschitz continuous about t}, we get 
\begin{align*}
\normb{\opLP^m\parens{\VarGraphonNDE\mylt[\ell][t_1]_g}-\opLP^m\parens{\VarGraphonNDE\mylt[\ell][t_2]_g}}\leq\parens{\frac{c_{\max}}{c_{\min}}}^m \normb{\VarGraphonNDE\mylt[\ell][t_1]_g-\VarGraphonNDE\mylt[\ell][t_2]_g}\leq \parens{\frac{c_{\max}}{c_{\min}}}^m \CXlip|t_1-t_2|.
\end{align*}
Then we take $X$, $C_{\max}$ and $C_{\mathrm{Lip}}$ in Lemma \ref{lemma: sup LUn - LP X leq 1/sqrt(n)} as $\opLP^m(\VarGraphonNDE_g\myl)$, $(c_{\max}/c_{\min})^m \CXmax$ and $(c_{\max}/c_{\min})^m \CXlip$, respectively. It follows that for each $\ell\in\mathbb{Z}_{L}$, $f\in[F]$ and non-negative integer $m$, with probability at least $1-\gamma_2$,
\begin{equation}\label{eq: Graphon-NDE sup LUn-LP Lp Xtu}
    \begin{aligned}
    &\sup_{(u,t)\in I\times[0,T]}\left|\parens{\parens{\opLUn-\mathcal{L}_{P}}\parens{\opLP^m\VarGraphonNDE\mylt_g}}(u)\right|\\
    &\lesssim \parens{\frac{\Csys[1]\sqrt{\log(4n_I/\gamma_1)} + \Csys[2]\sqrt{\log(4n_I/\gamma_2)}}{\sqrt{n}}}\parens{\frac{c_{\max}}{c_{\min}}}^m\max\curlbraket{\CXmax,\CXlip}.
    \end{aligned}
\end{equation}
We notice that the number of functions in the set 
$\left\{\opLP^{k-1-s}\VarGraphonNDE\myl_g:s\in\mathbb{Z}_{k},k\in[K],g\in[F],\ell\in\mathbb{Z}_{L}\right\}$
is equal to $\parens{LFK}$. Then we apply \eqref{eq: Graphon-NDE sup LUn-LP Lp Xtu} with an union bound for all functions in the set, which proves \eq{eq: upper bound of LUn - LP Xf}. And \eq{eq: upper bound of LUn - LP X for all} immediately follows.
\end{proof}

\begin{proposition}\label{proposition: graphon-NN MSE}
    Suppose that \ConvFilterLipschitz, \SigmaLipschitz, \dPBoundedBelow, and \GraphonLipschitz\ hold. Let $\VarGraphonNDE$ be the solution of Graphon-NDE \eqref{Graphon-NDE}. Let $\gamma_1,\gamma_2\in(0,1)$ with $2\gamma_1+\gamma_2<1$. Suppose that \eqs \eqref{upper bound of norm of L Un}, \eqref{eq: Ln - LWn 2norm} and \eqref{eq: upper bound of LUn - LP Xf} hold. Then for all $\ell\in[L]$ and $t\in[0,T]$, it holds that
    \begin{equation}\label{eq: MSE layer L leq MSE layer 0}
    \epsilon\mylt
    \leq\ \parens{{L_\sigma}FK\Ch}^\ell \epsilon^{(0,t)}+ \parens{\sum_{s=0}^{\ell-1}({L_\sigma}FK\Ch)^s}\QX(n,\gamma_1,\gamma_2),
    \end{equation}
    where 
\begin{equation}\label{def QX}
\begin{aligned}
&\QX(n,\gamma_1,\gamma_2)\\
\approx&\frac{{L_\sigma}FK\Ch\parens{\frac{2c_{\max}}{c_{\min}}}^K\parens{\Csys[1]\sqrt{\log(4n_I/\gamma_1)} + \Csys[2]\sqrt{\log(4n_ILFK/\gamma_2)}}\Csys[3]\max\curlbraket{\CXmax,\CXlip}}{\sqrt{n}}\\
& + \frac{{L_\sigma}FK^2\Ch\frac{c_{\max}}{c_{\min}^2}\CXmax}{\sqrt{\alpha_n n}}=\mathcal{O}\parens{\frac{1}{\sqrt{\alpha_n n}}+\frac{\sqrt{\log(n_ILFK/\gamma_2)}+\sqrt{\log(n_I/\gamma_1)}}{\sqrt{n}}}.
\end{aligned}
\end{equation}
Here, \(\approx\) denotes equality up to an absolute multiplicative constant.
\end{proposition}

\begin{proof}
We obtain from Proposition \ref{proposition: epsilon ell+1 leq epsilon ell + Delta} that 
\begin{equation}\label{in the proof recurrsion for layerwise MSE}
\epsilon^{(\ell+1,t)}\leq \parens{{L_\sigma}FK\Ch}\epsilon\mylt + \parens{\Delta_{1}\mylt + \Delta_{2}\mylt},\quad \ell\in\mathbb{Z}_L,\ t\in[0,T],
\end{equation} 
with $\Delta_1\mylt$ and $\Delta_2\mylt$ defined by \eqs\eqref{def: Delta1 ell t} and \eqref{def: Delta2 ell t}, respectively. It follows from Proposition \ref{proposition: norm of X ell is bounded by norm of X}, \eq{eq: Ln - LWn 2norm} and definition \eqref{def: Delta1 ell t} of $\Delta_{1}\mylt$ that 
\begin{equation*}%
\Delta_{1}\mylt\lesssim {L_\sigma}FK^2\Ch\frac{c_{\max}}{c_{\min}^2}\CXmax\frac{1}{\sqrt{\alpha_n n}},\quad  \ell\in\mathbb{Z}_L,\ t\in[0,T].
\end{equation*}
Moreover, assumption of \eq{eq: upper bound of LUn - LP Xf} implies \eq{eq: upper bound of LUn - LP X for all} due to Lemma \ref{lemma: upper bound of LUn - LP X for all}. Then definition \eqref{def: Delta2 ell t} of $\Delta_{2}\mylt$ with estimate \eqref{eq: upper bound of LUn - LP X for all} and assumption \eqref{upper bound of norm of L Un} yields that for all $\ell\in\mathbb{Z}_L$ and $t\in[0,T]$, 
\begin{align*}
\Delta_{2}\mylt\lesssim {L_\sigma}FK\Ch\parens{\frac{2c_{\max}}{c_{\min}}}^K\parens{\frac{\Csys[1]\sqrt{\log(4n_I/\gamma_1)} + \Csys[2]\sqrt{\log(4n_ILFK/\gamma_2)}}{\sqrt{n}}}\Csys[3]\max\curlbraket{\CXmax,\CXlip}.%
\end{align*}
Combining the above estimates for $\Delta_{1}\mylt$ and $\Delta_{2}\mylt$, with definition \eqref{def QX} of $\mathcal{Q}_X$, we have 
\begin{equation}\label{in the proof: sum of Delta 1 and Delta 2 leq Q tX}
    \Delta_1\mylt+
    \Delta_2\mylt\leq\QX(n,\gamma_1,\gamma_2).
\end{equation}
Then it follows from \eq{in the proof recurrsion for layerwise MSE} that for all $\ell\in\mathbb{Z}_L$ and $t\in[0,T]$, $\epsilon^{(\ell+1,t)}\leq \parens{{L_\sigma}FK\Ch}\epsilon\mylt + \QX(n,\gamma_1,\gamma_2)$. We obtain \eq{eq: MSE layer L leq MSE layer 0} by recursively applying the above relation, which completes the proof. 
\end{proof}

\begin{proposition}\label{prop: Graphon-MSE every layer}
    Suppose that \eq{eq: MSE layer L leq MSE layer 0} holds, and the initial values in \eqref{GNDE} and \eqref{Graphon-NDE} satisfy \eqref{eq: initial values of GNDE and Graphon-NDE equal}. Then
\begin{enumerate}
\item For all $\ell\in\mathbb{Z}_{L+1}$, it holds that
    \begin{equation}\label{eq: sup MSE bounded by sqrt(n)}
\supT\epsilon\mylt\leq \Ce \QX(n,\gamma_1,\gamma_2),
    \end{equation}
    where
    \begin{align}
    \Ce&:=\max\left\{\Ce^{(0)},\Ce^{(1)}\right\}\label{def: HT},\\
\Ce^{(0)}&:=\parens{\mathrm{exp}\parens{\parens{{L_\sigma}FK\Ch}^L T}-1}\parens{\sum_{\ell=0}^{L-1}({L_\sigma}FK\Ch)^{\ell}}/\parens{{L_\sigma}FK\Ch}^L,\label{def: HT (0)}\\
\Ce^{(1)}&:=\max\left\{\parens{{L_\sigma}FK\Ch}^\ell \Ce^{(0)} + \sum_{s=0}^{\ell-1}({L_\sigma}FK\Ch)^{s}:\ell\in[L]\right\}\label{def: HT (1)}.
    \end{align}

    \item For all $\ell\in[L]$, it holds that 
    \begin{align}
\supT\widetilde{\epsilon}\mylt&\leq {\tildeCe} \QX(n,\gamma_1,\gamma_2),\label{eq: sup tilde epsilon st leq QX}
    \end{align}
where $\tildeCe:=\parens{{L_\sigma}FK\Ch} \Ce + 1$.
\end{enumerate}
\end{proposition}
\begin{proof}
To prove Item 1, we first show that 
\begin{equation}\label{in the proof: estimate epsilon 0t}
\supT\epsilon^{(0,t)}\leq \Ce^{(0)} \QX(n,\gamma_1,\gamma_2).
    \end{equation}
Let
    \begin{equation}\label{in the proof def epsilon(t)}
        \Delta(t):=(\epsilon^{(0,t)})^2=(\MSE(\VarGraphonNDE(\cdot,t),\VarGNDE(t)))^2=\frac{1}{n}\norm{\SamplingOperator(\VarGraphonNDE(\cdot,t))- \VarGNDE(t)}_{\mathrm{F}}^2.
    \end{equation}
    Then 
    \begin{equation*}
        \dt\Delta(t)=\frac{2}{n}\iprod{\SamplingOperator(\VarGraphonNDE(\cdot,t))-\VarGNDE(t),\dt(\SamplingOperator(\VarGraphonNDE(\cdot,t)))-\dt\VarGNDE(t)}_{\mathrm{F}},
    \end{equation*}
    where $\langle\cdot,\cdot\rangle_{\mathrm{F}}$ is Frobenius inner product. Since $\VarGNDE$ and $\VarGraphonNDE$ are solutions of GNDE \eqref{GNDE} and Graphon-NDE \eqref{Graphon-NDE}, respectively, we obtain that
    \begin{align*}
    \dt\Delta(t)=\frac{2}{n}\iprod{\SamplingOperator(\VarGraphonNDE(\cdot,t))-\VarGNDE(t),\SamplingOperator\parens{\GraphonNN(\VarGraphonNDE(\cdot,t);\opLP,\tH(t))}-\GNN(\VarGNDE(t);\matL,\tH(t))}_{\mathrm{F}}.
    \end{align*} 
    By Cauchy-Schwartz inequality, it follows that 
    \begin{align*}
        \dt\Delta(t)&\leq \frac{2}{n}\norm{\SamplingOperator(\VarGraphonNDE(\cdot,t))-\VarGNDE(t)}_{\mathrm{F}}\normF{\SamplingOperator\parens{\GraphonNN(\VarGraphonNDE(\cdot,t);\opLP,\tH(t))}-\GNN(\VarGNDE(t);\matL,\tH(t))}\\
        &=2\sqrt{\Delta(t)}\cdot \epsilon^{(L,t)}.
    \end{align*}
    Let $a:=\parens{{L_\sigma}FK\Ch}^L$ and $b:=\sum_{\ell=0}^{L-1}({L_\sigma}FK\Ch)^{\ell}$ for short. Then we apply \eq{eq: MSE layer L leq MSE layer 0} with $\ell=L$ and obtain that 
    \begin{align*}
        \dt\Delta(t)&\leq 2\sqrt{\Delta(t)}\parens{a\epsilon^{(0,t)}+ b\QX(n,\gamma_1,\gamma_2)}=2a\Delta(t)+ 2b\QX(n,\gamma_1,\gamma_2)\sqrt{\Delta(t)}.%
    \end{align*}
    Therefore, $\Delta(t)\leq \Delta(0)+\int_0^t \parens{2a\Delta(x)+ 2b\QX(n,\gamma_1,\gamma_2)\sqrt{\Delta(x)}} dx$, $t\in[0,T]$. Note that assumption \eqref{eq: initial values of GNDE and Graphon-NDE equal} implies $\Delta(0)=0$. Applying Gr\"onwall's inequality (Lemma \ref{lemma: generalized Gronwall's inequality}), we get
\begin{align*}
    \Delta(t)\leq \parens{\frac{\mathrm{exp}(at)-1}{a}b\QX(n,\gamma_1,\gamma_2)}^2\leq\parens{\Ce^{(0)} \QX(n,\gamma_1,\gamma_2)}^2,\quad t\in[0,T].
\end{align*}
    By definition \eqref{in the proof def epsilon(t)} of $\Delta(t)$, the above inequality further implies \eqref{in the proof: estimate epsilon 0t}. It remains to show that for all $\ell\in[L]$, 
    \begin{equation*}
\supT\epsilon\mylt\leq \Ce^{(1)} \QX(n,\gamma_1,\gamma_2).
    \end{equation*}
    Note that the above inequality can be immediately obtained by substituting the estimate \eqref{in the proof: estimate epsilon 0t} into \eq{eq: MSE layer L leq MSE layer 0}. We finish the proof of \eq{eq: sup MSE bounded by sqrt(n)}.

    We next prove Item 2.  It follows that for $\ell\in\mathbb{Z}_L$, 
\begin{align*}
    \widetilde{\epsilon}^{(\ell+1,t)}&\leq \parens{{L_\sigma}FK\Ch} \epsilon\mylt + \Delta_1\mylt + \Delta_2\mylt\stepjust{Proposition \ref{proposition: epsilon ell+1 leq epsilon ell + Delta}}\\
    &\leq \parens{{L_\sigma}FK\Ch} \Ce \QX(n,\gamma_1,\gamma_2) + \QX(n,\gamma_1,\gamma_2)\stepjust{\eqs\eqref{eq: sup MSE bounded by sqrt(n)} and \eqref{in the proof: sum of Delta 1 and Delta 2 leq Q tX}}\\
    &= \tildeCe \QX(n,\gamma_1,\gamma_2)
\end{align*}
which proves \eq{eq: sup tilde epsilon st leq QX}. 
\end{proof}

\begin{proof}[\textbf{Proof of Theorem \ref{theorem: GNDE -> Graphon-NDE}}]
By Lemma \ref{lemma: LUn-LP leq sqrt(n)}, with probability at least $1-\gamma_1$, \eq{eq: sup dP dUn} holds. We condition on the event that \eq{eq: sup dP dUn} holds, referring to it as the first event.

According to Lemma \ref{lemma: Ln - LWn 2norm}, with probability at least $1-\gamma_1$, \eq{eq: Ln - LWn 2norm} holds. We condition on the event that \eq{eq: Ln - LWn 2norm} holds, referring to it as the second event.

Moreover, due to Lemma \ref{lemma: upper bound of LUn - LP X for all}, with probability at least $1-\gamma_2$, for all $\ell\in[L]$ and $t\in[0,T]$, \eq{eq: upper bound of LUn - LP Xf} holds. We condition on the event that \eq{eq: upper bound of LUn - LP Xf} holds, referring to it as the third event. 

Therefore, with probability at least $1-2\gamma_1-\gamma_2$, all three events occur, that is, all assumptions in Proposition \ref{proposition: graphon-NN MSE} are satisfied. This implies \eq{eq: MSE layer L leq MSE layer 0}, which, combining with the initial value condition \eqref{eq: initial values of GNDE and Graphon-NDE equal}, verifies that all assumptions in Proposition \ref{prop: Graphon-MSE every layer} hold. Now, by \eq{eq: sup MSE bounded by sqrt(n)} with $\ell=0$ in Proposition \ref{prop: Graphon-MSE every layer}, with probability of at least $1-2\gamma_1-\gamma_2$, it holds that 
    \begin{align*}
         &\supT\MSE(\VarGraphonNDE(\cdot,t),\VarGNDE(t))\leq \Ce^{(0)} \QX(n,\gamma_1,\gamma_2)\leq \Ce \QX(n,\gamma_1,\gamma_2)       
    \end{align*}
    where $\QX$ and $\Ce$ are defined in \eq{def QX} and \eq{def: HT}, respectively. This completes the proof.
\end{proof}

\section{Infinite-node Convergence of Discretized GNDEs}\label{Appendix: Infinite-node Convergence of Discretized GNDEs}
We remark that if the parameters $\hltfgk$, $\ell\in[L]$, $f,g\in[F]$, $k\in\mathbb{Z}_K$ satisfy \ConvFilterDiff, then 
\begin{equation}\label{eq: upper bound for filter derivative}
    \supT\max_{f,g\in[F], \ell\in[L], k\in\mathbb{Z}_K}\left|\dt\hltfgk \right|\leq \Liph.
\end{equation}
Similarly if $\sigma$ satisfies \SigmaDiff, then for any $\mZ\in\spaceRnF$, $Z\in\spaceb$ and $\gZ\in\spaceB$, $\normMax{\sigma'(\mZ)}\leq {L_\sigma}$, $\normb{\sigma'(Z)}\leq {L_\sigma}$, $\normB{\sigma'(\gZ)}\leq {L_\sigma}\sqrt{F}$.
In particular, by definitions \eqref{def: widehat mX ell} and \eqref{def: widehat tX ell}, for all $\ell\in[L]$, $t\in[0,T]$, $f\in[F]$, we have
\begin{equation}\label{eq: hat X leq 1}
\begin{aligned}
\normMax{\hatVarGNDE\mylt}&=\normMax{\sigma'(\tildeVarGNDE\mylt)} \le {L_\sigma},\\
\normb{\hatVarGraphonNDE\mylt_f}&=\normb{\sigma'(\tildeVarGraphonNDE\mylt_f)} \leq {L_\sigma},\quad 
\normB{\hatVarGraphonNDE\mylt} \leq {L_\sigma}\sqrt{F}.
\end{aligned}
\end{equation}

\begin{proposition}\label{proposition: Xns Xn0'' bounded above by sqrt n}
Suppose that \eq{eq: MSE layer L leq MSE layer 0} and initial condition \eqref{eq: initial values of GNDE and Graphon-NDE equal} hold. If \ConvFilterDiff\ and \SigmaDiff\ are satisfied, and $n$ is large enough such that  \begin{equation}\label{eq: n is big enough 2}
\Ce \QX(n,\gamma_1,\gamma_2) \leq \CXmax, 
\end{equation}
then the following statements hold. 
\begin{enumerate}
    \item For all $\ell\in\mathbb{Z}_{L+1}$, there holds 
\begin{align}
\supT\normF{\VarGNDE\mylt}&\leq\CmXGNDE\sqrt{n}\label{eq: estimate of mXnlt F norm}\\
\supT\normF{\dt\VarGNDE\mylt}&\leq \CmXDeriGNDE\sqrt{n}\label{eq: estimate of deri mXnlt F norm}  
\end{align}
where $\CmXGNDE:=2\CXmax$ and
$$
\CmXDeriGNDE:=\CmXGNDE\times\max\curlbraket{{L_\sigma}FK\Liph\sum_{s=0}^{\ell-1}\parens{{L_\sigma}FK\Ch}^{s} +  \parens{{L_\sigma}FK\Ch}^\ell:\ell\in\mathbb{Z}_{L+1}}^{\footnotemark}.
$$
\footnotetext{The sum inside is regarded as $0$ when $\ell=0$.}
As a result, 
\begin{align}\label{eq: dXnLt/dt F norm bound}
    \supT\normF{\frac{d^2}{dt^2}\VarGNDE\mylt[0]}\leq \CmXDeriGNDE \sqrt{n}.
    \end{align}

    \item For all $\ell\in[L]$, there holds 
\begin{align}
\supT\normF{\dt\tildeVarGNDE\mylt}&\leq \CmXTildeDeriGNDE\sqrt{n},\label{eq: estimate of deri mXnlt tilde F norm}
\end{align}
where $\CmXTildeDeriGNDE:=\CmXGNDE FK\Liph+FK\Ch\CmXDeriGNDE.$ Moreover, if \SigmaTwiceDiff holds, then for all $\ell\in[L]$, 
    \begin{equation}\label{eq: estimate of deri mXnlt hat F norm}
\supT\normF{\dt\hatVarGNDE\mylt}\leq \CmXHatDeriGNDE\sqrt{n}.
    \end{equation}
    where $ \CmXHatDeriGNDE:=L_{\sigma'}\CmXTildeDeriGNDE$.

\end{enumerate}

\end{proposition}

\begin{proof}
Recall that \eq{eq: MSE layer L leq MSE layer 0} with initial condition \eqref{eq: initial values of GNDE and Graphon-NDE equal} implies \eq{eq: sup MSE bounded by sqrt(n)}, according to Proposition \ref{prop: Graphon-MSE every layer}. 

We first prove Item 1. It follows from \eq{eq: sup MSE bounded by sqrt(n)} that for all $\ell\in\mathbb{Z}_{L+1}$, $\frac{1}{\sqrt{n}}\normF{\SamplingOperator(\VarGraphonNDE\mylt)-\VarGNDE\mylt}= \epsilon\mylt\leq \Ce \QX(n,\gamma_1,\gamma_2)$, 
which implies 
\begin{align*}
\normF{\VarGNDE\mylt}&\leq \sqrt{n}\Ce \QX(n,\gamma_1,\gamma_2) + \normF{\SamplingOperator(\VarGraphonNDE\mylt)}\\
&\leq \sqrt{n}\parens{\Ce \QX(n,\gamma_1,\gamma_2) + \CXmax}\stepjust{\eq{norm of operator delta Un} and Proposition \ref{proposition: norm of X ell is bounded by norm of X}}\\
&\leq 2\CXmax\sqrt{n}=\CmXGNDE\sqrt{n}.\stepjust{\eq{eq: n is big enough 2}}
\end{align*}
This proves \eq{eq: estimate of mXnlt F norm}. By the updating formula \eqref{def: updating formula gnn} of GNNs and Chain rule, for $\ell\in[L]$, we have $\dt\VarGNDE\mylt=\sigma'(\tildeVarGNDE\mylt)\odot \dt\tildeVarGNDE\mylt$. Then due to \eqs \eqref{eq: upper bound of matrix Z odot V} and \eqref{eq: hat X leq 1}, it follows that 
\begin{align}\label{in the proof: dXnlt/dt leq dtildeXnlt/dt}
    \normF{\dt\VarGNDE\mylt}\leq \normMax{\sigma'(\tildeVarGNDE\mylt)} \normF{\dt\tildeVarGNDE\mylt}\leq {L_\sigma}\normF{\dt\tildeVarGNDE\mylt}.
\end{align}
By definition of $\tildeVarGNDE\mylt$, we have 
$\tildeVarGNDE\mylt=\phi_{\vh\mylt}(\VarGNDE\mylt[\ell-1])$. It follows from Lemma \ref{lemma: d/dt phihlt(Zt) estimate} and \eq{eq: estimate of mXnlt F norm} that 
\begin{align}
\normF{\dt\tildeVarGNDE\mylt}
&\leq FK\Liph\normF{\VarGNDE\mylt[\ell-1]}+FK\Ch\normF{\dt\VarGNDE\mylt[\ell-1]}\nonumber\\
&\leq \sqrt{n}\underbrace{FK\Liph\CmXGNDE}_{\text{denoted by }a}+\underbrace{FK\Ch}_{\text{denoted by }b}\normF{\dt\VarGNDE\mylt[\ell-1]}\label{in the proof: ddtmXnlt tilde leq a sqrt n + b ddt Xnl-1t}
\end{align}
which together with \eq{in the proof: dXnlt/dt leq dtildeXnlt/dt} implies $
\normF{\dt\VarGNDE\mylt}\leq {L_\sigma}(a\sqrt{n} + b\normF{\dt\VarGNDE\mylt[\ell-1]})$, $\ell\in[L]$.
A recursion gives 
\begin{equation}\label{in the proof: dmXnlt F leq a sum b sqrt n + b ell}
    \normF{\dt\VarGNDE\mylt}\leq {L_\sigma}a\parens{\sum_{s=0}^{\ell-1}({L_\sigma}b)^{s}}\sqrt{n} + ({L_\sigma}b)^\ell\normF{\dt\VarGNDE\mylt[0]},\quad \ell\in[L].
\end{equation}
Note that $\VarGNDE\mylt[0]$ is the solution of GNDE \eqref{GNDE}, that is 
\begin{equation}\label{in the proof: dt mXnlt = mXnlt L}
\dt\VarGNDE\mylt[0]=\VarGNDE\mylt[L].    
\end{equation}
Therefore, \eq{in the proof: dmXnlt F leq a sum b sqrt n + b ell}, combining with \eq{in the proof: dt mXnlt = mXnlt L} and estimate \eqref{eq: estimate of mXnlt F norm} in the case of $\ell=L$, implies $\normF{\dt\VarGNDE\mylt}\leq \parens{{L_\sigma}a\sum_{s=0}^{\ell-1}({L_\sigma}b)^{s} + ({L_\sigma}b)^\ell\CmXGNDE }\sqrt{n}\leq \CmXDeriGNDE\sqrt{n}$, $\ell\in[L]$, where the last inequality follows from definition of constant $\CmXDeriGNDE$. This proves \eqref{eq: estimate of deri mXnlt F norm} for $\ell\in[L]$. Note that \eqref{eq: estimate of deri mXnlt F norm} is also true when $\ell=0$ due to \eqref{in the proof: dt mXnlt = mXnlt L} and \eqref{eq: estimate of mXnlt F norm}. Therefore, \eqref{eq: estimate of deri mXnlt F norm} is true for all $\ell\in\mathbb{Z}_{L+1}$. 

As a result, we obtain \eqref{eq: dXnLt/dt F norm bound} from taking derivatives about $t$ for both sides of \eq{in the proof: dt mXnlt = mXnlt L} and using estimate \eqref{eq: estimate of deri mXnlt F norm} with $\ell=L$.

We proceed to show Item 2. We immediately obtain \eq{eq: estimate of deri mXnlt tilde F norm} by substituting estimate \eqref{eq: estimate of deri mXnlt F norm} into \eq{in the proof: ddtmXnlt tilde leq a sqrt n + b ddt Xnl-1t}. Moreover, we get \eq{eq: estimate of deri mXnlt hat F norm} by noting that
    \begin{align*}
\normF{\dt\hatVarGNDE\mylt}&=\normF{\sigma''(\tildeVarGNDE\mylt)\odot \dt\tildeVarGNDE\mylt}\leq \normMax{\sigma''(\tildeVarGNDE\mylt)}\normF{\dt\tildeVarGNDE\mylt}\stepjust{Chain rule and \eq{eq: upper bound of matrix Z odot V}}\\
&\leq L_{\sigma'}\CmXTildeDeriGNDE\sqrt{n}=\CmXHatDeriGNDE\sqrt{n}.\stepjust{relation \eqref{eq: sigma twice prime leq} and \eq{eq: estimate of deri mXnlt tilde F norm}}
    \end{align*}
The proof is complete. 
\end{proof}

\begin{proposition}\label{proposition: discretized GNDE error leq 1/M}
Under assumptions of Proposition \ref{proposition: Xns Xn0'' bounded above by sqrt n}, let $\left\{\VarGNDE^{[m]} : m \in [M]\right\}$ be the sequence of approximations to the solution $\VarGNDE(t)$ of GNDE~\eqref{GNDE} computed via Euler's method \eqref{eq: Euler's method for Xn} with step size $\kappa := T/M$, $t_0=0$ and initial value satisfies \eqref{eq: initial value condition on GNDE and discretized GNDE}. Then for all $m\in[M]$, the following statements hold. 
\begin{enumerate}
    \item For all $\ell\in\mathbb{Z}_{L+1}$, \begin{equation}\label{eq: mXlnm leq 1/M}
    \frac{1}{\sqrt{n}}\normF{\VarGNDE\mylt[\ell][t_m]- \VarGNDE^{[\ell,m]}}\leq  \frac{\CXdiff}{M}
    \end{equation}
    where $\CXdiff:=\max\curlbraket{\CXdiff^{(0)},\CXdiff^{(1)}}$ and 
    $\CXdiff^{(0)}:=\frac{T \CmXDeriGNDE}{2({L_\sigma}FK\Ch)^L} \parens{e^{({L_\sigma}FK\Ch)^L T} - 1 }$, $\CXdiff^{(1)}:=\CXdiff^{(0)}\max\curlbraket{\parens{{L_\sigma}FK\Ch}^\ell:\ell\in[L]}$.
    \item If \SigmaDerivativeLipschitz\ holds, then 
    \begin{enumerate}
        \item for all $\ell\in\mathbb{Z}_{L+1}$, 
        \begin{align}
        \frac{1}{\sqrt{n}}\normF{\VarGNDE^{[\ell,m]}}&\leq \CmXDistGNDE,\label{eq: mXln[m] leq 1/M}\\
        \frac{1}{\sqrt{n}}\normF{\VarGNDE^{[\ell,m]}-\VarGNDE^{[\ell,m-1]}} &\leq \frac{\CXdiffDist}{M},\label{eq: mXln diff [m] leq 1/M every layer}
    \end{align}
    where $\CmXDistGNDE:=\parens{\CXdiff + \CmXGNDE}\max\curlbraket{\parens{{L_\sigma}FK\Ch}^\ell:\ell\in\mathbb{Z}_{L+1}}$, $\CXdiffDist:=\max\curlbraket{\CXdiffDist^{(0)},\CXdiffDist^{(1)}}$, $\CXdiffDist^{(0)}:=T\parens{{L_\sigma}FK\Ch}^L\CmXDistGNDE$, and $$\CXdiffDist^{(1)}:=\max\curlbraket{\parens{{L_\sigma}FK\Ch}^\ell \CXdiffDist^{(0)}+ {L_\sigma}FK\Liph \parens{\ell\parens{{L_\sigma}FK\Ch}^{\ell-1}}\CmXDistGNDE:\ell\in[L]}.$$ 
        \item for all $\ell\in[L]$,
        \begin{align}
            \frac{1}{\sqrt{n}}\normF{\hatVarGNDE\mylt[\ell][t_m]- \hatVarGNDE^{[\ell,m]}}&\leq  \frac{\CXdiffHat}{M},\label{eq: hat mXlnm leq 1/M}\\
            \frac{1}{\sqrt{n}}\normF{\hatVarGNDE^{[\ell,m]}-\hatVarGNDE^{[\ell,m-1]}} &\leq \frac{\CXdiffDistHat}{M},\label{eq: hat mXln diff [m] leq 1/M every layer}
        \end{align}
        where $\CXdiffHat:=L_{\sigma'}\CXdiff$ and $\CXdiffDistHat:=L_{\sigma'}\CXdiffDist$. 
    \end{enumerate}
    
\end{enumerate}

\end{proposition}

\begin{proof}
According to Lemma \ref{lemma: Euler's method}, with the bound of the second derivative of $\VarGNDE(t)$ derived in Proposition \ref{proposition: Xns Xn0'' bounded above by sqrt n} (i.e., \eq{eq: dXnLt/dt F norm bound}) and Lipschitz constant of the $\GNN$ function computed in Proposition \ref{proposition: Nn is Lipschitz continuous}, we obtain by $\kappa=T/M$ that 
    \begin{align}
        \normF{\VarGNDE(t_m) - \VarGNDE^{[m]}}\leq \frac{\sqrt{n}\CmXDeriGNDE}{2({L_\sigma}FK\Ch)^L} \parens{e^{({L_\sigma}FK\Ch)^L t_m} - 1 }\leq \CXdiff^{(0)}\frac{\sqrt{n}}{M}. \label{in the proof: Xn(tm) - Xn[m]}
    \end{align}
    Notice that for all $\ell\in[L]$, 
    \begin{align*}
&\normF{\VarGNDE^{[\ell,m]}-\VarGNDE\mylt[\ell][t_m]}=\normF{\sigma(\phi_{\vh\mylt[\ell][t_m]}(\VarGNDE^{[\ell-1,m]}))-\sigma(\phi_{\vh\mylt[\ell][t_m]}(\VarGNDE\mylt[\ell-1][t_m]))}\\
&\leq {L_\sigma}FK\Ch\normF{\VarGNDE^{[\ell-1,m]}-\VarGNDE\mylt[\ell-1][t_m]}\stepjust{\SigmaLipschitz\ and \eq{eq: phi t1 - t2 X - Z} in Lemma \ref{lemma: phi phi* t1 - t2 X - Z PQ}}\\
&\leq ({L_\sigma}FK\Ch)^\ell\normF{\VarGNDE^{[0,m]}-\VarGNDE\mylt[0][t_m]}\leq ({L_\sigma}FK\Ch)^\ell\CXdiff^{(0)}\frac{\sqrt{n}}{M}\leq \CXdiff^{(1)}\frac{\sqrt{n}}{M} \stepjust{recurrsion and \eq{in the proof: Xn(tm) - Xn[m]}}
    \end{align*}
Therefore, \eq{eq: mXlnm leq 1/M} holds for all $\ell\in\mathbb{Z}_{L+1}$. We proceed to prove \eq{eq: mXln[m] leq 1/M}. Note that 
\begin{align*}
\normF{\VarGNDE^{[m]}}&\leq \sqrt{n}\frac{\CXdiff}{M} + \normF{\VarGNDE(t_m) }\stepjust{\eq{eq: mXlnm leq 1/M} with $\ell=0$ and triangle equality}\\
&\leq \sqrt{n}\frac{\CXdiff}{M} + \CmXGNDE\sqrt{n}\leq \parens{\CXdiff + \CmXGNDE}\sqrt{n}\stepjust{\eq{eq: estimate of mXnlt F norm} in Proposition \ref{proposition: Xns Xn0'' bounded above by sqrt n}}
\end{align*}
By \SigmaLipschitz\ and Lemma \ref{lemma: estimate of phi(matrix)}, we have, for all $\ell\in\mathbb{Z}_{L+1}$ and $m\in[M]$, 
\begin{equation}\label{in the proof: Xnlm F leq FKhT^l Xn0m}
    \normF{\VarGNDE^{[\ell,m]}}\leq \parens{{L_\sigma}FK\Ch}^\ell \normF{\VarGNDE^{[0,m]}}=\parens{{L_\sigma}FK\Ch}^\ell \normF{\VarGNDE^{[m]}},
\end{equation} 
which combining with the previous inequality yields \eq{eq: mXln[m] leq 1/M}. We next show \eq{eq: mXln diff [m] leq 1/M every layer}. Recall that we adopt the notations of $\VarGNDE^{[0,m]}\equiv\VarGNDE^{[m]}$ and $\VarGNDE^{[L,m]}\equiv\GNN(\VarGNDE^{[m]};\matL,\tH(t_{m}))$. Hence \eq{eq: Euler's method for Xn} can be rephrased as $
\VarGNDE^{[0,m]}-\VarGNDE^{[0,m-1]}=\kappa\VarGNDE^{[L,m-1]}$, which implies 
\begin{align*}
    \normF{\VarGNDE^{[0,m]}-\VarGNDE^{[0,m-1]}}&=\frac{T}{M}\normF{\VarGNDE^{[L,m-1]}}\leq \frac{T}{M}\parens{{L_\sigma}FK\Ch}^L\normF{\VarGNDE^{[0,m-1]}}\stepjust{\eq{in the proof: Xnlm F leq FKhT^l Xn0m}}\\
    &\leq \frac{T}{M}\parens{{L_\sigma}FK\Ch}^L\sqrt{n}\CmXDistGNDE=\frac{\sqrt{n}}{M}\CXdiffDist^{(0)}.\stepjust{\eq{eq: mXln[m] leq 1/M} with $\ell=0$}
\end{align*}
In addition, for $\ell\in[L]$, we have 
\begin{align*}
    &\normF{\VarGNDE^{[\ell,m]}-\VarGNDE^{[\ell,m-1]}}=\normF{\sigma(\phi_{\vh\mylt[\ell][t_m]}(\VarGNDE^{[\ell-1,m]})) - \sigma(\phi_{\vh\mylt[\ell][t_{m-1}]}(\VarGNDE^{[\ell-1,m-1]}))}\\
    \leq\ & {L_\sigma}FK\Ch\normF{\VarGNDE^{[\ell-1,m]} - \VarGNDE^{[\ell-1,m-1]}} + {L_\sigma}FK\frac{\Liph}{M}\normF{\VarGNDE^{[\ell-1,m-1]}}\stepjust{\SigmaLipschitz\ and \eq{eq: phi t1 - t2 X - Z} in Lemma \ref{lemma: phi phi* t1 - t2 X - Z PQ}}    \\
    \leq\ &\parens{{L_\sigma}FK\Ch}^\ell \normF{\VarGNDE^{[0,m]}-\VarGNDE^{[0,m-1]}} + {L_\sigma}FK\frac{\Liph}{M} \parens{\sum_{s=0}^{\ell-1}\parens{{L_\sigma}FK\Ch}^{\ell-1-s}\normF{\VarGNDE^{[s,m-1]}}}\stepjust{recursion}\\
    \leq\ &\parens{{L_\sigma}FK\Ch}^\ell \frac{\sqrt{n}}{M}\CXdiffDist^{(0)}+ {L_\sigma}FK\frac{\Liph}{M} \parens{\ell\parens{{L_\sigma}FK\Ch}^{\ell-1}}\sqrt{n}\CmXDistGNDE\leq \frac{\sqrt{n}}{M}\CXdiffDist^{(1)}\stepjust{\eqs \eqref{eq: mXln diff [m] leq 1/M every layer}, \eqref{in the proof: Xnlm F leq FKhT^l Xn0m} and definition of $\CXdiffDist^{(1)}$}
\end{align*}
Therefore, \eq{eq: mXln diff [m] leq 1/M every layer} holds for all $\ell\in\mathbb{Z}_{L+1}$. We next prove \eq{eq: hat mXlnm leq 1/M}. By \SigmaDerivativeLipschitz\ and (the proof of) \eq{eq: mXlnm leq 1/M}, for all $\ell\in[L]$, we have $$\normF{\hatVarGNDE^{[\ell,m]}-\hatVarGNDE\mylt[\ell][t_m]}=\normF{\sigma'(\phi_{\vh\mylt[\ell][t_m]}(\VarGNDE^{[\ell-1,m]}))-\sigma'(\phi_{\vh\mylt[\ell][t_m]}(\VarGNDE\mylt[\ell-1][t_m]))}\leq L_{\sigma'}\frac{\CXdiff}{M}\sqrt{n},$$
which proves \eq{eq: hat mXlnm leq 1/M}. Finally, by \SigmaDerivativeLipschitz\ and (the proof of) \eq{eq: mXln diff [m] leq 1/M every layer}, for all $\ell\in[L]$, we have
\begin{align*}
    &\normF{\hatVarGNDE^{[\ell,m]}-\hatVarGNDE^{[\ell,m-1]}}=\normF{\sigma'(\phi_{\vh\mylt[\ell][t_m]}(\VarGNDE^{[\ell-1,m]})) - \sigma'(\phi_{\vh\mylt[\ell][t_{m-1}]}(\VarGNDE^{[\ell-1,m-1]}))}\leq L_{\sigma'}\frac{\CXdiffDist}{M}\sqrt{n},
\end{align*}
which proves \eq{eq: hat mXln diff [m] leq 1/M every layer}. 
\end{proof}

\begin{proof}[\textbf{Proof of Theorem \ref{theorem: discretized gnde to gnde}}]
    By Lemmas \ref{lemma: LUn-LP leq sqrt(n)}, \ref{lemma: Ln - LWn 2norm} and \ref{lemma: upper bound of LUn - LP X for all} with conditions \eqref{eq: n large enough for matrix L - LUn} and \eqref{assumption: sparsity alpha n large enough} on $n$ and $\alpha_n$, with probability at least $1-2\gamma_1-\gamma_2$, \eqs \eqref{eq: Ln - LWn 2norm}, \eqref{upper bound of norm of L Un}, \eqref{eq: upper bound of LUn - LP Xf} hold. Hence, we have \eq{eq: MSE layer L leq MSE layer 0} by Proposition \ref{proposition: graphon-NN MSE}. This combining with condition \eqref{eq: n is big enough 2} on $n$ verifies all assumptions in Proposition \ref{proposition: discretized GNDE error leq 1/M}. Then the desired result follows from \eq{eq: mXlnm leq 1/M} with $\ell=0$ in Proposition \ref{proposition: discretized GNDE error leq 1/M}. 
\end{proof}

\section{Property of Adjoint Graphon-NDE solutions}

\subsection{Boundedness and Well-posedness}\label{appendix: Boundedness and Well-posedness}
\begin{proposition}\label{proposition: norm of Y ell is bounded by norm of Y}
Suppose that \ConvFilterLipschitz, \SigmaLipschitz, and \dPBoundedBelow\ hold. Let $\VarGraphonNDE$ be the solution of the Graphon-NDE \eqref{Graphon-NDE}. Let $\VarGraphonNDEHiddenState\in \spaceC$. For each $t\in[0,T]$, let $\VarGraphonNDEHiddenState\mylt:I\to\mathbb{R}^{1\times F}$ be generated by \eq{updating rule for tY ell}. If \SigmaDiff\ holds, then for each $\ell\in\mathbb{Z}_{L+1}$,
    \begin{equation}\label{eq: norm of Y ell is bounded by norm of Y}
\normC{\VarGraphonNDEHiddenState\myl}\leq \CYmax\myl:=\norm{\VarGraphonNDEHiddenState}_{\spaceC}(F^{\frac{3}{2}}{L_\sigma}\Csys[0]\Ch)^{L-\ell}. 
    \end{equation}

\end{proposition}

\begin{proof}
For each $\ell\in\mathbb{Z}_L$, it holds that
\begin{align*}
\norm{\VarGraphonNDEHiddenState\mylt}_{\spaceB}&=\norm{\Phi_{\vh^{(\ell+1,t)}}^*\parens{\hatVarGraphonNDE\mylt[\ell+1]\odot \VarGraphonNDEHiddenState^{(\ell+1,t)}}}_{\spaceB}\stepjust{\eq{updating rule for tY ell}}\\
&\leq \norm{\Phi_{\vh^{(\ell+1,t)}}^*}_{\spaceB\to \spaceB}\norm{\hatVarGraphonNDE\mylt[\ell+1]}_{\spaceB}\norm{\VarGraphonNDEHiddenState^{(\ell+1,t)}}_{\spaceB}\stepjust{\eq{eq: upper bound of function Z odot V}}\\
&\leq (F\Csys[0]\Ch){L_\sigma}\sqrt{F}\norm{\VarGraphonNDEHiddenState^{(\ell+1,t)}}_{\spaceB}\stepjust{Lemma \ref{lemma: norm of operator mathfrak H and adjoint operator widehat mathfrak H} and \eq{eq: hat X leq 1}}
\end{align*}    
Noting that $\VarGraphonNDEHiddenState=\VarGraphonNDEHiddenState^{(L)}$, a recursion indicates \eq{eq: norm of Y ell is bounded by norm of Y} for $\ell\in\mathbb{Z}_L$. Note that \eq{eq: norm of Y ell is bounded by norm of Y} is trivially true when $\ell=L$. The proof is hence complete.
\end{proof}

\begin{proof}[\textbf{Proof of Theorem \ref{theorem: Well-posedness of Adjoint System}}]
The proof follows a standard Banach contraction mapping argument. We first rewrite the initial-value problem \eqref{adjoint Grahpon-NDE: Y} as an integral fixed-point equation. Crucially, by setting $\ell=0$ in \eq{eq: norm of Y ell is bounded by norm of Y} of Proposition \ref{proposition: norm of Y ell is bounded by norm of Y}, we can establish that the operator $(\frac{\delta \GraphonNN}{\delta \VarGraphonNDE^{(0)}})^*$ is uniformly bounded. This boundedness guarantees that, by choosing a sufficiently short backward time interval, the associated fixed-point operator becomes a contraction. Iterating this local argument then yields a unique solution over the entire interval $[0,T]$.

\end{proof}

\begin{proposition}\label{proposition: X odot Y is bounded}
    Suppose that \ConvFilterLipschitz, \SigmaDiff, and \dPBoundedBelow\ hold. Let $\VarGraphonNDE$ and \(\VarGraphonNDEHiddenState\) be the solutions of the Graphon-NDE \eqref{Graphon-NDE} and the adjoint system \eqref{adjoint Grahpon-NDE: Y}, respectively. Let $\hatVarGraphonNDE\mylt$ be defined by \eq{def: widehat tX ell}. Then for each $\ell\in\mathbb{Z}_{L+1}$, $
    \normC{\hatVarGraphonNDE\myl\odot \VarGraphonNDEHiddenState\myl}\leq \Codotmax\myl$, where $\Codotmax\myl:={L_\sigma}\sqrt{F}\CYmax\myl$.
\end{proposition}

\begin{proof}
Note that by \eq{eq: upper bound of function Z odot V}, \eq{eq: hat X leq 1} and Proposition \ref{proposition: norm of Y ell is bounded by norm of Y}, we have
$$\normB{\hatVarGraphonNDE\mylt\odot \VarGraphonNDEHiddenState\mylt}\leq \normB{\hatVarGraphonNDE\mylt}\normB{\VarGraphonNDEHiddenState\mylt}\leq {L_\sigma}\sqrt{F} \CYmax\myl.$$ We obtain the desired result by taking supremum about $t$ over $[0,T]$.
\end{proof}

\subsection{Temporal Lipschitz}

\begin{proposition}\label{proposition: temporal lipschitz for adjoint system}
    Suppose that \ConvFilterLipschitz, \SigmaDerivativeLipschitz, and \dPBoundedBelow\ hold. Let $\VarGraphonNDE$ and \(\VarGraphonNDEHiddenState\) be the solutions of the Graphon-NDE \eqref{Graphon-NDE} and the adjoint system \eqref{adjoint Grahpon-NDE: Y}, respectively. Let $\hatVarGraphonNDE\mylt$ be defined by \eq{def: widehat tX ell}. Then for each $\ell\in\mathbb{Z}_{L+1}$,
    \begin{equation}\label{eq: Y is Lipschitz continuous}
        \norm{\VarGraphonNDEHiddenState\mylt[\ell][t_1]_g-\VarGraphonNDEHiddenState\mylt[\ell][t_2]_g}_{\spaceb}\leq \CYlip\myl|t_1-t_2|,\quad\forall t_1,t_2\in[0,T],
    \end{equation}
    \begin{equation}\label{eq: X odot Y is Lipschitz continuous}
        \normb{\hatVarGraphonNDE\mylt[\ell][t_1]_g\VarGraphonNDEHiddenState\mylt[\ell][t_1]_g-\hatVarGraphonNDE\mylt[\ell][t_2]_g\VarGraphonNDEHiddenState\mylt[\ell][t_2]_g}\leq \Codotlip\myl|t_1-t_2|,\quad\forall t_1,t_2\in[0,T], 
    \end{equation}
    where  
    \begin{equation}\label{def: widetilde C Lip L}
\CYlip^{(L)}:=\CYmax^{(0)},
    \end{equation}
    \begin{equation}\label{def: delta Lip ell}
\Codotlip\myl:={L_\sigma}\CYlip\myl + L_{\sigma'}\CXlip\myl\CYmax\myl,\quad \ell\in\mathbb{Z}_{L+1},
    \end{equation}
    \begin{equation}\label{def: widetilde C Lip ell-1}
\CYlip\myl[\ell-1]:=\Csys[0]\parens{\Liph\sqrt{F} \Codotmax\myl+F\Ch\Codotlip\myl}, \quad \ell\in[L-1],
    \end{equation}
    and these constants are generated in the order of $\CYlip^{(L)}$, $\Codotlip^{(L)}$, $\CYlip^{(L-1)}$, $\Codotlip^{(L-1)}$, $\ldots$, $\CYlip^{(0)}$, $\Codotlip^{(0)}$. 
\end{proposition}

\begin{proof}
We begin with proving \eq{eq: Y is Lipschitz continuous} for the case of $\ell=L$. Recall notations in \eq{adjoint operator: Graphon-NDE to filters} that 
$\VarGraphonNDEHiddenState^{(L)}=\VarGraphonNDEHiddenState$ and $\VarGraphonNDEHiddenState^{(0)}=\parens{\frac{\delta \GraphonNN}{\delta \VarGraphonNDE^{(0)}}}^*(\VarGraphonNDEHiddenState).$ Since $\VarGraphonNDEHiddenState$ is the solution of \eq{adjoint Grahpon-NDE: Y}, we obtain that
\begin{align*}
\supT\left|\frac{\partial }{\partial t}\VarGraphonNDEHiddenState_f(u,t)\right|\leq\normC{-\parens{\frac{\delta \GraphonNN}{\delta \VarGraphonNDE^{(0)}}}^*(\VarGraphonNDEHiddenState)} 
=\normC{-\VarGraphonNDEHiddenState^{(0)}}\leq \CYmax^{(0)}=\CYlip^{(L)}
\end{align*}
where the second inequality follows from \eq{eq: norm of Y ell is bounded by norm of Y} with $\ell=0$, and the last equation is due to definition \eqref{def: widetilde C Lip L}. This combining with the fact that $\VarGraphonNDEHiddenState_f$ is continuously differentiable further implies 
$
\left|\VarGraphonNDEHiddenState_f(u,t_1)-\VarGraphonNDEHiddenState_f(u,t_2)\right|\leq \supT\left|\frac{\partial }{\partial t}\VarGraphonNDEHiddenState_f(u,t)\right||t_1-t_2|\leq \CYlip^{(L)}|t_1-t_2|$, which proves \eq{eq: Y is Lipschitz continuous} when $\ell=L$.

Now we assume that \eq{eq: Y is Lipschitz continuous} holds for some $\ell$ with constant $\CYlip\myl$. We will prove \eq{eq: X odot Y is Lipschitz continuous} with constant $\Codotlip\myl$ defined by \eq{def: delta Lip ell} (relying on $\CYlip\myl$). And then we will show that \eq{eq: Y is Lipschitz continuous} holds for $\ell-1$ with constant $\CYlip\myl[\ell-1]$ in the form of \eq{def: widetilde C Lip ell-1}. According to Proposition \ref{proposition: X odot Y is bounded}, we have 
\begin{equation}\label{in the proof: X odot Y is bounded}
\normC{\hatVarGraphonNDE\myl\odot \VarGraphonNDEHiddenState\myl}\leq \Codotmax\myl. 
\end{equation}
Moreover, 
\begin{align*}
    &\norm{ \hatVarGraphonNDE\mylt[\ell][t_1]_g \VarGraphonNDEHiddenState\mylt[\ell][t_1]_g-\hatVarGraphonNDE\mylt[\ell][t_2]_g \VarGraphonNDEHiddenState\mylt[\ell][t_2]_g}_{\spaceb}\\
    \leq\ & \norm{ \hatVarGraphonNDE\mylt[\ell][t_1]_g}_{\spaceb}\norm{\VarGraphonNDEHiddenState\mylt[\ell][t_1]_g-\VarGraphonNDEHiddenState\mylt[\ell][t_2]_g}_{\spaceb} + \norm{ \hatVarGraphonNDE\mylt[\ell][t_1]_g-\hatVarGraphonNDE\mylt[\ell][t_2]_g}_{\spaceb}\norm{\VarGraphonNDEHiddenState\mylt[\ell][t_2]_g}_{\spaceb}\stepjust{triangle inequality}\\
    \leq\ &\ {L_\sigma}\CYlip\myl|t_1-t_2| + L_{\sigma'}\CXlip\myl|t_1-t_2|\CYmax\myl=\Codotlip\myl|t_1-t_2| \stepjust{\SigmaDerivativeLipschitz, Propositions \ref{proposition: solution Xfell is piecewise Lipschitz continuous about t}, \ref{proposition: norm of Y ell is bounded by norm of Y}, \ref{proposition: temporal lipschitz for adjoint system}, and definition of $\Codotlip\myl$}
\end{align*}
which proves \eq{eq: X odot Y is Lipschitz continuous}. We proceed to prove that \eq{eq: Y is Lipschitz continuous} holds with $\ell$ being replaced by $\ell-1$. It follows from \eq{updating rule for tY ell} that 
$\VarGraphonNDEHiddenState\mylt[\ell-1][t]_g=\left[\Phi_{\vh\mylt[\ell][t]}^*\parens{\hatVarGraphonNDE\mylt[\ell][t]\odot \VarGraphonNDEHiddenState\mylt[\ell][t]}\right]_g$, $t\in\{t_1,t_2\}$, which with Lemma \ref{lemma: operator frakH is lipschitz continuous about t}, \eqs\eqref{eq: X odot Y is Lipschitz continuous}, \eqref{in the proof: X odot Y is bounded} and definition of $\CYlip\myl[\ell-1]$, implies
\begin{equation*}
\norm{\VarGraphonNDEHiddenState\mylt[\ell-1][t_1]_g-\VarGraphonNDEHiddenState\mylt[\ell-1][t_2]_g}_{\spaceb}\leq \Csys[0]\parens{\Liph\sqrt{F} \Codotmax\myl+F\Ch\Codotlip\myl} |t_1-t_2|=\CYlip\myl[\ell-1]|t_1-t_2|,
\end{equation*}
which proves \eq{eq: Y is Lipschitz continuous} for $\ell-1$. This completes the proof. 
\end{proof}

\section{Infinite-node Convergence of Adjoint GNDEs}\label{appendix: Infinite-node Convergence of Adjoint GNDEs}
\begin{proposition}\label{proposition: eta ell+1 leq eta ell + Delta}
Suppose that \ConvFilterLipschitz, \SigmaDerivativeLipschitz, and \dPBoundedBelow\ hold. Let $\VarGNDE$, $\VarGraphonNDE$, $\VarGNDEHiddenState$ and $\VarGraphonNDEHiddenState$ be the solutions of GNDE \eqref{GNDE}, Graphon-NDE \eqref{Graphon-NDE}, adjoint GNDE \eqref{adjoint GNDE: Y}, and adjoint Graphon-NDE \eqref{adjoint Grahpon-NDE: Y}, respectively. Define 
\begin{align}\label{def: eta ell}
\eta\mylt := \MSE(\VarGraphonNDEHiddenState\mylt,\VarGNDEHiddenState\mylt)=\frac{1}{\sqrt{n}}\norm{\SamplingOperator(\VarGraphonNDEHiddenState\mylt)-\VarGNDEHiddenState\mylt}_{\mathrm{F}},\quad \ell\in\mathbb{Z}_{L+1},\ t\in[0,T].
\end{align}
Then for $\ell\in[L]$ and $t\in[0,T]$, 
$$
\eta^{(\ell-1,t)}\leq  \parens{{L_\sigma}FK\Ch}\eta\mylt + \widetilde{\Delta}_{1}\mylt + \widetilde{\Delta}_{2}\mylt + \widetilde{\Delta}_{3}\mylt,
$$
    where 
    \begin{align}
    \widetilde{\Delta}_{1}\mylt&:=FK^2\Ch\norm{\matL-\matLUn}_2\norm{\hatVarGraphonNDE\mylt\odot\VarGraphonNDEHiddenState\mylt}_{\spaceB}\label{def: tilde Delta1 ell t}\\
    \widetilde{\Delta}_{2}\mylt&:=FK\Ch\norm{\opLUn}_{\spaceb\to\spaceb}^K\sqrt{ \sum_{g=1}^F\sum_{k=1}^{K}\sum_{s=0}^{k-1}\norm{(\opLUn-\mathcal{L}_{P})\parens{\opLP^{k-1-s}\parens{\hatVarGraphonNDE\mylt_g\odot\VarGraphonNDEHiddenState\mylt_g}}}_{\spaceb}^2}\label{def: tilde Delta2 ell t}\\
    \widetilde{\Delta}_{3}\mylt&:=FK\Ch L_{\sigma'}\CYmax\myl\widetilde{\epsilon}\mylt\label{def: tilde Delta3 ell t}
    \end{align}

\end{proposition}

\begin{proof}
    Note that
    \begin{align*}
    \eta^{(\ell-1,t)}&=\MSE\parens{\VarGraphonNDEHiddenState^{(\ell-1,t)},\VarGNDEHiddenState\mylt[\ell-1]}=\MSE\parens{\Phi_{\vh\mylt}^*\parens{\hatVarGraphonNDE\mylt\odot \VarGraphonNDEHiddenState\mylt},\phi_{\vh\mylt}^*\parens{\hatVarGNDE\mylt\odot \VarGNDEHiddenState\mylt}}\stepjust{\eqs\eqref{updating rule for tY ell} and \eqref{updating rule for mY ell}}\\
    &\leq \parens{FK\Ch} \MSE\parens{\hatVarGraphonNDE\mylt\odot \VarGraphonNDEHiddenState\mylt,\hatVarGNDE\mylt\odot \VarGNDEHiddenState\mylt} + \widetilde{\Delta}_1\mylt + \widetilde{\Delta}_2\mylt\stepjust{Lemma \ref{lemma: MSE hatfrakH function Z - hatfrakh matrix Z}}
\end{align*} 
in which 
\begin{align*}
    &\MSE\parens{\hatVarGraphonNDE\mylt\odot \VarGraphonNDEHiddenState\mylt,\hatVarGNDE\mylt\odot \VarGNDEHiddenState\mylt}=\frac{1}{\sqrt{n}}\norm{\SamplingOperator(\hatVarGraphonNDE\mylt)\odot\SamplingOperator(\VarGraphonNDEHiddenState\mylt)-\hatVarGNDE\mylt\odot\VarGNDEHiddenState\mylt}_{\mathrm{F}}\\
    \leq\ & \frac{1}{\sqrt{n}}\norm{\SamplingOperator(\hatVarGraphonNDE\mylt)-\hatVarGNDE\mylt}_{\mathrm{F}}\norm{\SamplingOperator(\VarGraphonNDEHiddenState\mylt)}_{\max} + \frac{1}{\sqrt{n}}\norm{\hatVarGNDE\mylt}_{\max}\norm{\SamplingOperator(\VarGraphonNDEHiddenState\mylt)-\VarGNDEHiddenState\mylt}_{\mathrm{F}}\stepjust{triangle inequality and \eq{eq: upper bound of matrix Z odot V}}\\
    \leq\ & \frac{1}{\sqrt{n}}\norm{\SamplingOperator(\sigma'(\tildeVarGraphonNDE\mylt))-\sigma'(\tildeVarGNDE\mylt)}_{\mathrm{F}}\norm{\VarGraphonNDEHiddenState\mylt}_{\spaceB} + \eta\mylt\stepjust{definitions \eqref{def: widehat mX ell}, \eqref{def: widehat tX ell}, \eqref{def: eta ell}; \eqs \eqref{norm of operator delta Un max}, \eqref{eq: hat X leq 1}}\\
    \leq\ & L_{\sigma'}\frac{1}{\sqrt{n}}\norm{\SamplingOperator(\tildeVarGraphonNDE\mylt)-\tildeVarGNDE\mylt}_{\mathrm{F}}\CYmax\myl + {L_\sigma}\eta\mylt\stepjust{\SigmaDerivativeLipschitz\ and Proposition \ref{proposition: norm of Y ell is bounded by norm of Y}}\\
    =\ & L_{\sigma'}\widetilde{\epsilon}\mylt\CYmax\myl + {L_\sigma}\eta\mylt.\stepjust{definition \eqref{def: widetilde epsilon ell} of $\widetilde{\epsilon}\mylt$}
\end{align*}
The proof is complete. 

\end{proof}

The following result can be obtained similarly as Lemma \ref{lemma: upper bound of LUn - LP X for all}. 
\begin{lemma}\label{lemma: upper bound of LUn - LP odot for all}
    Suppose that \ConvFilterLipschitz, \SigmaDerivativeLipschitz, \dPBoundedBelow, and \GraphonLipschitz\ hold. Let $u_j$, $j\in[n]$ be independent random variables following a distribution $P$. Let $\VarGraphonNDE$ be the solution of Graphon-NDE \eqref{Graphon-NDE} and $\VarGraphonNDEHiddenState$ be the solution of adjoint system \eqref{adjoint Grahpon-NDE: Y}. If \eq{eq: sup dP dUn} holds and $n$ satisfies \eqref{eq: n large enough for integral operator LP - LUn} hold, then with probability at least $1-\gamma_3$, for all $\ell\in[L]$, $g\in[F]$, $k\in[K]$ and $s\in\mathbb{Z}_k$, 
    \begin{equation}\label{eq: upper bound of LUn - LP hatX*Y f}
    \begin{aligned}
    &\supT\normb{\parens{\opLUn-\mathcal{L}_{P}}\parens{\opLP^{k-1-s}\hatVarGraphonNDE\mylt_g\VarGraphonNDEHiddenState\mylt_g}} \\
    &\lesssim\parens{\frac{\Csys[1]\sqrt{\log(4n_I/\gamma_1)} + \Csys[2]\sqrt{\log(4n_ILFK/\gamma_3)}}{\sqrt{n}}}\parens{\frac{c_{\max}}{c_{\min}}}^{k-1-s}\max\curlbraket{\Codotmax,\Codotlip}
    \end{aligned}
    \end{equation}
    where $\Csys[1]$ and $\Csys[2]$ are defined in \eqref{def: C1W} and \eqref{def: C2W}, respectively; and 
\begin{equation}\label{def: C odot max and lip}
\Codotmax:=\max\left\{\Codotmax\myl:\ell\in\mathbb{Z}_{L+1}\right\},\quad \Codotlip:=\max\left\{\Codotlip\myl:\ell\in\mathbb{Z}_{L+1}\right\}.
\end{equation}
    As a result, for all $\ell\in[L]$,
    \begin{equation}\label{eq: upper bound of LUn - LP odot for all}
\begin{aligned}
    &\supT\sqrt{ \sum_{g=1}^F\sum_{k=1}^{K}\sum_{s=0}^{k-1}\norm{(\opLUn-\mathcal{L}_{P})\parens{\opLP^{k-1-s}\hatVarGraphonNDE\mylt_g\VarGraphonNDEHiddenState\mylt_g}}_{\spaceb}^2}\\
    &\lesssim\parens{\frac{\Csys[1]\sqrt{\log(4n_I/\gamma_1)} + \Csys[2]\sqrt{\log(4n_ILFK/\gamma_3)}}{\sqrt{n}}}\Csys[3]\max\curlbraket{\Codotmax,\Codotlip}
\end{aligned}
\end{equation}
where $\Csys[3]$ is defined in \eqref{def: C3W}. 
\end{lemma}

\begin{proposition}\label{proposition: adjoint graphon-NN MSE}
Suppose that \ConvFilterLipschitz, \SigmaDerivativeLipschitz, \dPBoundedBelow, and \GraphonLipschitz\ hold. Let $\VarGNDE$, $\VarGraphonNDE$, $\VarGNDEHiddenState$ and $\VarGraphonNDEHiddenState$ be the solutions of GNDE \eqref{GNDE}, Graphon-NDE \eqref{Graphon-NDE}, adjoint GNDE \eqref{adjoint GNDE: Y}, and adjoint Graphon-NDE \eqref{adjoint Grahpon-NDE: Y}, respectively. Suppose that initial value condition \eqref{eq: initial values of GNDE and Graphon-NDE equal} holds. Let $\gamma_1,\gamma_2,\gamma_3\in(0,1)$ with $2\gamma_1+\gamma_2+\gamma_3<1$. Suppose that \eqs \eqref{upper bound of norm of L Un}, \eqref{eq: Ln - LWn 2norm}, \eqref{eq: upper bound of LUn - LP Xf} and \eqref{eq: upper bound of LUn - LP hatX*Y f} hold. Then for all $\ell\in\mathbb{Z}_{L}$ and $t\in[0,T]$, 
\begin{equation}\label{eq: eta 0 leq eta L + widetilde mathcal Q}
    \eta\mylt\leq \parens{{L_\sigma}FK\Ch}^{L-\ell}\eta^{(L,t)}+\parens{\sum_{s=1}^{L-\ell}\parens{FK\Ch}^{s-1}}\QY(n,\gamma_1,\gamma_2,\gamma_3),
\end{equation}
where 
\begin{equation}\label{def: Q Y}
\begin{aligned}
   & \QY(n,\gamma_1,\gamma_2,\gamma_3)\approx \parens{\frac{FK^2\Ch\frac{c_{\max}}{c_{\min}^2}\parens{\Codotmax+\mathcal{C}\CXmax}}{\sqrt{\alpha_n n}}} + \\
   &\hspace{3cm} \parens{FK\Ch\parens{\frac{2c_{\max}}{c_{\min}}}^K \Csys[3]\max\curlbraket{\Codotmax,\Codotlip,\CXmax,\CXlip}} \times\\
    &\hspace{1cm}\parens{\frac{\Csys[1]\parens{1+\mathcal{C}}\sqrt{\log(4n_I/\gamma_1)} + \mathcal{C}\Csys[2]\sqrt{\log(4n_ILFK/\gamma_2)} + \Csys[2]\sqrt{\log(4n_ILFK/\gamma_3)}}{\sqrt{n}} }\\
    &\hspace{3cm}=\mathcal{O}\parens{\frac{1}{\sqrt{\alpha_n n}}+\frac{\sqrt{\log(n_ILFK/\gamma_3)}+\sqrt{\log(n_ILFK/\gamma_2)}+\sqrt{\log(n_I/\gamma_1)}}{\sqrt{n}}}
\end{aligned}
\end{equation}
where $\mathcal{C}:=FK\Ch L_{\sigma'}\CYmax\tildeCe$ and 
\begin{equation}\label{def: C Y max}
\CYmax:=\max\curlbraket{\CYmax\myl:\ell\in\mathbb{Z}_{L+1}}.
\end{equation}
\end{proposition}

\begin{proof}
We obtain from Proposition \ref{proposition: eta ell+1 leq eta ell + Delta} that 
\begin{equation}\label{in the proof: eta l-1 leq eta l}
   \eta^{(\ell-1,t)}\leq  \parens{{L_\sigma}FK\Ch}\eta\mylt + \widetilde{\Delta}_{1}\mylt + \widetilde{\Delta}_{2}\mylt + \widetilde{\Delta}_{3}\mylt,\quad\ell\in[L],\ t\in[0,T],
\end{equation} 
with $\widetilde{\Delta}_{1}\mylt$, $\widetilde{\Delta}_{2}\mylt$ and $\widetilde{\Delta}_{3}\mylt$ defined in \eqs \eqref{def: tilde Delta1 ell t}, \eqref{def: tilde Delta2 ell t} and \eqref{def: tilde Delta3 ell t}, respectively. 

According to definition \eqref{def: tilde Delta1 ell t} of $\widetilde{\Delta}_{1}\mylt$, assumption of \eq{eq: Ln - LWn 2norm}, and Proposition \ref{proposition: X odot Y is bounded}, we get for all $\ell\in[L]$, $t\in[0,T]$, $\widetilde{\Delta}_{1}\mylt\lesssim FK^2\Ch\frac{c_{\max}}{c_{\min}^2}\frac{1}{\sqrt{\alpha_n n}} \Codotmax$. Moreover, according to Lemma \ref{lemma: upper bound of LUn - LP odot for all}, assumption of \eq{eq: upper bound of LUn - LP hatX*Y f} implies \eq{eq: upper bound of LUn - LP odot for all}. Then, according to definition \eqref{def: tilde Delta2 ell t} of $\widetilde{\Delta}_{2}\mylt$, estimates \eqref{upper bound of norm of L Un} and \eqref{eq: upper bound of LUn - LP odot for all}, we have for all $\ell\in[L]$, $t\in[0,T]$, 
$
\widetilde{\Delta}_{2}\mylt\lesssim FK\Ch\parens{\frac{2c_{\max}}{c_{\min}}}^K \parens{\frac{\Csys[1]\sqrt{\log(4n_I/\gamma_1)} + \Csys[2]\sqrt{\log(4n_ILFK/\gamma_3)}}{\sqrt{n}}}\Csys[3]\max\curlbraket{\Codotmax,\Codotlip}. 
$
It follows from Proposition \ref{proposition: graphon-NN MSE} that under assumptions of \eqs \eqref{upper bound of norm of L Un}, \eqref{eq: Ln - LWn 2norm}, \eqref{eq: upper bound of LUn - LP Xf}, we have estimate \eqref{eq: MSE layer L leq MSE layer 0}. This with initial condition \eqref{eq: initial values of GNDE and Graphon-NDE equal}, by Proposition \ref{prop: Graphon-MSE every layer}, we know that \eq{eq: sup tilde epsilon st leq QX} holds. Then, with definition \eqref{def: tilde Delta3 ell t} of $\widetilde{\Delta}_{3}\mylt$ and estimate \eqref{eq: sup tilde epsilon st leq QX}, we obtain that for all $\ell\in[L]$, $t\in[0,T]$, $
\widetilde{\Delta}_{3}\mylt\lesssim FK\Ch L_{\sigma'}\CYmax\tildeCe\QX(n,\gamma_1,\gamma_2)=\mathcal{C}\QX(n,\gamma_1,\gamma_2).$

By combining the above estimates for $\widetilde{\Delta}_1\mylt$, $\widetilde{\Delta}_2\mylt$, and $\widetilde{\Delta}_3\mylt$, we have 
\begin{equation}\label{in the proof: sum of tilde Delta 1 and tilde Delta 2 and tilde Delta 3 leq Q odot}
\widetilde{\Delta}_1\mylt+\widetilde{\Delta}_2\mylt+\widetilde{\Delta}_3\mylt\leq\QY(n,\gamma_1,\gamma_2,\gamma_3),\quad\ell\in[L],\ t\in[0,T],
\end{equation}
where $\QY(n,\gamma_1,\gamma_2,\gamma_3)$ is defined in \eq{def: Q Y}. We substitute estimate \eqref{in the proof: sum of tilde Delta 1 and tilde Delta 2 and tilde Delta 3 leq Q odot} into \eqref{in the proof: eta l-1 leq eta l} and get 
$
    \eta^{(\ell-1,t)}\leq\parens{{L_\sigma}FK\Ch}\eta\mylt+\QY(n,\gamma_1,\gamma_2,\gamma_3).
    $
A recursion leads to \eq{eq: eta 0 leq eta L + widetilde mathcal Q}. 
\end{proof}

\begin{proposition}\label{prop: Adjoint Graphon-MSE every layer}
Suppose that \eq{eq: eta 0 leq eta L + widetilde mathcal Q} holds. If the initial values of adjoint GNDE \eqref{adjoint GNDE: Y} and adjoint Graphon-NDE \eqref{adjoint Grahpon-NDE: Y} satisfy \eqref{eq: initial condition adjoint Hidden State}, then for all $\ell\in\mathbb{Z}_{L+1}$, there holds
    \begin{equation}\label{eq: MSE Y Yn adjoint equation}
    \supT \eta\mylt\leq \Ce \QY(n,\gamma_1,\gamma_2,\gamma_3),
    \end{equation}
    where $\Ce$ is defined by \eqref{def: HT}. 

\end{proposition}

\begin{proof}
The proof follows similarly as the argument present in Proposition \ref{prop: Graphon-MSE every layer}. We first show that  
\begin{equation}\label{in the proof: MSE Y Yn adjoint equation L layer}
    \supT \eta^{(L,t)}\leq \Ce^{(0)} \QY(n,\gamma_1,\gamma_2,\gamma_3),
\end{equation}
where $\Ce^{(0)}$ is defined in \eqref{def: HT (0)}. Let $\delta(t):=(\eta^{(L,t)})^2=(\MSE(\VarGraphonNDEHiddenState(\cdot,t),\VarGNDEHiddenState(t)))^2$. Since $\VarGNDEHiddenState$ and $\VarGraphonNDEHiddenState$ are solutions of \eqref{adjoint GNDE: Y} and \eqref{adjoint Grahpon-NDE: Y}, it follows from Cauchy-Schwartz inequality that 
    \begin{align*}
    \abs{\dt\delta(t)}&\leq \frac{2}{n}\norm{\VarGNDEHiddenState(t)-\SamplingOperator(\VarGraphonNDEHiddenState(\cdot,t))}_{\mathrm{F}}\norm{\VarGNDEHiddenState\mylt[0]-\SamplingOperator(\VarGraphonNDEHiddenState^{(0,t)})}_{\mathrm{F}}=2\eta^{(L,t)} \eta^{(0,t)}.
    \end{align*}
    Then we apply \eq{eq: eta 0 leq eta L + widetilde mathcal Q} with $\ell=0$ and get 
    \begin{align*}
    -\dt\delta(t)\leq\abs{\dt\delta(t)}&\leq 2\eta^{(L,t)}\parens{\parens{{L_\sigma}FK\Ch}^L\eta^{(L,t)} + \parens{\sum_{\ell=1}^{L}\parens{{L_\sigma}FK\Ch}^{\ell-1}}\QY(n,\gamma_1,\gamma_2,\gamma_3)}\\
&=2\parens{{L_\sigma}FK\Ch}^L\delta(t) + 2\parens{\sum_{\ell=1}^{L}\parens{{L_\sigma}FK\Ch}^{\ell-1}}\QY(n,\gamma_1,\gamma_2,\gamma_3)\sqrt{\delta(t)}.
    \end{align*}
    Note that assumption \eqref{eq: initial condition adjoint Hidden State} guarantees $\delta(T)=0$. We integrate the above inequality from $t$ to $T$. Then by Gr\"onwall's inequality (Lemma \ref{lemma: generalized Gronwall's inequality}), we obtain \eqref{in the proof: MSE Y Yn adjoint equation L layer} with $\Ce^{(0)}$ defined in \eqref{def: HT (0)}. It remains to show that for all $\ell\in\mathbb{Z}_{L}$, 
    \begin{equation}\label{in the proof: sup eta st leq QX all layers}
\supT\eta\mylt\leq \Ce^{(1)}\QY(n,\gamma_1,\gamma_2,\gamma_3),
    \end{equation}
    with $\Ce^{(1)}$ defined in \eqref{def: HT (1)}. This can be obtained in a similar way as in Proposition \eqref{prop: Graphon-MSE every layer} --- \eq{in the proof: sup eta st leq QX all layers} immediately follows from substituting the estimate \eqref{in the proof: MSE Y Yn adjoint equation L layer} into \eqref{eq: eta 0 leq eta L + widetilde mathcal Q}. 

\end{proof}

\begin{proof}[\textbf{Proof of Theorem \ref{theorem: MSE Y Yn adjoint equation final}}]
By Lemmas \ref{lemma: LUn-LP leq sqrt(n)}, \ref{lemma: Ln - LWn 2norm}, \ref{lemma: upper bound of LUn - LP odot for all} and \ref{lemma: upper bound of LUn - LP X for all}, with probability at least $1-2\gamma_1-\gamma_2-\gamma_3$, \eqs \eqref{eq: sup dP dUn}, \eqref{eq: Ln - LWn 2norm}, \eqref{eq: upper bound of LUn - LP Xf} and \eqref{eq: upper bound of LUn - LP hatX*Y f} hold. Then, according to Proposition \ref{proposition: adjoint graphon-NN MSE}, we have \eq{eq: eta 0 leq eta L + widetilde mathcal Q}, so the assumptions in Proposition~\ref{prop: Adjoint Graphon-MSE every layer} are satisfied. Recall that definition \eqref{def: eta ell} gives 
$
\eta\mylt[L] = \MSE(\VarGraphonNDEHiddenState\mylt[L],\VarGNDEHiddenState\mylt[L])=\MSE(\VarGraphonNDEHiddenState(\cdot,t),\VarGNDEHiddenState(t)). 
$
The desired result immediately follows from \eq{eq: MSE Y Yn adjoint equation} in Proposition~\ref{prop: Adjoint Graphon-MSE every layer} with $\ell=L$.
\end{proof}

\section{Infinite-node Convergence of Discretized Adjoint GNDEs}\label{Appendix: Infinite-node Convergence of Discretized Adjoint GNDEs}
We assume that \ConvFilterDiff, \SigmaTwiceDiff, \dPBoundedBelow, and \GraphonLipschitz\ are satisfied throughout this section. It is clear that if $\sigma$ satisfies \SigmaTwiceDiff, then for any $\mZ\in\spaceRnF$, there holds 
\begin{equation}\label{eq: sigma twice prime leq}
\normMax{\sigma''(\mZ)}\leq L_{\sigma'}. 
\end{equation}
\begin{proposition}\label{proposition: Yns Yn0'' bounded above by sqrt n}
Let $\VarGNDEHiddenState$ be the solution of adjoint GNDE \eqref{adjoint GNDE: Y}.
Suppose that \eqs \eqref{eq: MSE layer L leq MSE layer 0}, \eqref{eq: eta 0 leq eta L + widetilde mathcal Q}, and initial value conditions \eqref{eq: initial values of GNDE and Graphon-NDE equal}, \eqref{eq: initial condition adjoint Hidden State} hold. Then for all $\ell\in\mathbb{Z}_{L+1}$, there holds
    \begin{equation}\label{eq: sup mYnlt max leq P + CYmax}
\supT\normMax{\VarGNDEHiddenState\mylt}\leq \CmYGNDEmax,
    \end{equation}
    where 
$
\CmYGNDEmax=\mathcal{O}\parens{\frac{1}{\sqrt{\alpha_n}}+\sqrt{\log(n_ILFK/\gamma_3)}+\sqrt{\log(n_ILFK/\gamma_2)}+\sqrt{\log(n_I/\gamma_1)}}. $ In addition, if $n$ is large enough such that \eqref{eq: n is big enough 2} holds and 
\begin{equation}\label{eq: n is big enough 3}
\Ce\QY(n,\gamma_1,\gamma_2,\gamma_3) \leq \CYmax,
\end{equation} 
then for all $\ell\in\mathbb{Z}_{L+1}$, there holds 
\begin{align}
\supT\normF{\VarGNDEHiddenState\mylt}&\leq \CmYGNDE\sqrt{n},\label{eq: estimate of mYnlt F norm}\\
\supT\normF{\frac{d}{dt}\VarGNDEHiddenState\mylt}&\leq \CmYDeriGNDE\sqrt{n},\label{eq: dYnLt/dt F norm bound}
\end{align}
where $\CmYGNDE:=2\CYmax$ and 
$$\CmYDeriGNDE:=\mathcal{O}\parens{\frac{1}{\sqrt{\alpha_n}}+\sqrt{\log(n_ILFK/\gamma_3)}+\sqrt{\log(n_ILFK/\gamma_2)}+\sqrt{\log(n_I/\gamma_1)}}.$$
As a result, 
\begin{align}\label{eq: second derivative dYnLt/dt F norm bound}
\supT\normF{\frac{d^2}{dt^2}\VarGNDEHiddenState(t)}\leq \CmYDeriGNDE\sqrt{n}.
\end{align}
\end{proposition}

\begin{proof}
We first prove \eq{eq: sup mYnlt max leq P + CYmax}. Given \eq{eq: eta 0 leq eta L + widetilde mathcal Q} and initial condition \eqref{eq: initial condition adjoint Hidden State}, by Proposition \ref{prop: Adjoint Graphon-MSE every layer}, \eq{eq: MSE Y Yn adjoint equation} holds. Then by definition \eqref{def: eta ell} of $\eta\mylt$ that for all $\ell\in\mathbb{Z}_{L+1}$, 
\begin{equation*}
    \normMax{\SamplingOperator(\VarGraphonNDEHiddenState\mylt)-\VarGNDEHiddenState\mylt}\leq \normF{\SamplingOperator(\VarGraphonNDEHiddenState\mylt)-\VarGNDEHiddenState\mylt}\leq{\Ce} \QY(n,\gamma_1,\gamma_2,\gamma_3)\sqrt{n}.
\end{equation*}
Then by triangle inequality and Proposition \ref{proposition: norm of Y ell is bounded by norm of Y} and \eq{norm of operator delta Un max}, we obtain that 
\begin{equation*}
\normMax{\VarGNDEHiddenState\mylt}\leq{\Ce} \QY(n,\gamma_1,\gamma_2,\gamma_3)\sqrt{n} + \normMax{\SamplingOperator(\VarGraphonNDEHiddenState\mylt)} \leq {\Ce} \QY(n,\gamma_1,\gamma_2,\gamma_3)\sqrt{n} + \CYmax. 
\end{equation*}
Then \eq{eq: sup mYnlt max leq P + CYmax} directly follows from definition \eqref{def: Q Y} of $\QY$.

We next prove \eq{eq: estimate of mYnlt F norm}. Again by \eq{eq: MSE Y Yn adjoint equation} and triangle inequality, we have for all $\ell\in\mathbb{Z}_{L+1}$, 
\begin{align*}
\normF{\VarGNDEHiddenState\mylt}&\leq \sqrt{n}{\Ce} \QY(n,\gamma_1,\gamma_2,\gamma_3) + \normF{\SamplingOperator(\VarGraphonNDEHiddenState\mylt)}\\
&\leq \sqrt{n} \parens{{\Ce} \QY(n,\gamma_1,\gamma_2,\gamma_3) + \normB{\VarGraphonNDEHiddenState\mylt}}\stepjust{\eq{norm of operator delta Un}}\\
&\leq \sqrt{n}\parens{{\Ce} \QY(n,\gamma_1,\gamma_2,\gamma_3) + \CYmax}\stepjust{Proposition \ref{proposition: norm of Y ell is bounded by norm of Y} and \eq{def: C Y max}}\\
&\leq 2\CYmax\sqrt{n}=\CmYGNDE\sqrt{n}.\stepjust{\eq{eq: n is big enough 3}}
\end{align*}
This proves \eq{eq: estimate of mYnlt F norm}.

We proceed to show \eqref{eq: dYnLt/dt F norm bound}. We begin with the case of $\ell\in\mathbb{Z}_L$. Recall the updating rule \eqref{updating rule for mY ell}, which gives $\VarGNDEHiddenState\mylt=\phi_{\vh\mylt[\ell+1]}^*\parens{\hatVarGNDE\mylt[\ell+1]\odot\VarGNDEHiddenState\mylt[\ell+1]},\ \ell\in\mathbb{Z}_L$. Then it follows from Lemma \ref{lemma: d/dt phihlt(Zt) estimate} that 
\begin{align}
\normF{\dt\VarGNDEHiddenState\mylt}\leq FK\Liph\underbrace{\normF{\hatVarGNDE\mylt[\ell+1]\odot\VarGNDEHiddenState\mylt[\ell+1]}}_{\text{denoted by }\Delta_1}+FK\Ch\underbrace{\normF{\dt\parens{\hatVarGNDE\mylt[\ell+1]\odot\VarGNDEHiddenState\mylt[\ell+1]}}}_{\text{denoted by }\Delta_2}\label{in the proof: dYnltdt leq Delta1 + Delta2}.
\end{align}
Note that by \eqs \eqref{eq: upper bound of matrix Z odot V}, \eqref{eq: hat X leq 1} and \eqref{eq: estimate of mYnlt F norm}, we have $\Delta_1\leq \normMax{\hatVarGNDE\mylt[\ell+1]}\normF{\VarGNDEHiddenState\mylt[\ell+1]}\leq {L_\sigma}\normF{\VarGNDEHiddenState\mylt[\ell+1]}\leq {L_\sigma}\CmYGNDE\sqrt{n}$. Moreover, 
\begin{align*}
\Delta_2&\leq \normF{\frac{d\hatVarGNDE\mylt[\ell+1]}{dt}}\normMax{\VarGNDEHiddenState\mylt[\ell+1]}+\normMax{\hatVarGNDE\mylt[\ell+1]}\normF{\frac{d\VarGNDEHiddenState\mylt[\ell+1]}{dt}} \stepjust{Chain rule, triangle inequality and \eq{eq: upper bound of matrix Z odot V}}\\
&\leq \CmXHatDeriGNDE\sqrt{n} \normMax{\VarGNDEHiddenState\mylt[\ell+1]} + {L_\sigma}\normF{\frac{d\VarGNDEHiddenState\mylt[\ell+1]}{dt}}\stepjust{\eq{eq: estimate of deri mXnlt hat F norm} in Proposition \ref{proposition: Xns Xn0'' bounded above by sqrt n} and \eq{eq: hat X leq 1}}\\
&\leq \CmXHatDeriGNDE\sqrt{n} \CmYGNDEmax + {L_\sigma}\normF{\frac{d\VarGNDEHiddenState\mylt[\ell+1]}{dt}}.\stepjust{\eq{eq: sup mYnlt max leq P + CYmax}}
\end{align*}
We substitute estimates of $\Delta_1$ and $\Delta_2$ into \eq{in the proof: dYnltdt leq Delta1 + Delta2}, and get 
\begin{align*}
\normF{\dt\VarGNDEHiddenState\mylt}&\leq {L_\sigma}FK\Liph \CmYGNDE\sqrt{n} + FK\Ch\parens{\CmXHatDeriGNDE\sqrt{n} \CmYGNDEmax + {L_\sigma}\normF{\frac{d\VarGNDEHiddenState\mylt[\ell+1]}{dt}}}\\
&=\underbrace{\parens{\CmYGNDE {L_\sigma}FK\Liph+FK\Ch\CmXHatDeriGNDE \CmYGNDEmax}}_{\text{denoted by $\mathcal{C}_1(\alpha_n,\gamma_1,\gamma_2,\gamma_3)$}}\sqrt{n} + \parens{{L_\sigma}FK\Ch}\normF{\frac{d\VarGNDEHiddenState\mylt[\ell+1]}{dt}}.
\end{align*}
A recursion gives 
\begin{equation}\label{in the proof: intermediate recursion estiamte dmYnlt dt}
    \normF{\dt\VarGNDEHiddenState\mylt}\leq \parens{{L_\sigma}FK\Ch}^{L-\ell}\normF{\frac{d}{dt}\VarGNDEHiddenState\mylt[L]}+\parens{\sum_{s=1}^{L-\ell}\parens{{L_\sigma}FK\Ch}^{s-1}}\mathcal{C}_1(\alpha_n,\gamma_1,\gamma_2,\gamma_3)\sqrt{n},\quad\ell\in\mathbb{Z}_L. 
\end{equation}
Note that $\VarGNDEHiddenState\mylt[L]$ is the solution of adjoint system \eqref{adjoint GNDE: Y}, that is 
\begin{equation}\label{in the proof: dt mYnlt = mYnlt L}
\dt\VarGNDEHiddenState\mylt[L]=\VarGNDEHiddenState\mylt[0].    
\end{equation}
Therefore, \eq{in the proof: intermediate recursion estiamte dmYnlt dt} combining with \eq{in the proof: dt mYnlt = mYnlt L} and \eq{eq: estimate of mYnlt F norm} implies that for all $\ell\in\mathbb{Z}_L$, 
\begin{align*}
\normF{\dt\VarGNDEHiddenState\mylt}\leq\underbrace{\parens{\CmYGNDE\parens{{L_\sigma}FK\Ch}^{L-\ell}+\parens{\sum_{s=1}^{L-\ell}\parens{{L_\sigma}FK\Ch}^{s-1}}\mathcal{C}_1(\alpha_n,\gamma_1,\gamma_2,\gamma_3)}}_{\text{denoted by }\mathcal{C}_2(\alpha_n,\gamma_1,\gamma_2,\gamma_3)}\sqrt{n}
\end{align*}
By letting $\CmYDeriGNDE:=\max\curlbraket{\mathcal{C}_2(\alpha_n,\gamma_1,\gamma_2,\gamma_3):\ell\in\mathbb{Z}_{L+1}}^{\footnotemark}$, \footnotetext{The sum inside is regarded as $0$ when $\ell=L$.}
we find that \eq{eq: dYnLt/dt F norm bound} holds for $\ell\in\mathbb{Z}_L$. For the case of $\ell=L$, we observe that \eq{eq: dYnLt/dt F norm bound} still holds due to \eqref{in the proof: dt mYnlt = mYnlt L} and \eqref{eq: dYnLt/dt F norm bound}. This completes the proof of \eq{eq: dYnLt/dt F norm bound} for all $\ell\in\mathbb{Z}_{L+1}$.

We finally prove \eqref{eq: second derivative dYnLt/dt F norm bound}. It follows from \eqref{in the proof: dt mYnlt = mYnlt L} that $\ddt\VarGNDEHiddenState\mylt[L]=\dt\VarGNDEHiddenState\mylt[0]$. Then the desired inequality \eqref{eq: second derivative dYnLt/dt F norm bound} immediately follows from \eq{eq: dYnLt/dt F norm bound} with $\ell=0$ and recalling the notation $\VarGNDEHiddenState\mylt[L]=\VarGNDEHiddenState(t)$. 
\end{proof}

\begin{lemma}\label{lemma: Ynlm - Unltm estimates}
If \eqs \eqref{eq: hat mXlnm leq 1/M} and \eqref{eq: sup mYnlt max leq P + CYmax} hold, then for all $\ell\in\mathbb{Z}_{L}$ and $m\in[M]$, 
\begin{equation}\label{eq: error bound of discretized Yn operator D ell layer}
    \begin{aligned}
\normF{\VarGNDEHiddenState^{[\ell,m]}-\VarGNDEHiddenState\mylt[\ell][t_m]} &\leq \parens{{L_\sigma}FK\Ch}^{L-\ell} \normF{\VarGNDEHiddenState^{[m]}-\VarGNDEHiddenState(t_m)} + \\
&\ \parens{\CXdiffHat\CmYGNDEmax\sum_{s=1}^{L-\ell} \parens{{L_\sigma}FK\Ch}^{s}}\frac{\sqrt{n}}{M}.
    \end{aligned}
\end{equation}
    In particular, 
    \begin{equation}\label{eq: error bound of discretized Yn operator D}
    \begin{split}
        \normF{\DX^{[m]}(\VarGNDEHiddenState^{[m]})-\DX^{(t_m)}(\VarGNDEHiddenState(t_m))} &\leq \parens{{L_\sigma}FK\Ch}^{L} \normF{\VarGNDEHiddenState^{[m]}-\VarGNDEHiddenState(t_m)} +\\
        &\parens{\CXdiffHat\CmYGNDEmax\sum_{s=1}^{L} \parens{{L_\sigma}FK\Ch}^{s}}\frac{\sqrt{n}}{M}.
    \end{split}
    \end{equation}
\end{lemma}

    \begin{proof}
Note that for each $\ell\in[L]$, 
\begin{align*}
&\normF{\VarGNDEHiddenState^{[\ell-1,m]}-\VarGNDEHiddenState\mylt[\ell-1][t_m]} = \normF{\phi_{\vh\mylt[\ell][t_m]}^*\parens{\hatVarGNDE^{[\ell,m]}\odot\VarGNDEHiddenState^{[\ell,m]}}-\phi_{\vh\mylt[\ell][t_m]}^*\parens{\hatVarGNDE\mylt[\ell][t_m]\odot\VarGNDEHiddenState\mylt[\ell][t_m]}} \\
    &\leq FK\Ch\parens{\normMax{\hatVarGNDE^{[\ell,m]}}\normF{\VarGNDEHiddenState^{[\ell,m]}-\VarGNDEHiddenState\mylt[\ell][t_m]} + \normF{\hatVarGNDE^{[\ell,m]}-\hatVarGNDE\mylt[\ell][t_m]}\normMax{\VarGNDEHiddenState\mylt[\ell][t_m]}} \stepjust{\eq{eq: phi * t1 - t2 XP - ZQ} in Lemma \ref{lemma: phi phi* t1 - t2 X - Z PQ}}\\
    &\leq FK\Ch\parens{{L_\sigma}\normF{\VarGNDEHiddenState^{[\ell,m]}-\VarGNDEHiddenState\mylt[\ell][t_m]} + \frac{\CXdiffHat\CmYGNDEmax\sqrt{n}}{M}}. \stepjust{\SigmaDiff, and \eqs \eqref{eq: sup mYnlt max leq P + CYmax}, \eqref{eq: hat mXlnm leq 1/M}}
\end{align*}
A recursion gives that for $\ell\in\mathbb{Z}_{L}$, 
\begin{align*}
    \normF{\VarGNDEHiddenState^{[\ell,m]}-\VarGNDEHiddenState\mylt[\ell][t_m]} &\leq \parens{{L_\sigma}FK\Ch}^{L-\ell} \normF{\VarGNDEHiddenState^{[L,m]}-\VarGNDEHiddenState\mylt[L][t_m]} + \frac{\CXdiffHat\CmYGNDEmax\sqrt{n}}{M}\parens{\sum_{s=1}^{L-\ell} \parens{{L_\sigma}FK\Ch}^{s}}.
\end{align*}
Recalling notations $\VarGNDEHiddenState^{[L,m]}=\VarGNDEHiddenState^{[m]}$ and $\VarGNDEHiddenState\mylt[L][t_m]=\VarGNDEHiddenState(t_m)$, we obtain \eqref{eq: error bound of discretized Yn operator D ell layer} for $\ell\in\mathbb{Z}_L$. Particularly, we take $\ell=0$ in \eq{eq: error bound of discretized Yn operator D ell layer} and note that $\VarGNDEHiddenState^{[0,m]}=\DX^{[m]}(\VarGNDEHiddenState^{[m]})$ and $\VarGNDEHiddenState\mylt[0][t_m]=\DX^{(t_m)}(\VarGNDEHiddenState(t_m))$, which proves \eqref{eq: error bound of discretized Yn operator D}. 
\end{proof}

\begin{proposition}\label{proposition: error Ynm - Yntm estimate}
Let $\VarGNDEHiddenState$ be the solution of adjoint GNDE \eqref{adjoint GNDE: Y}, and $\VarGNDEHiddenState^{[m]}$ be generated from \eqref{eq: discretized Yn}. Suppose that \eqs \eqref{eq: MSE layer L leq MSE layer 0}, \eqref{eq: eta 0 leq eta L + widetilde mathcal Q}, and initial value conditions \eqref{eq: initial values of GNDE and Graphon-NDE equal}, \eqref{eq: initial condition adjoint Hidden State}, \eqref{eq: initial value condition on GNDE and discretized GNDE} hold. Suppose that $n$ is large enough such that \eqref{eq: n is big enough 2} and \eqref{eq: n is big enough 3} hold. If the initial values of \eqref{adjoint GNDE: Y} and \eqref{eq: discretized Yn} satisfy \eqref{eq: initial condition GNDE discretized GNDE Adjoint A}, then for all $m\in\mathbb{Z}_M$ and $\ell\in\mathbb{Z}_{L+1}$, it holds that 
    \begin{align}\label{eq: Ynlm error}
        \frac{1}{\sqrt{n}}\normF{\VarGNDEHiddenState^{[\ell,m]}-\VarGNDEHiddenState\mylt[\ell][t_m]}\leq  \frac{\CYdiff}{M},
    \end{align}
    where $\CYdiff=\mathcal{O}\parens{\frac{1}{\sqrt{\alpha_n}}+\sqrt{\log(n_ILFK/\gamma_3)}+\sqrt{\log(n_ILFK/\gamma_2)}+\sqrt{\log(n_I/\gamma_1)}}$. In addition, if $M$ is large enough such that 
    \begin{equation}\label{eq: M is large enough}
\frac{\CYdiff}{M}\leq \CYmax,
    \end{equation}
then for all $m\in[M]$ and $\ell\in\mathbb{Z}_{L+1}$, 
    \begin{align}\label{eq: normF of Yn[lm]}
        \normF{\VarGNDEHiddenState^{[\ell,m]}}\leq \CmYDistGNDE\sqrt{n}.
    \end{align}
    where $\CmYDistGNDE:=3\CYmax$. 
\end{proposition}

\begin{proof}
Let $C_{\Delta \VarGNDEHiddenState}(\alpha_n,\gamma_1,\gamma_2,\gamma_3):=\max\curlbraket{C_{\Delta \VarGNDEHiddenState}^{(0)}(\alpha_n,\gamma_1,\gamma_2,\gamma_3),C_{\Delta \VarGNDEHiddenState}^{(1)}(\alpha_n,\gamma_1,\gamma_2,\gamma_3)}$ where 
    \begin{align*}
&C_{\Delta \VarGNDEHiddenState}^{(0)}(\alpha_n,\gamma_1,\gamma_2,\gamma_3)\\
&:=
\frac{\parens{e^{T\parens{{L_\sigma}FK\Ch}^L }-1}\parens{2\CXdiffHat\CmYGNDEmax\sum_{s=1}^{L} \parens{{L_\sigma}FK\Ch}^{s}+T\CmYDeriGNDE}}{2\parens{{L_\sigma}FK\Ch}^L},
    \end{align*}
    \begin{align*}
C_{\Delta \VarGNDEHiddenState}^{(1)}(\alpha_n,\gamma_1,\gamma_2,\gamma_3):=&\parens{\sum_{s=1}^{L} \parens{{L_\sigma}FK\Ch}^s}\parens{C_{\Delta \VarGNDEHiddenState}^{(0)}(\alpha_n,\gamma_1,\gamma_2,\gamma_3)+\CXdiffHat\CmYGNDEmax}.
    \end{align*}
We first prove that 
\begin{align}\label{eq: Ynm error}
\frac{1}{\sqrt{n}}\normF{\VarGNDEHiddenState^{[m]}-\VarGNDEHiddenState(t_m)}\leq \frac{C_{\Delta \VarGNDEHiddenState}^{(0)}(\alpha_n,\gamma_1,\gamma_2,\gamma_3)}{M}, 
\end{align}
which is \eq{eq: Ynlm error} in the case of $\ell=L$. We expand $\VarGNDEHiddenState(t)$ in a Taylor series at $t_m$, and get 
\begin{equation*}
    \VarGNDEHiddenState(t)=\VarGNDEHiddenState(t_m)+(t-t_m) \dt\VarGNDEHiddenState(t_m)+\frac{(t-t_m)^2}{2}\ddt\VarGNDEHiddenState(\xi_{t}), 
\end{equation*}
where $\xi_t$ is between $t$ and $t_m$. Together with \eqref{eq: discretized Yn} and \eqref{adjoint GNDE: Y}, we obtain
\begin{align}
    &\normF{\VarGNDEHiddenState^{[m-1]}-\VarGNDEHiddenState(t_{m-1})}=\normF{\braket{\VarGNDEHiddenState^{[m]}+\kappa\cdot\DX^{[m]}(\VarGNDEHiddenState^{[m]})} - \braket{\VarGNDEHiddenState(t_m)-\kappa \dt\VarGNDEHiddenState(t_m)+\frac{\kappa^2}{2}\ddt\VarGNDEHiddenState(\xi_{t_{m-1}})}}\nonumber\\
    \leq\ & \normF{\VarGNDEHiddenState^{[m]}-\VarGNDEHiddenState(t_m)} + \kappa \normF{\DX^{[m]}(\VarGNDEHiddenState^{[m]})-\DX^{(t_m)}(\VarGNDEHiddenState(t_m))}+\frac{\kappa^2}{2}\normF{\ddt\VarGNDEHiddenState(\xi_{t_{m-1}})}\label{in the proof: Ynm-e - Yntm-1 leq Ynm - Yntm}.
\end{align}
Note that all assumptions in Proposition \ref{proposition: Yns Yn0'' bounded above by sqrt n} and Proposition \ref{proposition: discretized GNDE error leq 1/M} are assumed in the current proposition, and hence \eqs \eqref{eq: hat mXlnm leq 1/M} and \eqref{eq: sup mYnlt max leq P + CYmax} hold. Then by Lemma \ref{lemma: Ynlm - Unltm estimates}, we have \eq{eq: error bound of discretized Yn operator D}, which implies $\normF{\DX^{[m]}(\VarGNDEHiddenState^{[m]}) - \DX^{(t_m)}(\VarGNDEHiddenState(t_m))} \leq a\normF{\VarGNDEHiddenState^{[m]}-\VarGNDEHiddenState(t_m)}+b\sqrt{n}/M$, where $a:=\parens{{L_\sigma}FK\Ch}^L$ and $b:=\CXdiffHat\CmYGNDEmax\sum_{s=1}^{L} \parens{{L_\sigma}FK\Ch}^{s}$. It follows from \eq{eq: second derivative dYnLt/dt F norm bound} in Proposition \ref{proposition: Yns Yn0'' bounded above by sqrt n} that $\supT\normF{\frac{d^2}{dt^2}\VarGNDEHiddenState(t)}\leq \CmYDeriGNDE\sqrt{n}$. Let $y^{(m)}:=\frac{1}{\sqrt{n}}\normF{\VarGNDEHiddenState^{[m]}-\VarGNDEHiddenState(t_m)}$. Then, \eq{in the proof: Ynm-e - Yntm-1 leq Ynm - Yntm} implies $y^{(m-1)}\leq \parens{1+\kappa a}y^{(m)} + \frac{\kappa}{M}b + \kappa^2\frac{\CmYDeriGNDE}{2}$. A recursion gives that for all $m\in\mathbb{Z}_M$, 
\begin{align*}
    y^{(m)} \leq \parens{1+\kappa a}^{M-m}y^{(M)}+\parens{\sum_{s=0}^{M-m-1}(1+\kappa a)^s}\parens{\frac{\kappa}{M}b + \kappa^2\frac{\CmYDeriGNDE}{2}}.
\end{align*}
Note that $\kappa=T/M$, $t_m=\kappa m$; and the initial condition \eqref{eq: initial condition GNDE discretized GNDE Adjoint A} leads to $y^{(M)}=0$; and direct computation gives
$\sum_{s=0}^{M-m-1}(1+\kappa a)^s=\frac{\parens{1+\kappa a}^{M-m}-1}{\kappa a}$ and $\parens{1+\kappa a}^{M-m}\leq e^{\kappa a(M-m)}=e^{a(T-t_m)}\leq e^{aT}$. These estimates combining with the last inequality yields \eq{eq: Ynm error}. 

Then \eq{eq: Ynlm error} for $\ell\in\mathbb{Z}_L$ can be immediately obtained by substituting \eq{eq: Ynm error} into estimate \eqref{eq: error bound of discretized Yn operator D ell layer}. 

Note that 
\begin{align*}
\normF{\VarGNDEHiddenState^{[\ell,m]}}&\leq  \sqrt{n}\frac{\CYdiff}{M} + \normF{\VarGNDEHiddenState\mylt[\ell][t_m]}\stepjust{\eq{eq: Ynlm error} and triangle inequality}\\
&\leq \sqrt{n}\parens{\frac{\CYdiff}{M} +\CmYGNDE}\stepjust{\eq{eq: estimate of mYnlt F norm} in Proposition \ref{proposition: Yns Yn0'' bounded above by sqrt n}}\\
&\leq 3\CYmax\sqrt{n}=\CmYDistGNDE\sqrt{n}\stepjust{\eq{eq: M is large enough} and definition of $\CmYDistGNDE$}
\end{align*}
which proves \eqref{eq: normF of Yn[lm]}.

\end{proof}

\begin{proof}[\textbf{Proof of Theorem \ref{theorem: discretized adjoint gnde to adjoint gnde Hidden state}.}]
By Lemmas \ref{lemma: LUn-LP leq sqrt(n)}, \ref{lemma: Ln - LWn 2norm}, \ref{lemma: upper bound of LUn - LP X for all} and \ref{lemma: upper bound of LUn - LP odot for all}, with probability at least $1-2\gamma_1-\gamma_2-\gamma_3$, \eqs \eqref{eq: sup dP dUn}, \eqref{eq: Ln - LWn 2norm}, \eqref{eq: upper bound of LUn - LP Xf} and \eqref{eq: upper bound of LUn - LP hatX*Y f} hold. Hence, by Proposition \ref{proposition: graphon-NN MSE}, we have \eq{eq: MSE layer L leq MSE layer 0}; by Proposition \ref{proposition: adjoint graphon-NN MSE}, we have  \eqref{eq: eta 0 leq eta L + widetilde mathcal Q}. Therefore, all assumptions in Proposition \ref{proposition: error Ynm - Yntm estimate} are satisfied. The desired result directly follows from \eq{eq: Ynlm error} in Proposition \ref{proposition: error Ynm - Yntm estimate} with $\ell=L$. 
\end{proof}

\section{DTO versus OTD for Gradients of Hidden States}\label{Appendix: DTO versus OTD for Gradients of Hidden States}
We assume that \ConvFilterDiff, \SigmaTwiceDiff, \dPBoundedBelow, and \GraphonLipschitz\ are satisfied throughout this section. The proof of the following lemma is very similar to Lemma \ref{lemma: Ynlm - Unltm estimates}. 
\begin{lemma}\label{lemma: Gnlm - Unltm estimates}
Let $\VarGNDEHiddenState$ be the solution of adjoint GNDE \eqref{adjoint GNDE: Y}, and $\VarDTOGradient^{[m]}$ be generated from \eqref{eq: discretized Gn}. If \eqref{eq: sup mYnlt max leq P + CYmax}, \eqref{eq: estimate of mYnlt F norm}, \eqref{eq: hat mXlnm leq 1/M} and \eqref{eq: hat mXln diff [m] leq 1/M every layer} hold, then for all $\ell\in\mathbb{Z}_{L}$ and $m\in[M]$, 
\begin{equation}
    \label{eq: error bound of discretized Gn operator D ell layer}
    \begin{aligned}
    &\normF{\VarDTOGradient^{[\ell,m]}-\VarGNDEHiddenState\mylt[\ell][t_m]} \leq \parens{{L_\sigma}FK\Ch}^{L-\ell} \normF{\VarDTOGradient^{[m]}-\VarGNDEHiddenState(t_m)}\\
    &+ \parens{\parens{FK\Ch\parens{\CXdiffDistHat+\CXdiffHat}\CmYGNDEmax+FK\Liph\CmYGNDE}\sum_{s=0}^{L-\ell-1}\parens{{L_\sigma}FK\Ch}^s}\frac{\sqrt{n}}{M}.
    \end{aligned}
    \end{equation}
    In particular, 
    \begin{equation}\label{eq: error bound of discretized Gn operator D}
    \begin{aligned}
        &\normF{\DX^{[m-1]}(\VarDTOGradient^{[m]})-\DX^{(t_m)}(\VarGNDEHiddenState(t_m))} \leq \parens{{L_\sigma}FK\Ch}^{L} \normF{\VarDTOGradient^{[m]}-\VarGNDEHiddenState^{(t_m)}}\\
        &+ \parens{\parens{FK\Ch\parens{\CXdiffDistHat+\CXdiffHat}\CmYGNDEmax+FK\Liph\CmYGNDE}\sum_{s=0}^{L-1}\parens{{L_\sigma}FK\Ch}^s}\frac{\sqrt{n}}{M}. 
    \end{aligned}
    \end{equation}
\end{lemma}

\begin{proof}
Note that for each $\ell\in[L]$, 
\begin{align*}
    &\normF{\VarDTOGradient^{[\ell-1,m]}-\VarGNDEHiddenState\mylt[\ell-1][t_m]} = \normF{\phi_{\vh\mylt[\ell][t_{m-1}]}^*\parens{\hatVarGNDE^{[\ell,m-1]}\odot\VarDTOGradient^{[\ell,m]}}-\phi_{\vh\mylt[\ell][t_m]}^*\parens{\hatVarGNDE\mylt[\ell][t_m]\odot\VarGNDEHiddenState\mylt[\ell][t_m]}} \\
    \leq\ &FK\Ch\parens{\normMax{\hatVarGNDE^{[\ell,m-1]}}\normF{\VarDTOGradient^{[\ell,m]}-\VarGNDEHiddenState\mylt[\ell][t_m]} + \normF{\hatVarGNDE^{[\ell,m-1]}-\hatVarGNDE\mylt[\ell][t_m]}\normMax{\VarGNDEHiddenState\mylt[\ell][t_m]}}+FK\frac{\Liph}{M}\normF{\VarGNDEHiddenState\mylt[\ell][t_m]}\stepjust{\eq{eq: phi * t1 - t2 XP - ZQ} in Lemma \ref{lemma: phi phi* t1 - t2 X - Z PQ}}
\end{align*}
Recall that by \SigmaDiff, $
\normMax{\hatVarGNDE^{[\ell,m-1]}}\leq {L_\sigma}$; by \eq{eq: sup mYnlt max leq P + CYmax}, $\normMax{\VarGNDEHiddenState\mylt[\ell][t_m]}\leq \CmYGNDEmax$; by \eq{eq: estimate of mYnlt F norm}, $\normF{\VarGNDEHiddenState\mylt[\ell][t_m]}\leq \CmYGNDE\sqrt{n}$; and by \eqs \eqref{eq: hat mXlnm leq 1/M} and \eqref{eq: hat mXln diff [m] leq 1/M every layer}, $\normF{\hatVarGNDE^{[\ell,m-1]}-\hatVarGNDE\mylt[\ell][t_m]}\leq \normF{\hatVarGNDE^{[\ell,m-1]}-\hatVarGNDE^{[\ell,m]}} + \normF{\hatVarGNDE^{[\ell,m]}-\hatVarGNDE\mylt[\ell][t_m]}\leq \parens{\CXdiffDistHat+\CXdiffHat}\frac{\sqrt{n}}{M}.$ Therefore, for $\ell\in[L]$, 
\begin{equation*}
\begin{aligned}
    \normF{\VarDTOGradient^{[\ell-1,m]}-\VarGNDEHiddenState\mylt[\ell-1][t_m]}&\leq {L_\sigma}FK\Ch\normF{\VarDTOGradient^{[\ell,m]}-\VarGNDEHiddenState\mylt[\ell][t_m]} + \\
&\parens{FK\Ch\parens{\CXdiffDistHat+\CXdiffHat}\CmYGNDEmax+FK\Liph\CmYGNDE}\frac{\sqrt{n}}{M}.
\end{aligned}
\end{equation*}
A recursion gives that for $\ell\in\mathbb{Z}_{L}$, 
\begin{align*}
&\normF{\VarDTOGradient^{[\ell,m]}-\VarGNDEHiddenState\mylt[\ell][t_m]} \leq \parens{{L_\sigma}FK\Ch}^{L-\ell} \normF{\VarDTOGradient^{[L,m]}-\VarGNDEHiddenState\mylt[L][t_m]}\\
&+ \parens{FK\Ch\parens{\CXdiffDistHat+\CXdiffHat}\CmYGNDEmax+FK\Liph\CmYGNDE}\frac{\sqrt{n}}{M}\parens{\sum_{s=0}^{L-\ell-1}\parens{{L_\sigma}FK\Ch}^s}.
\end{align*}
Recalling notations $\VarDTOGradient^{[L,m]}=\VarDTOGradient^{[m]}$ and $\VarGNDEHiddenState\mylt[L][t_m]=\VarGNDEHiddenState(t_m)$, we obtain \eqref{eq: error bound of discretized Gn operator D ell layer} for $\ell\in\mathbb{Z}_L$. Particularly, we take $\ell=0$ in \eq{eq: error bound of discretized Gn operator D ell layer} and note that $\VarDTOGradient^{[0,m]}=\DX^{[m-1]}(\VarDTOGradient^{[m]})$ and $\VarGNDEHiddenState\mylt[0][t_m]=\DX^{(t_m)}(\VarGNDEHiddenState(t_m))$, which proves \eqref{eq: error bound of discretized Gn operator D}. 
\end{proof}

\begin{proposition}\label{proposition: error Gnm - Yntm estimate} 
If all assumptions of Proposition \ref{proposition: error Ynm - Yntm estimate} are satisfied, then for all $m\in[M]$ and $\ell\in\mathbb{Z}_{L+1}$, 
    \begin{align}\label{eq: Gnm error every layer}
        \frac{1}{\sqrt{n}}\normF{\VarDTOGradient^{[\ell,m]}-\VarGNDEHiddenState\mylt[\ell][t_m]}\leq \frac{\CGdiff}{M},
    \end{align}
    where $\CGdiff=\mathcal{O}\parens{\frac{1}{\sqrt{\alpha_n}}+\sqrt{\log(n_ILFK/\gamma_3)}+\sqrt{\log(n_ILFK/\gamma_2)}+\sqrt{\log(n_I/\gamma_1)}}$.
\end{proposition}
The proof of Proposition \ref{proposition: error Gnm - Yntm estimate} is similar to that of Proposition \ref{proposition: error Ynm - Yntm estimate}, and is therefore omitted.

\begin{proof}[\textbf{Proof of Theorem \ref{theorem: hidden states gradients DTO and OTD}}]
We notice from the proof of Theorem \ref{theorem: discretized adjoint gnde to adjoint gnde Hidden state} that with probability at least $1-2\gamma_1-\gamma_2-\gamma_3$, the assumptions in Proposition \ref{proposition: error Ynm - Yntm estimate} (or Proposition \ref{proposition: error Gnm - Yntm estimate}) are satisfied. Then the result immediately follows from the triangle inequality, \eq{eq: Ynlm error} in Proposition \ref{proposition: error Ynm - Yntm estimate}; and \eq{eq: Gnm error every layer} with $\ell=L$ in Proposition \ref{proposition: error Gnm - Yntm estimate}. 
\end{proof}

\section{Infinite-node Convergence of Adjoint GNDEs (Parameter Gradients)}\label{Appendix: Infinite-node Convergence of Adjoint GNDEs (Parameter Gradients)}

We assume that \ConvFilterDiff, \SigmaTwiceDiff, \dPBoundedBelow, and \GraphonLipschitz\ hold throughout this section. 
\begin{lemma}\label{lemma: estimate Delta2}
Let $\VarGraphonNDE$ and $\VarGraphonNDEHiddenState$ be the solutions of Graphon-NDE \eqref{Graphon-NDE} and adjoint Graphon-NDE \eqref{adjoint Grahpon-NDE: Y}, respectively. Suppose that \eq{upper bound of norm of L Un} holds. Let $\gamma_4\in(0,1)$. Then with probability at least $1-\gamma_4$, it holds that for all $k\in\mathbb{Z}_{K}$, $f,g\in[F]$ and $\ell\in[L]$, 
    \begin{align}
&\abs{\iprod{\opLUn^k\VarGraphonNDE\mylt_g,\hatVarGraphonNDE\mylt_f\VarGraphonNDEHiddenState\mylt_f}_{\spaceltwo}-\frac{1}{n}\iprod{\SamplingOperator(\opLUn^k\VarGraphonNDE\mylt_g),\SamplingOperator(\hatVarGraphonNDE\mylt_f\VarGraphonNDEHiddenState\mylt_f)}_{\spaceRn}}\nonumber\\
\lesssim\ & \frac{\sqrt{\log(4LF^2K/\gamma_4)}}{\sqrt{n}}\parens{1+ C_T}\parens{\frac{c_{\max}}{c_{\min}}}^{K-1}\max\left\{\CXmax\Codotmax,\CXmax\Codotlip+\CXlip\Codotmax\right\}\nonumber\\
=\ &\mathcal{O}\parens{\frac{\sqrt{\log(LF^2K/\gamma_4)}}{\sqrt{n}}}.\label{eq: estimate of Delta2}
    \end{align}
\end{lemma}
\begin{proof}
We notice that to get the desired result, it suffices to take graphon $\tW(u,v)\equiv1$, $u,v\in I$ and $X=\parens{\opLUn^k\VarGraphonNDE\mylt_g}\hatVarGraphonNDE\mylt_f\VarGraphonNDEHiddenState\mylt_f$, $k\in\mathbb{Z}_{K}$, $f,g\in[F]$, $\ell\in[L]$ in Lemma \ref{lemma: concentration ineq Yut}. It remains to estimate $C_{\max}$ and $C_{\mathrm{Lip}}$ in Lemma \ref{lemma: concentration ineq Yut}. It follows from \eq{upper bound of norm of L Un}, Propositions \ref{proposition: norm of X ell is bounded by norm of X} and \ref{proposition: X odot Y is bounded} that
$
\normb{\parens{\opLUn^k\VarGraphonNDE\mylt_g}\hatVarGraphonNDE\mylt_f\VarGraphonNDEHiddenState\mylt_f}\lesssim\parens{\frac{c_{\max}}{c_{\min}}}^k \CXmax\Codotmax,
$
where constants $C_{\VarGraphonNDE,\max}$ and $\Codotmax$ are defined in \eq{def: C X max and lip} and \eq{def: C odot max and lip}, respectively. According to Propositions \ref{proposition: solution Xfell is piecewise Lipschitz continuous about t} and \ref{proposition: temporal lipschitz for adjoint system}, with constants $\CXlip$ and $\Codotlip$ defined in \eqs \eqref{def: C X max and lip} and \eqref{def: C odot max and lip}, we obtain that for any $t_1,t_2\in[0,T]$, 
\begin{align*}
&\normb{\parens{\opLUn^k\VarGraphonNDE\mylt[\ell][t_1]_g}\hatVarGraphonNDE_f\mylt[\ell][t_1]\VarGraphonNDEHiddenState\mylt[\ell][t_1]_f-\parens{\opLUn^k\VarGraphonNDE\mylt[\ell][t_2]_g}\hatVarGraphonNDE_f\mylt[\ell][t_2]\VarGraphonNDEHiddenState\mylt[\ell][t_2]_f}\\
\leq\ &  \normb{\opLUn^k\VarGraphonNDE\mylt[\ell][t_1]_g}\normb{\hatVarGraphonNDE_f\mylt[\ell][t_1]\VarGraphonNDEHiddenState\mylt[\ell][t_1]_f-\hatVarGraphonNDE_f\mylt[\ell][t_2]\VarGraphonNDEHiddenState\mylt[\ell][t_2]_f}+\normb{\opLUn^k\parens{\VarGraphonNDE\mylt[\ell][t_1]_g-\VarGraphonNDE\mylt[\ell][t_2]_g}}\normb{\hatVarGraphonNDE_f\mylt[\ell][t_2]\VarGraphonNDEHiddenState\mylt[\ell][t_2]_f}\\
\lesssim & \parens{\frac{c_{\max}}{c_{\min}}}^k \parens{\CXmax\Codotlip+\CXlip\Codotmax}\abs{t_1-t_2}.
\end{align*}
Since $k\in\mathbb{Z}_K$, we enlarge $k$ to $K-1$, and the desired result follows.
\end{proof}

\begin{proposition}\label{proposition: Phi* - phi* leq order}
Let $\gamma_1,\gamma_2,\gamma_3,\gamma_4\in(0,1)$ with $2\gamma_1+\gamma_2+\gamma_3+\gamma_4<1$. Let $\VarGNDE$, $\VarGraphonNDE$, $\VarGNDEHiddenState$ and $\VarGraphonNDEHiddenState$ be the solutions of GNDE \eqref{GNDE}, Graphon-NDE \eqref{Graphon-NDE}, adjoint GNDE \eqref{adjoint GNDE: Y}, and adjoint Graphon-NDE \eqref{adjoint Grahpon-NDE: Y}, respectively. Suppose that initial value conditions satisfy \eqref{eq: initial values of GNDE and Graphon-NDE equal} and \eqref{eq: initial condition adjoint Hidden State}. Suppose that \eqs \eqref{upper bound of norm of L Un}, \eqref{eq: Ln - LWn 2norm}, \eqref{eq: upper bound of LUn - LP Xf}, \eqref{eq: upper bound of LUn - LP hatX*Y f} and \eqref{eq: estimate of Delta2} hold. Then for all $\ell\in[L]$ and $t\in[0,T]$, it holds that
$\normMax{\Phi_{\VarGraphonNDE\mylt}^*\parens{\hatVarGraphonNDE\mylt \odot \VarGraphonNDEHiddenState\mylt}-\frac{1}{n}\phi_{\VarGNDE\mylt}^*\parens{\hatVarGNDE\mylt\odot\VarGNDEHiddenState\mylt}}\leq \Qh(n,\gamma_1,\gamma_2,\gamma_3,\gamma_4)
$, 
where 
\begin{align}
\Qh(n,\gamma_1,\gamma_2,\gamma_3,\gamma_4)=\mathcal{O}&\parens{\frac{\sqrt{\log(LF^2 K/\gamma_4)}}{\sqrt{n}}}+\mathcal{O}\Big(\parens{\frac{1}{\sqrt{\alpha_n n}}+\frac{\sqrt{\log(n_ILFK/\gamma_2)}+\sqrt{\log(n_I/\gamma_1)}}{\sqrt{n}}}\times\nonumber\\
&\parens{\frac{1}{\sqrt{\alpha_n}}+\sqrt{\log(n_ILFK/\gamma_3)}+\sqrt{\log(n_ILFK/\gamma_2)}+\sqrt{\log(n_I/\gamma_1)}}\Big).\label{def: Qh}
\end{align}

\end{proposition}
\begin{proof}
For all $\ell\in[L]$, $f,g\in[F]$, $k\in\mathbb{Z}_{K}$ and $t\in[0,T]$, by definition of operators $\Phi_{\VarGraphonNDE\mylt}^*$ and $\phi_{\VarGNDE\mylt}^*$, we apply the triangle inequality and bound the error as
\begin{align}
&\abs{\braket{\Phi_{\VarGraphonNDE\mylt}^*\parens{\hatVarGraphonNDE\mylt \odot \VarGraphonNDEHiddenState\mylt}-\frac{1}{n}\phi_{\VarGNDE\mylt}^*(\hatVarGNDE\mylt\odot\VarGNDEHiddenState\mylt)}_{fgk}}\nonumber\\
=& \abs{\iprod{\opLP^k\VarGraphonNDE\mylt_g,\hatVarGraphonNDE\mylt_f\VarGraphonNDEHiddenState\mylt_f}_{\spaceltwo}-\frac{1}{n}\iprod{\matL^k\VarGNDE\mylt_g,\hatVarGNDE\mylt_f\odot\VarGNDEHiddenState\mylt_f}_{\spaceRn}}
\leq\sum_{j\in[5]}\Delta_j\label{in the proof: split error terms Delta 12345}
\end{align}
in which 
\begin{align*}
\Delta_1&:=\abs{\iprod{\opLP^k\VarGraphonNDE\mylt_g,\hatVarGraphonNDE\mylt_f\VarGraphonNDEHiddenState\mylt_f}_{\spaceltwo}-\iprod{\opLUn^k\VarGraphonNDE\mylt_g,\hatVarGraphonNDE\mylt_f\VarGraphonNDEHiddenState\mylt_f}_{\spaceltwo}}\\
\Delta_2&:=\abs{\iprod{\opLUn^k\VarGraphonNDE\mylt_g,\hatVarGraphonNDE\mylt_f\VarGraphonNDEHiddenState\mylt_f}_{\spaceltwo}-\frac{1}{n}\iprod{\SamplingOperator(\opLUn^k\VarGraphonNDE\mylt_g),\SamplingOperator(\hatVarGraphonNDE\mylt_f\VarGraphonNDEHiddenState\mylt_f)}_{\spaceRn}}\\
\Delta_3&:=\abs{\frac{1}{n}\iprod{\SamplingOperator(\opLUn^k\VarGraphonNDE\mylt_g),\SamplingOperator(\hatVarGraphonNDE\mylt_f\VarGraphonNDEHiddenState\mylt_f)}_{\spaceRn}-\frac{1}{n}\iprod{\matLUn^k\VarGNDE\mylt_g,\SamplingOperator(\hatVarGraphonNDE\mylt_f\VarGraphonNDEHiddenState\mylt_f)}_{\spaceRn}}\\
\Delta_4&:=\abs{\frac{1}{n}\iprod{\matLUn^k\VarGNDE\mylt_g,\SamplingOperator(\hatVarGraphonNDE\mylt_f\VarGraphonNDEHiddenState\mylt_f)}_{\spaceRn}-\frac{1}{n}\iprod{\matL^k\VarGNDE\mylt_g,\SamplingOperator(\hatVarGraphonNDE\mylt_f\VarGraphonNDEHiddenState\mylt_f)}_{\spaceRn}}\\
\Delta_5&:=\abs{\frac{1}{n}\iprod{\matL^k\VarGNDE\mylt_g,\SamplingOperator(\hatVarGraphonNDE\mylt_f\VarGraphonNDEHiddenState\mylt_f)}_{\spaceRn}-\frac{1}{n}\iprod{\matL^k\VarGNDE\mylt_g,\hatVarGNDE\mylt_f\odot\VarGNDEHiddenState\mylt_f}_{\spaceRn}}.
\end{align*}
We will bound each $\Delta_j$, for $j \in [5]$, individually. 

Note that, according to Proposition \ref{proposition: X odot Y is bounded} and with constant $\Codotmax$ defined in \eq{def: C odot max and lip} , we have 
\begin{equation}\label{in the proof: bound hat X odot Y}
\normb{\hatVarGraphonNDE\mylt_f\VarGraphonNDEHiddenState\mylt_f}\leq \Codotmax,\quad \forall\ f\in[F],\ \ell\in\mathbb{Z}_{L+1}. 
\end{equation}

We begin with estimating $\Delta_1$ by
\begin{align}\label{in the proof: Delta 1 first estimate}
\Delta_1&\leq \Codotmax\normb{(\opLP^k-\opLUn^k)(\VarGraphonNDE\mylt_g)}\stepjust{Cauchy-Schwartz inequality and \eq{in the proof: bound hat X odot Y}}\\
&\leq \Codotmax\sum_{s=0}^{k-1}\norm{\opLUn}_{\spaceb\to\spaceb}^s\normb{(\opLUn-\opLP)\opLP^{k-1-s}\VarGraphonNDE\mylt_g}\nonumber\\
&\lesssim \Codotmax\sum_{s=0}^{k-1}\parens{\frac{c_{\max}}{c_{\min}}}^{k-1}\parens{\frac{\Csys[1]\sqrt{\log(4n_I/\gamma_1)} + \Csys[2]\sqrt{\log(4n_ILFK/\gamma_2)}}{\sqrt{n}}}\max\curlbraket{\CXmax,\CXlip}\stepjust{\eq{eq: upper bound of LUn - LP Xf} and \eq{upper bound of norm of L Un}}\nonumber\\ 
&=\mathcal{O}\parens{\frac{\sqrt{\log(n_ILFK/\gamma_2)}+\sqrt{\log(n_I/\gamma_1)}}{\sqrt{n}}}.\nonumber
\end{align}

For $\Delta_2$, we notice that assumption \eqref{eq: estimate of Delta2} gives $
\Delta_2\lesssim\mathcal{O}\parens{\frac{\sqrt{\log(LF^2K/\gamma_4)}}{\sqrt{n}}}$.

We proceed to estimate $\Delta_3$. It follows from \eq{eq: delta Phi = phi delta} that $\SamplingOperator(\opLUn^k\VarGraphonNDE\mylt_g)=\matLUn^k \SamplingOperator(\VarGraphonNDE\mylt_g)$ and therefore, 
\begin{align}
    \Delta_3&=\frac{1}{n}\abs{\iprod{\matLUn^k\parens{\SamplingOperator(\VarGraphonNDE\mylt_g)-\VarGNDE\mylt_g}, \SamplingOperator(\hatVarGraphonNDE\mylt_f\VarGraphonNDEHiddenState\mylt_f)}_{\spaceRn}}\nonumber\\
    &\leq \frac{1}{\sqrt{n}}\normb{\hatVarGraphonNDE\mylt_f\VarGraphonNDEHiddenState\mylt_f}\normTwo{\SamplingOperator(\VarGraphonNDE\mylt_g)-\VarGNDE\mylt_g}\stepjust{Cauchy-Schwartz inequality, $\normTwo{\matLUn}\leq 1$ and \eq{norm of operator delta Un one d}}\\
    &\leq \Codotmax\frac{1}{\sqrt{n}}\normTwo{\SamplingOperator(\VarGraphonNDE\mylt_g)-\VarGNDE\mylt_g}\leq \Codotmax\ \epsilon\mylt\stepjust{\eq{in the proof: bound hat X odot Y} and definition \eqref{def: epsilon ell}}\\
    &\leq \Codotmax\ \Ce \QX(n,\gamma_1,\gamma_2)\stepjust{\eqref{eq: sup MSE bounded by sqrt(n)} in Proposition \ref{prop: Graphon-MSE every layer}}\\
    &=\mathcal{O}\parens{\frac{1}{\sqrt{\alpha_n n}}+\frac{\sqrt{\log(n_ILFK/\gamma_2)}+\sqrt{\log(n_I/\gamma_1)}}{\sqrt{n}}}.\nonumber
\end{align}

We then estimate $\Delta_4$ with
\begin{align}
    \Delta_4&\leq \frac{1}{\sqrt{n}}\normTwo{\matLUn^k-\matL^k}\normTwo{\VarGNDE\mylt_g}\normb{\hatVarGraphonNDE\mylt_f\VarGraphonNDEHiddenState\mylt_f}\stepjust{Cauchy-Schwartz inequality and \eq{norm of operator delta Un one d}}\\
    &\leq k\normTwo{\matLUn-\matL}\Codotmax\normTwo{\VarGNDE\mylt_g}/\sqrt{n}\stepjust{$\normTwo{\matLUn}\leq1$,  $\normTwo{\matL}\leq1$ and {\eq{in the proof: bound hat X odot Y}}}\\
    &\leq k\normTwo{\matLUn-\matL}\Codotmax\CmXGNDE\stepjust{\eq{eq: estimate of mXnlt F norm} in Proposition \ref{proposition: Xns Xn0'' bounded above by sqrt n}}\\
    &\lesssim k\frac{c_{\max}}{c_{\min}^2}\frac{1}{\sqrt{\alpha_n n}}\Codotmax\CmXGNDE=\mathcal{O}\parens{\frac{1}{\sqrt{\alpha_n n}}}\stepjust{\eq{eq: Ln - LWn 2norm}}
\end{align}

We finally estimate $\Delta_5$. Note that 
\begin{align*}
    \Delta_5&\leq \frac{1}{n}\normTwo{\VarGNDE\mylt_g}\normTwo{\SamplingOperator(\hatVarGraphonNDE\mylt_f\VarGraphonNDEHiddenState\mylt_f)-\hatVarGNDE\mylt_f\odot\VarGNDEHiddenState\mylt_f}\stepjust{Cauchy-Schwartz inequality and $\normTwo{\matL}\leq 1$}\\
    &\leq \frac{\CmXGNDE}{\sqrt{n}}\normTwo{\SamplingOperator(\hatVarGraphonNDE\mylt_f\VarGraphonNDEHiddenState\mylt_f)-\hatVarGNDE\mylt_f\odot\VarGNDEHiddenState\mylt_f}\stepjust{\eq{eq: estimate of mXnlt F norm} in Proposition \ref{proposition: Xns Xn0'' bounded above by sqrt n}}
\end{align*}
in which
\begin{align*}
&\frac{1}{\sqrt{n}}\normTwo{\SamplingOperator(\hatVarGraphonNDE\mylt_f\VarGraphonNDEHiddenState\mylt_f)-\hatVarGNDE\mylt_f\odot\VarGNDEHiddenState\mylt_f}\nonumber\\
\leq\ & \normMax{\SamplingOperator(\hatVarGraphonNDE\mylt_f)}\frac{\normTwo{\SamplingOperator(\VarGraphonNDEHiddenState\mylt_f)-\VarGNDEHiddenState_{f}\mylt}}{\sqrt{n}} + \frac{\normTwo{\SamplingOperator(\hatVarGraphonNDE\mylt_f)-\hatVarGNDE_{f}\mylt}}{\sqrt{n}}\normMax{\VarGNDEHiddenState_{f}\mylt}\stepjust{triangle inequality and \eq{eq: upper bound of matrix Z odot V}}\\
\leq\ & \normMax{\SamplingOperator(\hatVarGraphonNDE\mylt_f)}\eta\mylt + L_{\sigma'}\widetilde{\epsilon}\mylt\normMax{\VarGNDEHiddenState_{f}\mylt}\stepjust{definitions \eqref{def: widetilde epsilon ell}, \eqref{def: eta ell}, \eqref{def: widehat mX ell}, \eqref{def: widehat tX ell} and \SigmaDerivativeLipschitz}\\
\leq\ & \eta\mylt+L_{\sigma'}\widetilde{\epsilon}\mylt\CmYGNDEmax\stepjust{relation \eqref{eq: hat X leq 1} and \eq{eq: sup mYnlt max leq P + CYmax} in Proposition \ref{proposition: Yns Yn0'' bounded above by sqrt n}}\\
\leq\ &{\Ce}\QY(n,\gamma_1,\gamma_2,\gamma_3)+L_{\sigma'}{\tildeCe} \QX(n,\gamma_1,\gamma_2)\CmYGNDEmax\stepjust{Proposition \ref{prop: Adjoint Graphon-MSE every layer} and \eq{eq: sup tilde epsilon st leq QX} in Proposition \ref{prop: Graphon-MSE every layer}}
\end{align*}
Therefore,
\begin{align}
    \Delta_5\leq &\ \CmXGNDE\left[{\Ce}\QY(n,\gamma_1,\gamma_2,\gamma_3)+ L_{\sigma'}{\tildeCe} \QX(n,\gamma_1,\gamma_2)\CmYGNDEmax\right]\nonumber\\
    \leq\ &\mathcal{O}\Big(\parens{\frac{1}{\sqrt{\alpha_n n}}+\frac{\sqrt{\log(n_ILFK/\gamma_2)}+\sqrt{\log(n_I/\gamma_1)}}{\sqrt{n}}}\times\nonumber\\
    &\qquad \parens{\frac{1}{\sqrt{\alpha_n}}+\sqrt{\log(n_ILFK/\gamma_3)}+\sqrt{\log(n_ILFK/\gamma_2)}+\sqrt{\log(n_I/\gamma_1)}}\Big).\nonumber
\end{align}
We obtain the desired result by substituting estimates of $\Delta_j$, $j\in[5]$, into \eqref{in the proof: split error terms Delta 12345}.  
\end{proof}

\begin{proof}[\textbf{Proof of Theorem \ref{theorem: z - zn leq TQh}}]
Note that by \eq{adjoint operator: GNDE to filters}, $\VarGNDEParameter\mylt$ in form of \eqref{adjoint GNDE: z} can be rewritten as $\VarGNDEParameter\mylt = \phi_{\VarGNDE\mylt}^*\parens{\hatVarGNDE\mylt\odot\VarGNDEHiddenState\mylt}$. And according to \eq{adjoint operator: Graphon-NDE to filters}, $\VarGraphonNDEParameter\mylt$ in form of \eqref{adjoint Graphon-NDE: H} can be rewritten as $\VarGraphonNDEParameter\mylt = \Phi_{\VarGraphonNDE\mylt}^*\parens{\hatVarGraphonNDE\mylt\odot\VarGraphonNDEHiddenState\mylt}$. By Lemmas \ref{lemma: LUn-LP leq sqrt(n)}, \ref{lemma: Ln - LWn 2norm}, \ref{lemma: upper bound of LUn - LP X for all}, \ref{lemma: upper bound of LUn - LP odot for all}, \ref{lemma: estimate Delta2}, with probability at least $1-2\gamma_1-\gamma_2-\gamma_3-\gamma_4$,  that \eqs \eqref{eq: sup dP dUn}, \eqref{eq: Ln - LWn 2norm}, \eqref{eq: upper bound of LUn - LP Xf}, \eqref{eq: upper bound of LUn - LP hatX*Y f} and \eqref{eq: estimate of Delta2} hold. Then the desired result immediately follows from Proposition \ref{proposition: Phi* - phi* leq order}.
\end{proof}

\section{Infinite-node Convergence of Discretized Adjoint GNDEs (Parameter Gradients)}\label{Appendix: Infinite-node Convergence of Discretized Adjoint GNDEs (Parameter Gradients)}
We assume that \ConvFilterDiff, \SigmaTwiceDiff, \dPBoundedBelow, and \GraphonLipschitz\ hold throughout this section.

\begin{proposition}\label{proposition: ddt D Yn upper bound}
Let $\VarGNDEHiddenState(t)$ be the solution of adjoint GNDE \eqref{adjoint GNDE: Y}. Suppose that \eqs \eqref{eq: MSE layer L leq MSE layer 0}, \eqref{eq: eta 0 leq eta L + widetilde mathcal Q}, and initial value conditions \eqref{eq: initial values of GNDE and Graphon-NDE equal}, \eqref{eq: initial condition adjoint Hidden State} hold. Suppose that $n$ is large enough such that \eqref{eq: n is big enough 2} and \eqref{eq: n is big enough 3} hold. Then for all $\ell\in[L]$ and $t\in[0,T]$, it holds that 
    \begin{align*}
\normMax{\dt\Dh\mylt(\VarGNDEHiddenState(t))}\leq n\parens{{L_\sigma}\CmXDeriGNDE\CmYGNDE  +\CmXGNDE\parens{\CmXHatDeriGNDE\CmYGNDEmax + {L_\sigma}\CmYDeriGNDE}}. 
    \end{align*}

\end{proposition}

\begin{proof}
Recall that $\Dh\mylt(\VarGNDEHiddenState(t))$ on the right side of \eqref{adjoint GNDE: z} can be rewritten via \eqref{adjoint operator: GNDE to filters} as $
\Dh\mylt(\VarGNDEHiddenState(t))=\phi_{\VarGNDE\mylt}^*\parens{\hatVarGNDE\mylt\odot\VarGNDEHiddenState\mylt}$. Then by Lemma \ref{lemma: d/dt phiXlt(Zt) estimate} with $\mP\myt:=\hatVarGNDE\mylt$, $\mQ\myt:=\VarGNDEHiddenState\mylt$, we obtain that for all $\ell\in[L]$, $f,g\in[F]$, $k\in\mathbb{Z}_{K}$ and $t\in[0,T]$, 
    \begin{align*}
&\abs{\dt\braket{\Dh\mylt(\VarGNDEHiddenState(t))}_{fgk}}=\abs{\dt\braket{\phi_{\VarGNDE\mylt}^*\parens{\hatVarGNDE\mylt\odot\VarGNDEHiddenState\mylt}}_{fgk}}\\
\leq\ &\normF{\dt\VarGNDE\mylt}\normMax{\hatVarGNDE\mylt}\normF{\VarGNDEHiddenState\mylt} +\normF{\VarGNDE\mylt}\parens{\normF{\dt\hatVarGNDE\mylt}\normMax{\VarGNDEHiddenState\mylt} + \normMax{\hatVarGNDE\mylt}\normF{\dt\VarGNDEHiddenState\mylt}}.
\end{align*}
Note that all assumptions in Proposition \ref{proposition: Xns Xn0'' bounded above by sqrt n} and Proposition \ref{proposition: Yns Yn0'' bounded above by sqrt n} are assumed in the current proposition. Therefore, by \eq{eq: hat X leq 1}, $\normMax{\hatVarGNDE\mylt}\leq {L_\sigma}$; by \eq{eq: estimate of mXnlt F norm}, $\normF{\VarGNDE\mylt}\leq\CmXGNDE\sqrt{n}$; by \eq{eq: estimate of deri mXnlt F norm}, $\normF{\dt\VarGNDE\mylt}\leq \CmXDeriGNDE\sqrt{n}$; by \eq{eq: estimate of deri mXnlt hat F norm}, $\normF{\dt\hatVarGNDE\mylt}\leq\CmXHatDeriGNDE\sqrt{n}$; by \eq{eq: sup mYnlt max leq P + CYmax}, $\normMax{\VarGNDEHiddenState\mylt}\leq \CmYGNDEmax$; by \eq{eq: estimate of mYnlt F norm}, $\normF{\VarGNDEHiddenState\mylt}\leq \CmYGNDE\sqrt{n}$; by \eq{eq: dYnLt/dt F norm bound}, $\normF{\dt\VarGNDEHiddenState\mylt}\leq \CmYDeriGNDE\sqrt{n}$. We obtain the desired result by applying the above bounds. 
    
\end{proof}

\begin{proposition}\label{proposition: diff of D Yn [] ()}
Let $\VarGNDEHiddenState(t)$ be the solution of adjoint GNDE \eqref{adjoint GNDE: Y}, and $\VarGNDEHiddenState^{[m]}$ be generated from \eqref{eq: discretized Yn}. Suppose that \eqs \eqref{eq: MSE layer L leq MSE layer 0}, \eqref{eq: eta 0 leq eta L + widetilde mathcal Q}, and initial value conditions \eqref{eq: initial values of GNDE and Graphon-NDE equal}, \eqref{eq: initial condition adjoint Hidden State}, \eqref{eq: initial value condition on GNDE and discretized GNDE}, \eqref{eq: initial condition GNDE discretized GNDE Adjoint A} hold. Suppose that $n$ is large enough such that \eqref{eq: n is big enough 2} and \eqref{eq: n is big enough 3} hold; $M$ is large enough such that \eqref{eq: M is large enough} holds. Then, for $\ell\in[L]$ and $m\in[M]$, 
\begin{align}
&\frac{1}{n}\normMax{\Dh^{[\ell,m]}(\VarGNDEHiddenState^{[m]})-\Dh\mylt[\ell][t_m](\VarGNDEHiddenState(t_m))}\nonumber\\
&\leq \frac{{L_\sigma}\CXdiff\CmYDistGNDE+\CmXGNDE\parens{{L_\sigma}\CYdiff+\CXdiffHat\CmYGNDEmax}}{M}.  \label{eq: Max norm diff of Dnhlm Ynm - Dnhltm Yntm}
\end{align}    

\end{proposition}
\begin{proof}
It follows from Lemma \ref{lemma: phi* X1 hatX1 A1 - X2 hatX2 A2} that 
    \begin{align*}
&\normMax{\Dh^{[\ell,m]}(\VarGNDEHiddenState^{[m]})-\Dh\mylt[\ell][t_m](\VarGNDEHiddenState(t_m))}=\normMax{\phi_{\VarGNDE^{[\ell,m]}}^*\parens{\hatVarGNDE^{[\ell,m]}\odot\VarGNDEHiddenState^{[\ell,m]}}-\phi_{\VarGNDE\mylt[\ell][t_m]}^*\parens{\hatVarGNDE\mylt[\ell][t_m]\odot\VarGNDEHiddenState\mylt[\ell][t_m]}}\\
\leq&\normF{\VarGNDE^{[\ell,m]}-\VarGNDE\mylt[\ell][t_m]}\normF{\VarGNDEHiddenState^{[\ell,m]}}\normMax{\hatVarGNDE^{[\ell,m]}} \\
&+ \normF{\VarGNDE\mylt[\ell][t_m]}\parens{\normMax{\hatVarGNDE^{[\ell,m]}}\normF{\VarGNDEHiddenState^{[\ell,m]}-\VarGNDEHiddenState\mylt[\ell][t_m]}+\normF{\hatVarGNDE^{[\ell,m]}-\hatVarGNDE\mylt[\ell][t_m]}\normMax{\VarGNDEHiddenState\mylt[\ell][t_m]}}.
    \end{align*}
Note that all assumptions in Propositions \ref{proposition: Xns Xn0'' bounded above by sqrt n}, \ref{proposition: discretized GNDE error leq 1/M}, \ref{proposition: Yns Yn0'' bounded above by sqrt n}, \ref{proposition: error Ynm - Yntm estimate}, are assumed in the current proposition. Then by \eq{eq: hat X leq 1}, $\normMax{\hatVarGNDE^{[\ell,m]}}\leq {L_\sigma}$; by \eq{eq: mXlnm leq 1/M} in Proposition \ref{proposition: discretized GNDE error leq 1/M}, $\normF{\VarGNDE^{[\ell,m]}-\VarGNDE\mylt[\ell][t_m]}\leq \CXdiff\sqrt{n}/M$; by \eq{eq: normF of Yn[lm]} in Proposition \ref{proposition: error Ynm - Yntm estimate}, $\normF{\VarGNDEHiddenState^{[\ell,m]}}\leq \CmYDistGNDE\sqrt{n}$; by \eq{eq: estimate of mXnlt F norm} in Proposition \ref{proposition: Xns Xn0'' bounded above by sqrt n}, $\normF{\VarGNDE\mylt[\ell][t_m]}\leq \CmXGNDE\sqrt{n}$; by \eq{eq: Ynlm error} in Proposition \ref{proposition: error Ynm - Yntm estimate}, $\normF{\VarGNDEHiddenState^{[\ell,m]}-\VarGNDEHiddenState\mylt[\ell][t_m]}\leq \CYdiff\sqrt{n}/M$; by \eq{eq: hat mXlnm leq 1/M} in Proposition \ref{proposition: discretized GNDE error leq 1/M}, $\normF{\hatVarGNDE^{[\ell,m]}-\hatVarGNDE\mylt[\ell][t_m]}\leq \CXdiffHat\sqrt{n}/M$; by \eq{eq: sup mYnlt max leq P + CYmax} in Proposition \ref{proposition: Yns Yn0'' bounded above by sqrt n}, $\normMax{\VarGNDEHiddenState\mylt[\ell][t_m]}\leq \CmYGNDEmax$. We obtain \eq{eq: Max norm diff of Dnhlm Ynm - Dnhltm Yntm} by using the above estimates. 
\end{proof}

\begin{proof}[\textbf{Proof of Theorem \ref{theorem: discretized adjoint gnde to adjoint gnde parameters}}]
Note that by \eqref{adjoint GNDE: z} and \eqref{grad of h optimize-then-dicretize}, we have $
\VarGNDEParameter\mylm=\Dh^{[\ell,m]}(\VarGNDEHiddenState^{[m]})$ and $\VarGNDEParameter\mylt[\ell][t_m]=\Dh\mylt[\ell][t_m](\VarGNDEHiddenState(t_m))$. Similar to the proof of Theorem \ref{theorem: discretized adjoint gnde to adjoint gnde Hidden state}, with probability at least $1-2\gamma_1-\gamma_2-\gamma_3$, the assumptions in Proposition \ref{proposition: diff of D Yn [] ()} are satisfied. Therefore, the desired result immediately follows from Proposition \ref{proposition: diff of D Yn [] ()}. 

\end{proof}

\section{DTO versus OTD for Gradients of Parameters}\label{Appendix: DTO versus OTD for Gradients of Parameters}

\begin{proposition}\label{proposition: parameter gradients DTO and OTD}
If all assumptions of Proposition \ref{proposition: error Ynm - Yntm estimate} are satisfied, then 
\begin{align*}%
&\frac{1}{n}\normMax{\Dh^{[\ell,m]}(\VarDTOGradient^{[m+1]})-\Dh^{[\ell,m]}(\VarGNDEHiddenState^{[m]})}\\
&\leq  \frac{\CmXDistGNDE{L_\sigma}\parens{\CGdiff + \CmYDeriGNDE + \CYdiff}}{M}.
\end{align*}  
\end{proposition}
\begin{proof}
Note that 
    \begin{align*}
&\normMax{\Dh^{[\ell,m]}(\VarDTOGradient^{[m+1]})-\Dh^{[\ell,m]}(\VarGNDEHiddenState^{[m]})}=\normMax{\phi_{\VarGNDE^{[\ell,m]}}^*\parens{\hatVarGNDE^{[\ell,m]}\odot\VarDTOGradient^{[\ell,m+1]}}-\phi_{\VarGNDE\mylm}^*\parens{\hatVarGNDE\mylm\odot\VarGNDEHiddenState\mylm}}\\
=&\normMax{\phi_{\VarGNDE^{[\ell,m]}}^*\parens{\hatVarGNDE^{[\ell,m]}\odot\parens{\VarDTOGradient^{[\ell,m+1]} - \VarGNDEHiddenState\mylm}}}\stepjust{linearity of $\phi^*$}\\
=&\normF{\VarGNDE^{[\ell,m]}}\normMax{\hatVarGNDE^{[\ell,m]}}\normF{\VarDTOGradient^{[\ell,m+1]} - \VarGNDEHiddenState\mylm}\stepjust{Lemma \ref{lemma: phiX*Z leq X Z F norm} and \eq{eq: upper bound of matrix Z odot V}}\\
\leq&\normF{\VarGNDE^{[\ell,m]}}\normMax{\hatVarGNDE^{[\ell,m]}}\parens{\normF{\VarDTOGradient^{[\ell,m+1]} - \VarGNDEHiddenState\mylt[\ell][t_{m+1}]}  +   \normF{\VarGNDEHiddenState\mylt[\ell][t_{m+1}] - \VarGNDEHiddenState\mylt[\ell][t_{m}]}    +   \normF{\VarGNDEHiddenState\mylt[\ell][t_{m}] - \VarGNDEHiddenState\mylm}}\stepjust{triangle inequality}
    \end{align*}   
Note that assumptions of Proposition \ref{proposition: error Ynm - Yntm estimate} guarantee results in Propositions \ref{proposition: discretized GNDE error leq 1/M}, \ref{proposition: error Gnm - Yntm estimate} and \ref{proposition: Yns Yn0'' bounded above by sqrt n}. Then by \SigmaDiff, $\normMax{\hatVarGNDE^{[\ell,m]}}\leq {L_\sigma}$; by \eq{eq: mXln[m] leq 1/M} in Proposition \ref{proposition: discretized GNDE error leq 1/M}, $\normF{\VarGNDE^{[\ell,m]}}\leq \CmXDistGNDE\sqrt{n}$; by \eq{eq: Gnm error every layer} in Proposition \ref{proposition: error Gnm - Yntm estimate}, $\normF{\VarDTOGradient^{[\ell,m+1]} - \VarGNDEHiddenState\mylt[\ell][t_{m+1}]}\leq \CGdiff\sqrt{n}/M$; by \eq{eq: dYnLt/dt F norm bound} in Proposition \ref{proposition: Yns Yn0'' bounded above by sqrt n}, $\normF{\VarGNDEHiddenState\mylt[\ell][t_{m+1}] - \VarGNDEHiddenState\mylt[\ell][t_{m}]}\leq \sup_{t\in[0,T]}\normF{\dt\VarGNDEHiddenState\mylt}\abs{t_{m+1}-t_m} \leq  \frac{\CmYDeriGNDE}{M}\sqrt{n}$; by \eq{eq: Ynlm error} in Proposition \ref{proposition: error Ynm - Yntm estimate}, $\normF{\VarGNDEHiddenState\mylt[\ell][t_{m}] - \VarGNDEHiddenState\mylm}\leq \frac{\CYdiff}{M}\sqrt{n}$. 
We obtain the desired the desired result by using above estimates. 

\end{proof}

\begin{proof}[\textbf{Proof of Theorem \ref{theorem: parameter gradients DTO and OTD}}]
    We notice from the proof of Theorem \ref{theorem: discretized adjoint gnde to adjoint gnde Hidden state} that with probability at least $1-2\gamma_1-\gamma_2-\gamma_3$, the assumptions in Proposition \ref{proposition: error Ynm - Yntm estimate} are satisfied. Then the result immediately follows from Proposition \ref{proposition: parameter gradients DTO and OTD} and noting that $\kappa=T/M$. 
\end{proof}

\section{Graph Convolution Operators}\label{appendix: graph convolutional operators}
Recall that graph convolution operator $\phi$ is defined in \eqref{def: operator phi}. For a fixed $\vh\in\spaceRFFK$, we define an operator $\phi_{\vh}:\spaceRnF\to \spaceRnF$, associated to $\vh$, by 
\begin{equation}\label{def: graph convolutional operator phi h}
    \phi_{\vh}(\VarGNDE):=\phi(\vh,\VarGNDE),\quad \VarGNDE\in\spaceRnF.
\end{equation} 
For a fixed $\VarGNDE\in\spaceRnF$, we define an operator $\phi_{\VarGNDE}:\spaceRFFK\to \spaceRnF$, associated to $\VarGNDE$, by
\begin{equation}\label{def: graph convolutional operator phi X}
    \phi_{\VarGNDE}(\vh):=\phi(\vh,\VarGNDE),\quad \vh\in\spaceRFFK.
\end{equation}
It is straightforward to verify that both operators $\phi_{\vh}$ and $\phi_{\VarGNDE}$ are linear. 
\paragraph{Adjoint operators.} Noting that the symmetric normalized adjacency matrix $\matL$ is symmetric, the adjoint operator of $\phi_{\vh}$, denoted by $\phi_{\vh}^*:\spaceRnF\to \spaceRnF$, is given by
\begin{equation*}
\phi_{\vh}^*(\mZ)=\left[\sum_{f=1}^F\sum_{k=0}^{K-1}\vh_{fgk}\matL^k\mZ_f:g\in[F]\right],\quad \mZ\in\spaceRnF,
\end{equation*}
and the adjoint operator of $\phi_{\VarGNDE}$, denoted by $\phi_{\VarGNDE}^*:\spaceRnF\to \spaceRFFK$, is given by
\begin{equation*}
\phi_{\VarGNDE}^*(\mZ)=\left[\iprod{\matL^k \VarGNDE_g,\mZ_f }_{\spaceRn}:f,g\in[F],k\in\mathbb{Z}_{K}\right],\quad \mZ\in\spaceRnF.
\end{equation*}

\paragraph{Adjoint operators of Gateaux derivatives of single-layer GNNs.}
Let $\tildeVarGNDE=\phi(\vh,\VarGNDE)$ and $\mY=\sigma(\tildeVarGNDE)$. Let $\hatVarGNDE:=\sigma'(\tildeVarGNDE)$. By Chain rule, the Gateaux derivatives of $\mY$ with respect to $\VarGNDE$ and $\vh$ are given respectively by 
\begin{align*}
    \frac{\delta \mY}{\delta \VarGNDE}(\mZ)=\hatVarGNDE\odot \phi_{\vh}(\mZ),\ \ \mZ\in\spaceRnF,\qquad\frac{\delta \mY}{\delta \vh}(\vz)=\hatVarGNDE\odot \phi_{\VarGNDE}(\vz),\ \ \vz\in\spaceRFFK.
\end{align*}
Their adjoint operators are given by 
\begin{align}
    \parens{\frac{\delta\mY}{\delta\VarGNDE}}^*(\mV)=\phi_{\vh}^*\parens{\hatVarGNDE\odot \mV},\quad \parens{\frac{\delta\mY}{\delta\vh}}^*(\mV)=\phi_{\VarGNDE}^*\parens{\hatVarGNDE\odot \mV},\quad \mV\in\spaceRnF.\label{eq: adjoint delta mY / delta mX, delta h}
\end{align}

\paragraph{Adjoint operators of Gateaux Derivative of multi-layer GNNs.} Now we consider the multi-layer GNNs defined in \eqref{def: GNN multi layer}. Let 
\begin{equation}\label{def: widehat mX ell}
\hatVarGNDE\myl:=\sigma'(\tildeVarGNDE\myl),\quad \ell\in[L]. 
\end{equation}
It is straightforward to compute with Chain rule and \eq{eq: adjoint delta mY / delta mX, delta h} that for $\mV\in\spaceRnF$,
\begin{align*}
\parens{\frac{\delta \GNN}{\delta \VarGNDE^{(0)}}}^*(\mV)&=\phi_{\vh^{(1)}}^*\parens{\hatVarGNDE^{(1)}\odot\cdots\phi_{\vh^{(L-1)}}^*\parens{\hatVarGNDE^{(L-1)}\odot\parens{\phi_{\vh^{(L)}}^*\parens{\hatVarGNDE^{(L)}\odot \mV}}}},\\
\parens{\frac{\delta \GNN}{\delta\vh\myl}}^*(\mV)&=\phi_{\VarGNDE\myl}^*\parens{\hatVarGNDE\myl\odot\cdots\phi_{\vh^{(L-1)}}^*\parens{\hatVarGNDE^{(L-1)}\odot\parens{\phi_{\vh^{(L)}}^*\parens{\hatVarGNDE^{(L)}\odot \mV}}}}.
\end{align*}
We let
\begin{align}
\VarGNDEHiddenState^{(L)}:=\mV,\quad\VarGNDEHiddenState\myl&:=\phi_{\vh\myl[\ell+1]}^*\parens{\hatVarGNDE\myl[\ell+1]\odot\VarGNDEHiddenState\myl[\ell+1]},\quad \ell=L-1,L-2,\cdots,0,\label{updating rule for mY ell}
\end{align}
which enables us to rewrite 
\begin{equation}\label{adjoint operator: GNDE to filters}
   \parens{\frac{\delta \GNN}{\delta \VarGNDE^{(0)}}}^*(\VarGNDEHiddenState^{(L)})=\VarGNDEHiddenState^{(0)},\quad \parens{\frac{\delta \GNN}{\delta\vh\myl}}^*(\VarGNDEHiddenState^{(L)})=\phi_{\VarGNDE\myl}^*\parens{\hatVarGNDE\myl\odot\VarGNDEHiddenState\myl}.
\end{equation}

\paragraph{Norm estimates.} In the following, we present several useful norm estimates regarding to graph convolution operators.
\begin{lemma}\label{lemma: estimate of phi(matrix)}
Suppose that \ConvFilterLipschitz holds. For any matrix $\mZ\in\spaceRnF$, it holds that 
\begin{align*}
&\normF{\phi_{\vh\mylt}(\mZ)}\leq FK\Ch\|\mZ\|_{\mathrm{F}},   \\
\Or{&\normF{\phi_{\vh\mylt}^*(\mZ)}}
\end{align*}
where
\begin{equation}\label{def: constant Ch}
    \Ch:=\supT\max_{f,g\in[F], \ell\in[L], k\in\mathbb{Z}_K}\left|\hltfgk \right|.
\end{equation}
\end{lemma}

\begin{proof}
We recall that if \ConvFilterLipschitz holds, then the parameters $\hltfgk$, $f,g\in[F]$, $\ell\in[L]$, $k\in\mathbb{Z}_K$, appearing in \eq{parameterized H(t) filters}, are continuous functions of $t$. Therefore, the constant $\Ch$ is well-defined. Note that $\|\matL\|_2\leq 1$. Therefore, for any $\mZ\in\spaceRnF$,  
\begin{align*}
\normF{\phi_{\vh\mylt}(\mZ)}=\sqrt{\sum_{f=1}^F\norm{\sum_{g=1}^F\sum_{k=0}^{K-1}\hltfgk\matL^k \mZ_g}_{2}^2}\leq \sqrt{FK}\Ch\sqrt{\sum_{f=1}^F\sum_{g=1}^F\sum_{k=0}^{K-1}\norm{\mZ_g}_{2}^2}=FK\Ch\normF{\mZ}.
\end{align*}
The proof for the case of $\phi_{\vh\mylt}^*$ is similar. 
\end{proof}

\begin{proposition}\label{proposition: Nn is Lipschitz continuous}
    Suppose that \ConvFilterLipschitz\ and \SigmaLipschitz\ hold. For any $t\in[0,T]$, the GNN function $\GNN(\cdot;\matL,\tH(t))$ in GNDE \eqref{GNDE} is $({L_\sigma}FK\Ch)^L$-Lipschitz continuous.  
\end{proposition}
\begin{proof}
    Let $\VarGNDE,\mZ\in\mathbb{R}^{n\times F}$ and $\VarGNDE^{(L)}:=\GNN(\VarGNDE;\matL,\tH(t))$, $\mZ^{(L)}:=\GNN(\mZ;\matL,\tH(t))$. By \SigmaLipschitz, linearity of $\phi_{\vh}$, and Lemma \ref{lemma: estimate of phi(matrix)}, for each $\ell\in\mathbb{Z}_L$, we have
    \begin{align*}
        &\normF{\mZ\myl[\ell+1]-\VarGNDE\myl[\ell+1]}=\normF{\sigma\parens{\phi_{\vh\mylt}(\mZ\myl)}-\sigma\parens{\phi_{\vh\mylt}(\VarGNDE\myl)}}\\
        &\leq {L_\sigma}\normF{\phi_{\vh\mylt}(\mZ\myl-\VarGNDE\myl)}\leq {L_\sigma}FK\Ch\normF{\mZ\myl-\VarGNDE\myl}.
    \end{align*}
    A recursion gives $\normF{\mZ^{(L)}-\VarGNDE^{(L)}}\leq ({L_\sigma}FK\Ch)^L\normF{\mZ-\VarGNDE}$, which completes the proof.
\end{proof}

\begin{lemma}\label{lemma: d/dt phihlt(Zt) estimate}
    Suppose that \ConvFilterDiff\ holds. Let $\mZ\myt\in\spaceRnF$ be a matrix-valued function, which is (entry-wisely) differentiable about $t$. Then for any $\ell\in[L]$, it holds that
    \begin{align*}
        &\normF{\dt\phi_{\vh\mylt}(\mZ\myt)}\leq FK\Liph\normF{\mZ\myt}+FK\Ch\normF{\dt\mZ\myt}.\\
        \Or{&\normF{\dt\phi_{\vh\mylt}^*(\mZ\myt)}}
    \end{align*}
\end{lemma}
\begin{proof}
    By definition of operator $\phi_{\vh\mylt}$, we compute
\begin{align*}
&\dt\phi_{\vh\mylt}(\mZ\myt)=\dt(\braket{\sum_{g=1}^F\sum_{k=0}^{K-1}\hltfgk\matL^k\mZ_g\myt:f\in[F]})\\
&=\braket{\sum_{g=1}^F\sum_{k=0}^{K-1}\dt\hltfgk\matL^k\mZ_g\myt+\hltfgk\matL^k\dt\mZ_g\myt:f\in[F]}=\phi_{\dt\vh\mylt}\parens{\mZ\myt} + \phi_{\vh\mylt}\parens{\dt\mZ\myt}.
\end{align*}
Therefore,
\begin{align*}
\normF{\dt\phi_{\vh\mylt}(\mZ\myt)}&\leq \normF{\phi_{\dt\vh\mylt}\parens{\mZ\myt}} + \normF{\phi_{\vh\mylt}\parens{\dt\mZ\myt}}\stepjust{triangle inequality}\\
&\leq FK\Liph\normF{\mZ\myt}+FK\Ch\normF{\dt\mZ\myt}\stepjust{Lemma \ref{lemma: estimate of phi(matrix)} and \eq{eq: upper bound for filter derivative}}
\end{align*}
It is similar to prove the case of $\phi_{\vh\mylt}^*$. 
\end{proof}

\begin{lemma}\label{lemma: phiX*Z leq X Z F norm}
    Let $\VarGNDE,\mZ\in\spaceRnF$. Then $\normMax{\phi_{\VarGNDE}^*(\mZ)}\leq \normF{\VarGNDE}\normF{\mZ}$.
\end{lemma}

\begin{proof}
It follows from the definition of $\phi_{\VarGNDE}^*$ and the fact of $\normTwo{\matL}\leq 1$ that 
$
\abs{\braket{\phi_{\VarGNDE}^*(\mZ)}_{fgk}}=\abs{\iprod{\matL^k\VarGNDE_g,\mZ_f}_{\spaceRn}}\leq \normTwo{\VarGNDE_g}\normTwo{\mZ_f}\leq \normF{\VarGNDE}\normF{\mZ}$, for all $f,g\in[F],k\in\mathbb{Z}_K$. 
\end{proof}

\begin{lemma}\label{lemma: d/dt phiXlt(Zt) estimate}
Let $\VarGNDE\myt,\mP\myt,\mQ\myt\in\spaceRnF$ be matrix-valued functions, which are (entry-wisely) differentiable about $t$. Let $\mZ\myt:=\mP\myt\odot\mQ\myt$. Then 
$
\normMax{\dt\braket{\phi_{\VarGNDE\myt}^*(\mZ\myt)}}\leq \normF{\dt\VarGNDE\myt}\normMax{\mP\myt}\normF{\mQ\myt}$ $+\normF{\VarGNDE\myt}\parens{\normF{\dt\mP\myt}\normMax{\mQ\myt}+\normMax{\mP\myt}\normF{\dt\mQ\myt}}.
$
\end{lemma}
\begin{proof}
According to the definition of $\phi_{\VarGNDE}^*$ and Chain rule, for any $f,g\in[F]$ and $k\in\mathbb{Z}_K$, we have
$
\dt\braket{\phi_{\VarGNDE\myt}^*(\mZ\myt)}_{fgk}=\dt\iprod{\matL^k\VarGNDE_g\myt,\mZ_f\myt}_{\spaceRn}=\iprod{\matL^k\dt\VarGNDE_g\myt,\mZ_f\myt}_{\spaceRn}+\iprod{\matL^k\VarGNDE_g\myt,\dt\mZ_f\myt}_{\spaceRn}=\braket{\phi_{\dt\VarGNDE\myt}^*\parens{\mZ\myt} + \phi_{\VarGNDE\myt}^*\parens{\dt\mZ\myt}}_{fgk}$. Therefore, by Lemma \ref{lemma: phiX*Z leq X Z F norm}, we obtain 
\begin{align}\label{in the proof: ddt phi Xlt star Zt fgk estimate}
\normMax{\dt\braket{\phi_{\VarGNDE\myt}^*(\mZ\myt)}}\leq \normF{\dt\VarGNDE\myt}\normF{\mZ\myt} + \normF{\VarGNDE\myt}\normF{\dt\mZ\myt}.
\end{align}
By definition of $\mZ\myt$, it follows that $\normF{\mZ\myt}\leq \normMax{\mP\myt}\normF{\mQ\myt}$, and 
$\normF{\dt\mZ\myt}\leq\normF{\dt\mP\myt}\normMax{\mQ\myt}+\normMax{\mP\myt}\normF{\dt\mQ\myt}.$ The desired result immediately follows from plugging the above two estimates into \eq{in the proof: ddt phi Xlt star Zt fgk estimate}. 
\end{proof}

\begin{lemma}\label{lemma: phi phi* t1 - t2 X - Z PQ}
Suppose that \ConvFilterLipschitz\ holds. For any matrices $\mX_1,\mX_2,\mA_1,\mA_2\in\spaceRnF$ and $t_1,t_2\in[0,T]$, it holds that 
\begin{equation}\label{eq: phi t1 - t2 X - Z}
        \normF{\phi_{\vh\mylt[\ell][t_1]}(\mX_1) - \phi_{\vh\mylt[\ell][t_2]}(\mX_2)}\leq FK\Ch\normF{\mX_1 - \mX_2} + FK\abs{t_1-t_2}\normF{\mX_2},
    \end{equation}
    and
    \begin{align}
        &\normF{\phi_{\vh\mylt[\ell][t_1]}^*\parens{\mX_1\odot\mA_1} - \phi_{\vh\mylt[\ell][t_2]}^*\parens{\mX_2\odot\mA_2}}\nonumber \\
        \leq\ &FK\Ch\parens{\normMax{\mX_1}\normF{\mA_1-\mA_2} + \normF{\mX_1-\mX_2}\normMax{\mA_2}}+FK\Liph\abs{t_1-t_2}\normMax{\mX_2}\normF{\mA_2}. \label{eq: phi * t1 - t2 XP - ZQ}
    \end{align}
\end{lemma}
\begin{proof}
Note that 
\begin{align*}
    &\normF{\phi_{\vh\mylt[\ell][t_1]}(\mX_1) - \phi_{\vh\mylt[\ell][t_2]}(\mX_2)}\\
    \leq\ & \normF{\phi_{\vh\mylt[\ell][t_1]}(\mX_1 - \mX_2)} + \normF{\phi_{\vh\mylt[\ell][t_1]-\vh\mylt[\ell][t_2]}(\mX_2)}\stepjust{triangle inequality}\\
    \leq\ & FK\Ch\normF{\mX_1 - \mX_2} + FK\abs{t_1-t_2}\normF{\mX_2}\stepjust{\ConvFilterLipschitz\ and Lemma \ref{lemma: estimate of phi(matrix)}}
    \end{align*}
    and
    \begin{align*}
    & \normF{\phi_{\vh\mylt[\ell][t_1]}^*\parens{\mX_1\odot\mA_1}-\phi_{\vh\mylt[\ell][t_2]}^*\parens{\mX_2\odot\mA_2}} \\
    \leq\ & \normF{\phi_{\vh\mylt[\ell][t_1]}^*\parens{\mX_1\odot\mA_1-\mX_2\odot\mA_2}}+\normF{\phi_{\vh\mylt[\ell][t_1]-\vh\mylt[\ell][t_2]}^*\parens{\mX_2\odot\mA_2}}\stepjust{triangle inequality}\\
    \leq\ & FK\Ch\normF{\mX_1\odot\mA_1-\mX_2\odot\mA_2}+FK\Liph\abs{t_1-t_2}\normF{\mX_2\odot\mA_2}
    \stepjust{\ConvFilterLipschitz\ and Lemma \ref{lemma: estimate of phi(matrix)}}\\
    \leq\ &FK\Ch\parens{\normMax{\mX_1}\normF{\mA_1-\mA_2} + \normF{\mX_1-\mX_2}\normMax{\mA_2}}+FK\Liph\abs{t_1-t_2}\normMax{\mX_2}\normF{\mA_2}\stepjust{triangle inequality and \eq{eq: upper bound of matrix Z odot V}}
\end{align*}
\end{proof}

\begin{lemma}\label{lemma: phi* X1 hatX1 A1 - X2 hatX2 A2}
    For any matrices $\mX_1,\mX_2,\widehat{\mX}_1,\widehat{\mX}_2,\mA_1,\mA_2\in\spaceRnF$, it holds that 
    \begin{align*}
        &\normMax{\phi_{\mX_1}^*\parens{\widehat{\mX}_1\odot\mA_1}-\phi_{\mX_2}^*\parens{\widehat{\mX}_2\odot\mA_2}}\\
        &\leq \normF{\mX_1-\mX_2}\normF{\mA_1}\normMax{\widehat{\mX}_1}+\normF{\mX_2}\parens{\normMax{\widehat{\mX}_1}\normF{\mA_1-\mA_2}+\normF{\widehat{\mX}_1-\widehat{\mX}_2}\normMax{\mA_2}}. 
    \end{align*}
\end{lemma}
\begin{proof}
Note that
    \begin{align*}
&\normMax{\phi_{\mX_1}^*\parens{\widehat{\mX}_1\odot\mA_1}-\phi_{\mX_2}^*\parens{\widehat{\mX}_2\odot\mA_2}}\\
\leq\ &\normMax{\phi_{\mX_1}^*\parens{\widehat{\mX}_1\odot\mA_1}-\phi_{\mX_2}^*\parens{\widehat{\mX}_1\odot\mA_1}} + \normMax{\phi_{\mX_2}^*\parens{\widehat{\mX}_1\odot\mA_1}-\phi_{\mX_2}^*\parens{\widehat{\mX}_2\odot\mA_2}}\stepjust{triangle inequality}\\
=\ &\normMax{\phi_{\mX_1-\mX_2}^*\parens{\widehat{\mX}_1\odot\mA_1}} + \normMax{\phi_{\mX_2}^*\parens{\widehat{\mX}_1\odot\mA_1-\widehat{\mX}_2\odot\mA_2}} \stepjust{linearity of $\phi^*$}\\
\leq\ &\normF{\mX_1-\mX_2}\normF{\widehat{\mX}_1\odot\mA_1} + \normF{\mX_2} \normF{\widehat{\mX}_1\odot\mA_1-\widehat{\mX}_2\odot\mA_2}\stepjust{Lemma \ref{lemma: phiX*Z leq X Z F norm}}\\
\leq\ &\normF{\mX_1-\mX_2}\normF{\mA_1}\normMax{\widehat{\mX}_1}+\normF{\mX_2}\parens{\normMax{\widehat{\mX}_1}\normF{\mA_1-\mA_2}+\normF{\widehat{\mX}_1-\widehat{\mX}_2}\normMax{\mA_2}}\stepjust{triangle inequality and \eq{eq: upper bound of matrix Z odot V}}
    \end{align*}
    
\end{proof}
\section{Graphon Convolution Operators}\label{appendix: graphon convolutional operators}
Recall that graphon convolution operator $\Phi$ is defined in \eqref{def: operator Phi}. For a fixed $\vh\in\spaceRFFK$, we define an operator $\Phi_{\vh}:\spaceLTWO\to \spaceLTWO$, associated to $\vh$, by 
\begin{equation}\label{def: Phih}
    \Phi_{\vh}(\VarGraphonNDE):=\Phi(\vh,\VarGraphonNDE),\quad \VarGraphonNDE\in\spaceLTWO.
\end{equation} 
For a fixed $\VarGraphonNDE\in\spaceLTWO$, we define an operator $\Phi_{\VarGraphonNDE}:\spaceRFFK\to \spaceLTWO$, associated to $\VarGraphonNDE$, by
\begin{equation}\label{def: frakx}
\Phi_{\VarGraphonNDE}(\vh):=\Phi(\vh,\VarGraphonNDE),\quad \vh\in\spaceRFFK.
\end{equation}
It is straightforward to verify that both operators $\Phi_{\vh}$ and $\Phi_{\VarGraphonNDE}$ are linear.

\paragraph{Adjoint operators.} Note that the integral operator $\opLP$ is self-adjoint, then the adjoint operator of $\Phi_{\vh}$, denoted by $\Phi_{\vh}^*:\spaceLTWO\to \spaceLTWO$, is given by
\begin{equation}\label{def: adjoint Phih}
\Phi_{\vh}^*(\gZ):=\left[\sum_{f=1}^F\sum_{k=0}^{K-1}\vh_{fgk}\opLP^k\gZ:g\in[F]\right],\quad \gZ\in\spaceLTWO,
\end{equation}
and the adjoint operator of $\Phi_{\VarGraphonNDE}$, denoted by $\Phi_{\VarGraphonNDE}^*:\spaceLTWO\to \spaceRFFK$, is given by
\begin{equation*}
\Phi_{\VarGraphonNDE}^*(\gZ):=\left[\iprod{\opLP^k \VarGraphonNDE_g,\gZ_f }_{\spaceltwo}:f,g\in[F],k\in\mathbb{Z}_{K}\right],\quad \gZ\in\spaceLTWO.
\end{equation*}

\paragraph{Adjoint operators of Gateaux derivatives of single-layer Graphon-NNs.} Let $\tildeVarGraphonNDE=\Phi(\vh,\VarGraphonNDE)$ and $\gY=\sigma(\tildeVarGraphonNDE)$. Let $\hatVarGraphonNDE:=\sigma'(\tildeVarGraphonNDE)$. By Chain rule, the Gateaux derivatives of $\gY$ with respect to $\VarGraphonNDE$ and $\vh$ are given respectively by 
\begin{align*}
    \frac{\delta \gY}{\delta \VarGraphonNDE}(\gZ)=\hatVarGraphonNDE\odot \Phi_{\vh}(\gZ),\ \  \gZ\in\spaceLTWO,\qquad
    \frac{\delta \gY}{\delta \vh}(\vz)=\hatVarGraphonNDE\odot \Phi_{\VarGraphonNDE}(\vz),\ \  \vz\in\spaceRFFK.
\end{align*}
Their adjoint operators are given by 
\begin{align}\label{eq: adjoint delta tY / delta tX, delta h}
    \parens{\frac{\delta\gY}{\delta\VarGraphonNDE}}^*(\gV)=\Phi_{\vh}^*\parens{\hatVarGraphonNDE\odot \gV},\quad
    \parens{\frac{\delta\gY}{\delta\vh}}^*(\gV)=\Phi_{\VarGraphonNDE}^*\parens{\hatVarGraphonNDE\odot \gV},\quad \gV\in\spaceLTWO.
\end{align}

\paragraph{Adjoint operators of Gateaux derivatives of multi-layer Graphon-NNs.} Now we consider the multi-layer Graphon-NNs defined in \eqref{def: Graphon-NN multi layer}. Let 
\begin{equation}\label{def: widehat tX ell}
\hatVarGraphonNDE\myl:=\sigma'(\tildeVarGraphonNDE\myl),\quad \ell\in[L]. 
\end{equation}
It is straightforward to compute with Chain rule and \eq{eq: adjoint delta tY / delta tX, delta h} that for $\gV\in\spaceLTWO$,
\begin{align*}
\parens{\frac{\delta \GraphonNN}{\delta \VarGraphonNDE^{(0)}}}^*(\gV)&=\Phi_{\vh^{(1)}}^*\parens{\hatVarGraphonNDE^{(1)}\odot\cdots\Phi_{\vh^{(L-1)}}^*\parens{\hatVarGraphonNDE^{(L-1)}\odot\parens{\Phi_{\vh^{(L)}}^*\parens{\hatVarGraphonNDE^{(L)}\odot \gV}}}},\\
\parens{\frac{\delta \GraphonNN}{\delta\vh\myl}}^*(\gV)&=\Phi_{\VarGraphonNDE\myl}^*\parens{\hatVarGraphonNDE\myl\odot\cdots\Phi_{\vh^{(L-1)}}^*\parens{\hatVarGraphonNDE^{(L-1)}\odot\parens{\Phi_{\vh^{(L)}}^*\parens{\hatVarGraphonNDE^{(L)}\odot \gV}}}}.
\end{align*}
We let 
\begin{align}
\VarGraphonNDEHiddenState^{(L)}:=\gV,\quad\VarGraphonNDEHiddenState\myl:=\Phi_{\vh\myl[\ell+1]}^*\parens{\hatVarGraphonNDE\myl[\ell+1]\odot \VarGraphonNDEHiddenState\myl[\ell+1]},\quad \ell=L-1,L-2,\cdots,0\label{updating rule for tY ell}
\end{align}
which enables to rewrite 
\begin{equation}\label{adjoint operator: Graphon-NDE to filters}
    \parens{\frac{\delta \GraphonNN}{\delta \VarGraphonNDE^{(0)}}}^*(\VarGraphonNDEHiddenState^{(L)})=\VarGraphonNDEHiddenState^{(0)},\quad \parens{\frac{\delta \GraphonNN}{\delta\vh\myl}}^*(\VarGraphonNDEHiddenState^{(L)})=\Phi_{\VarGraphonNDE\myl}^*\parens{\hatVarGraphonNDE\myl\odot\VarGraphonNDEHiddenState\myl}.
\end{equation}

\paragraph{Norm estimates.} In the following, we present several useful norm estimates regarding to graphon convolution operators. Note that by letting $c_{\max}:=\sup_{u,v\in I}\tW(u,v)$ and \dPBoundedBelow, we have
\begin{equation}\label{TW operator infinity norm is bounded}
\|\opLP\|_{\spaceb\to \spaceb}\leq \sup_{u,v\in I}|\tL_P(u,v)|\leq c_{\max}/c_{\min}.
\end{equation}
\begin{lemma}\label{lemma: norm of operator mathfrak H and adjoint operator widehat mathfrak H}
If \ConvFilterLipschitz\ and \dPBoundedBelow\ hold, then for all $\ell\in[L]$, 
    \begin{align}\label{eq: norm of operator mathfrak H}
&\norm{\Phi_{\vh\mylt}}_{\spaceB\to \spaceB}\leq F\Csys[0]\Ch,\\
    \Or{&\norm{\Phi_{\vh\mylt}^*}_{\spaceB\to \spaceB}}\nonumber
    \end{align}
    where constant $\Ch$ is defined in \eq{def: constant Ch} and
    \begin{equation}\label{def: C0W}
\Csys[0]:=\sum_{k=0}^{K-1}\parens{\frac{c_{\max}}{c_{\min}}}^k.
\end{equation}
\end{lemma}

\begin{proof}
    Let $\gX\in\spaceB$ be arbitrary but fixed, and $\gY:=\Phi_{\vh\mylt}(\gX)$. It follows from definition of $\Phi_{\vh\mylt}$ and graphon convolution that $\gY_f=\sum_{g=1}^F\sum_{k=0}^{K-1}\hltfgk\opLP^k \gX_g$. Then by \hyperlink{AS2}{AS2} and \eq{TW operator infinity norm is bounded}, we have $\norm{\gY_f }_{\spaceb}\leq \sum_{k=0}^{K-1}\sum_{g=1}^{F}\left|\hltfgk\right|\frac{c_{\max}^k}{c_{\min}^k} \norm{\gX_g}_{\spaceb}\leq \parens{\Csys[0]\Ch}\sum_{g=1}^{F}\norm{ \gX_g }_{\spaceb}$. It follows that $
    \norm{\gY_f}_{\spaceb}^2\leq F\parens{\Csys[0]\Ch}^2\sum_{g=1}^{F}\norm{ \gX_g }_{\spaceb}^2=F\parens{\Csys[0]\Ch}^2\norm{ \gX}_{\spaceB}^2$. We sum the above inequality over $f\in[F]$ for both sides and obtain that $
    \norm{\gY }_{\spaceB}^2=\sum_{f=1}^F\norm{\gY_f}_{\spaceb}^2\leq  \parens{F\Csys[0]\Ch}^2\|\gX\|_{\spaceB}^2$, that is, for any $\gX\in\spaceB$, there holds $\norm{\Phi_{\vh\mylt}(\gX)}_{\spaceB}\leq \parens{F\Csys[0]\Ch} \|\gX\|_{\spaceB}$. This proves \eq{eq: norm of operator mathfrak H}. The proof for operators $\Phi_{\vh\mylt}^*$ is similar. 
\end{proof}

\begin{lemma}\label{lemma: operator frakH is lipschitz continuous about t}
Suppose that $\gZ=[\gZ_g:g\in[F]]\in \spaceC$ and for all $g\in[F]$, $\gZ_g$ is $\mathrm{Lip}(\gZ)$-Lipschitz continuous about $t$, i.e., 
\begin{equation}\label{eq: Vg is lipschitz continuous}
\norm{\gZ_g(\cdot,t_1)-\gZ_g(\cdot,t_2)}_{\spaceb}\leq\mathrm{Lip}(\gZ)|t_1-t_2|,\quad\forall t_1,t_2\in[0,T], \forall g\in[F].
\end{equation}
Then for any $\ell\in[L]$ and $f\in[F]$, it holds that
\begin{align*}
    \norm{\left[\Phi_{\vh\mylt[\ell][t_1]}(\gZ(\cdot,t_1))\right]_f-\left[\Phi_{\vh\mylt[\ell][t_2]}(\gZ(\cdot,t_2))\right]_f}_{\spaceb}&\leq C|t_1-t_2|,\quad\forall t_1,t_2\in[0,T],\\
    \Or{\norm{\left[\Phi_{\vh\mylt[\ell][t_1]}^*(\gZ(\cdot,t_1))\right]_f-\left[\Phi_{\vh\mylt[\ell][t_2]}^*(\gZ(\cdot,t_2))\right]_f}_{\spaceb}&}
\end{align*}
where $C:=\Csys[0]\parens{\Liph\sqrt{F} \|\gZ\|_{\spaceC}+F\Ch \mathrm{Lip}(\gZ)}$.
\end{lemma}

\begin{proof}
We obtain from triangle inequality and definition of operators $\Phi_{\vh\mylt}$ that, for $t_1,t_2\in[0,T]$, $\norm{\left[\Phi_{\vh\mylt[\ell][t_1]}(\gZ(\cdot,t_1))\right]_f-\left[\Phi_{\vh\mylt[\ell][t_2]}(\gZ(\cdot,t_2))\right]_f}_{\spaceb}\leq \Delta_1+\Delta_2$ where 
\begin{align*}
\Delta_1:=\norm{\sum_{g=1}^{F}\sum_{k=0}^{K-1}(\hltfgk[t_1]-\hltfgk[t_2])\opLP^k \gZ_g(\cdot,t_1)}_{\spaceb},\Delta_2:=\norm{\sum_{g=1}^{F}\sum_{k=0}^{K-1}\hltfgk[t_2] \opLP^k \parens{\gZ_g(\cdot,t_1)-\gZ_g(\cdot,t_2)} }_{\spaceb}.
\end{align*}
Note that
\begin{align*}
&\Delta_1\leq \sum_{g=1}^{F}\sum_{k=0}^{K-1}\left|\hltfgk[t_1]-\hltfgk[t_2]\right| \|\opLP\|_{\spaceb\to\spaceb}^k \|\gZ_g(\cdot,t_1)\|_{\spaceb}\\
&\leq \sum_{g=1}^{F}\sum_{k=0}^{K-1}\Liph|t_1-t_2|\frac{c_{\max}^k}{c_{\min}^k} \|\gZ_g(\cdot,t_1)\|_{\spaceb}=\parens{\Csys[0] \Liph|t_1-t_2|}\sum_{g=1}^{F}\|\gZ_g(\cdot,t_1)\|_{\spaceb}\stepjust{\ConvFilterLipschitz\ and \eq{TW operator infinity norm is bounded}}\\
&\leq \Csys[0] \Liph\sqrt{F} \|\gZ\|_{\spaceC}|t_1-t_2| \stepjust{norm defined in $\spaceC$}
\end{align*}
and 
\begin{align*}
\Delta_2&\leq \sum_{g=1}^{F}\sum_{k=0}^{K-1}\left|\hltfgk[t_2] \right| \|\opLP\|_{\spaceb\to\spaceb}^k \norm{\gZ_g(\cdot,t_1)-\gZ_g(\cdot,t_2)}_{\spaceb}\\
\leq & \sum_{g=1}^{F}\sum_{k=0}^{K-1} \Ch \frac{c_{\max}^k}{c_{\min}^k} \mathrm{Lip}(\gZ)|t_1-t_2|= F\Csys[0]\Ch \mathrm{Lip}(\gZ)|t_1-t_2| \stepjust{\eq{TW operator infinity norm is bounded} and \eq{eq: Vg is lipschitz continuous}}.
\end{align*}
We apply estimates of $\Delta_1$ and $\Delta_2$ and obtain the desired result. The proof is similar when operator $\Phi_{\vh\mylt}$ is replaced by $\Phi_{\vh\mylt}^*$.
\end{proof}

\section{Iterative MSE estimates}
\paragraph{Empirical operators.}
Recall that $\opLUn$ is the empirical graphon integral operator defined in \eqref{def: empirical graphon integral operator LUn}. Let $\vh\in\mathbb{R}^{F\times F\times K}$. In the following, we introduce several empirical operators. We define an operator $\Phi_{\vh,U_n}:\spaceB\to\spaceB$  (empirical version of \eq{def: Phih}) as
\begin{equation*}
\Phi_{\vh,U_n}(\gZ):=\left[\sum_{g=1}^{F}\sum_{k=0}^{K-1} \vh_{fgk} \opLUn^k \gZ_g:f\in [F]\right],\quad \gZ\in\spaceB.
\end{equation*}
Define an operator $\Phi_{\vh,U_n}^*:\spaceB\to \spaceB$ (empirical version of \eq{def: adjoint Phih}) as
\begin{equation*}
    \Phi_{\vh,U_n}^*(\gZ):=\left[\sum_{f=1}^{F}\sum_{k=0}^{K-1}\vh_{fgk} \opLUn^k \gZ_f:g\in [F]\right],\quad \gZ\in\spaceB.
\end{equation*}
Let matrix $\matLUn$ be defined in \eqref{def: matrix LUn}. For each $\ell\in[L]$, we define operators
\begin{equation*}
\phi_{\vh,U_n}(\mZ):=\left[\sum_{g=1}^{F}\sum_{k=0}^{K-1}\vh_{fgk} \matLUn^k \mZ_f:f\in [F]\right],\quad \mZ\in\spaceRnF,
\end{equation*}
and 
\begin{equation*}
\phi_{\vh,U_n}^*(\mZ):=\left[\sum_{f=1}^F\sum_{k=0}^{K-1}\vh_{fgk}\matLUn^k\mZ_f:g\in[F]\right],\quad \mZ\in\spaceRnF.
\end{equation*}

The following observations are useful. For any function $Z\in\spaceb$, 
\begin{equation}\label{eq: delta Phi = phi delta}
\matLUn\SamplingOperator(Z)=\SamplingOperator(\opLUn Z),
\end{equation}
\begin{equation}\label{norm of operator delta Un one d}
\|\SamplingOperator(Z)\|_{2}\leq \sqrt{n}\normb{Z}.
\end{equation}
For any functions $\gZ,\gV\in\spaceB$, 
\begin{equation}\label{eq: upper bound of function Z odot V}
\norm{\gZ\odot\gV}_{\spaceB}\leq \norm{\gZ}_{\spaceB}\norm{\gV}_{\spaceB}.
\end{equation}
For any matrices $\mZ,\mV\in\spaceRnF$, 
\begin{equation}\label{eq: upper bound of matrix Z odot V}
\normF{\mZ\odot\mV}\leq \normF{\mZ}\normMax{\mV}.
\end{equation}
For any function $\gZ\in\spaceB$, 
\begin{equation}\label{swap delta Un frakH = frakh delta Un}
\SamplingOperator(\Phi_{\vh,U_n}(\gZ))=\phi_{\vh,U_n}\parens{\SamplingOperator(\gZ)},
\end{equation}
\begin{equation}\label{norm of operator delta Un}
\|\SamplingOperator(\gZ)\|_{\mathrm{F}}\leq \sqrt{n}\|\gZ\|_{\spaceB},
\end{equation}
\begin{equation}\label{norm of operator delta Un max}
\|\SamplingOperator(\gZ)\|_{\max}\leq \max_{f\in[F]}\normb{\gZ_f}\leq \normB{\gZ}.
\end{equation}

\begin{lemma}\label{lemma: diff frakHP - frakHU matrix}
Suppose that \ConvFilterLipschitz\ holds. For any matrix $\mZ\in\spaceRnF$, it holds that 
    \begin{align*}
&\norm{\phi_{\vh\mylt,U_n}\parens{\mZ} - \phi_{\vh\mylt}\parens{\mZ}}_{\mathrm{F}} \leq FK^2\Ch\|\matLUn-\matL\|_2\|\mZ\|_{\mathrm{F}}\\
\Or{&\norm{\phi_{\vh\mylt,U_n}^*\parens{\mZ} - \phi_{\vh\mylt}^*\parens{\mZ}}_{\mathrm{F}}}.
    \end{align*}
\end{lemma}
\begin{proof}
Note that $\normTwo{\matLUn}\leq 1$ and $\normTwo{\matL}\leq1$, implying
$\normTwo{\matLUn^k-\matL^k}\leq k\normTwo{\matLUn-\matL},\quad k\in\mathbb{Z}_K$. The result then immediately follows from standard norm estimates.
\end{proof}

\begin{lemma}\label{lemma: diff frakHP - frakHU tV}
Suppose that \ConvFilterLipschitz\ holds. For any $\gZ\in\spaceB$, it holds that
\begin{align*}
&\normB{\Phi_{\vh\mylt}(\gZ) - \Phi_{\vh\mylt,U_n}(\gZ)}\Or{\normB{\Phi_{\vh\mylt}^*(\gZ) - \Phi_{\vh\mylt,U_n}^*(\gZ)}}\\
&\leq FK\Ch\norm{\opLUn}_{\spaceb\to\spaceb}^K\sqrt{\sum_{g=1}^F\sum_{k=1}^K\sum_{s=0}^{k-1}\norm{(\opLUn-\opLP)\parens{\opLP^{k-1-s}\gZ_g}}_{\spaceb}^2}
\end{align*}
\end{lemma}

\begin{proof}
This result can be directly shown by noting that the difference between the $k$-th powers of the operators can be expanded via the telescoping identity $\opLUn^k - \opLP^k = \sum_{s=0}^{k-1} \opLUn^s (\opLUn - \opLP) \opLP^{k-1-s}$, followed by standard norm inequalities.

\end{proof}

\begin{lemma}\label{lemma: MSE hatfrakH function Z - hatfrakh matrix Z}
Suppose that \ConvFilterLipschitz\ holds. For any matrix $\mZ\in\spaceRnF$ and function $\gZ\in\spaceB$, it holds that for all $t\in[0,T]$ and $\ell\in[L]$,
\begin{align*}
&\MSE(\Phi_{\vh\mylt}(\gZ),\phi_{\vh\mylt}(\mZ))\leq \parens{FK\Ch}\MSE(\gZ,\mZ) + \Delta_1 + \Delta_2, \\
\Or{&\MSE(\Phi_{\vh\mylt}^*(\gZ),\phi_{\vh\mylt}^*(\mZ)))}
\end{align*}
where $\Delta_1:=FK^2\Ch\|\matLUn-\matL\|_2\|\gZ\|_{\spaceB}$ and 
$$
\Delta_2:=FK\Ch\norm{\opLUn}_{\spaceb\to\spaceb}^K\sqrt{\sum_{g=1}^F\sum_{k=1}^K\sum_{s=0}^{k-1}\norm{(\opLUn-\opLP)\parens{\opLP^{k-1-s}\gZ_g}}_{\spaceb}^2}.
$$
\end{lemma}

\begin{proof}
By triangle inequality, we have 
$$\MSE(\Phi_{\vh\mylt}(\gZ),\phi_{\vh\mylt}(\mZ))=\frac{1}{\sqrt{n}}\normF{\SamplingOperator(\Phi_{\vh\mylt}(\gZ))-\phi_{\vh\mylt}(\mZ)}\leq \Delta+\widetilde{\Delta},$$ where $\Delta:=\frac{1}{\sqrt{n}}\norm{\SamplingOperator(\Phi_{\vh\mylt}(\gZ))-\phi_{\vh\mylt}\parens{\SamplingOperator(\gZ)}}_{\mathrm{F}}$ and $\widetilde{\Delta}:=\frac{1}{\sqrt{n}}\norm{\phi_{\vh\mylt}(\SamplingOperator(\gZ)-\mZ)}_{\mathrm{F}}$. We first estimate $\Delta$. Again by triangle inequality, we have $\Delta\leq\alpha_1+\alpha_2$, where $\alpha_1:=\frac{1}{\sqrt{n}}\norm{\SamplingOperator(\Phi_{\vh\mylt}(\gZ)-\Phi_{\vh\mylt,U_n}(\gZ))}_{\mathrm{F}}$ and $\alpha_2:=\frac{1}{\sqrt{n}}\norm{\SamplingOperator(\Phi_{\vh\mylt,U_n}(\gZ)) - \phi_{\vh\mylt}\parens{\SamplingOperator(\gZ)}}_{\mathrm{F}}$
Note that by \eq{swap delta Un frakH = frakh delta Un}, Lemma \ref{lemma: diff frakHP - frakHU matrix}, \eq{norm of operator delta Un}, and definition of $\Delta_1$, we have
\begin{align*}
\alpha_1&=\frac{1}{\sqrt{n}}\norm{\phi_{\vh\mylt,U_n}\parens{\SamplingOperator(\gZ)} - \phi_{\vh\mylt}\parens{\SamplingOperator(\gZ)}}_{\mathrm{F}} \leq \frac{1}{\sqrt{n}}FK^2\Ch\|\matLUn-\matL\|_2\|\SamplingOperator(\gZ)\|_{\mathrm{F}}\leq \Delta_1.
\end{align*}
Moreover, due to \eq{norm of operator delta Un}, Lemma \ref{lemma: diff frakHP - frakHU tV}, and definition of $\Delta_2$, we have 
\begin{align*}
\alpha_2&\leq\normB{\Phi_{\vh\mylt}(\gZ)-\Phi_{\vh\mylt,U_n}(\gZ)}\leq \Delta_2.
\end{align*}
In addition, by Lemma \ref{lemma: estimate of phi(matrix)}, we obtain that
\begin{align*}
    \widetilde{\Delta}&\leq \frac{1}{\sqrt{n}}FK\Ch\|\SamplingOperator(\gZ)-\mZ\|_{\mathrm{F}}=\parens{FK\Ch}\MSE(\gZ,\mZ).
\end{align*}
We get the desired result by collecting all above estimates. 
\end{proof}

\section{Auxiliary results}
The following lemma extends Lemma~4 in \cite{keriven2020convergence} to a time-uniform setting, and its proof follows the same chaining argument.

\begin{lemma}\label{lemma: concentration ineq Yut}
Let $\tW:I^2\to I$ be a graphon satisfying \GraphonLipschitz. Let $\Omega=I\times[0,T]$. Suppose that function $X:\Omega\to\mathbb{R}$ is Lipschitz continuous about $t$, i.e., there exists a positive constant $C_{\mathrm{Lip}}$ such that for any $u\in I$, 
    \begin{equation}\label{f_t is Lipschitz continuous about t}
        |X(u,t_1)-X(u,t_2)|\leq C_{\mathrm{Lip}}|t_1-t_2|,\quad\forall t_1,t_2\in[0,T]. 
    \end{equation}
    Suppose that $C_{\max}:=\supT\|X(\cdot,t)\|_{\spaceb}$ exists. Let $u_j$, $j\in[n]$ be independent random variables following a distribution $P$. For each $(u,t)\in \Omega$, let 
    \begin{equation}\label{def Yut}
    Y_{(u,t)}:=\frac{1}{n} \sum_{j=1}^n \tW(u, u_j) X(u_j,t)-\int_I \tW(u, v) X(v,t) d P(v). 
    \end{equation}
    Then with probability at least $1-\gamma$, it holds that 
    \begin{equation}\label{sup Yut lesssim sqrt(n)}
    \sup_{(u,t)\in \Omega}\left|Y_{(u,t)}\right|\lesssim\frac{\sqrt{\log(4n_I/\gamma)}}{\sqrt{n}}\parens{1+ C_T}\max\curlbraket{c_{\max},c_{\mathrm{Lip}}}\max\curlbraket{C_{\max},C_{\mathrm{Lip}}},
    \end{equation}
    where $C_T$ is a constant only depending on $T$. 
\end{lemma}

\begin{proof}
Let $\{I_s:s\in[n_I]\}$ be the partition of $I$ assumed in \GraphonLipschitz. Let $s\in[n_I]$ be arbitrary but fixed. By $\Omega_s$ we denote the set $I_{s}\times[0,T]$. Let $(u_0,t_0)$ be an arbitrary but fixed point in $\Omega_s$. By triangle inequality, we have 
\begin{equation}\label{in the proof Yut leq Yut-Yu't' + Yu0t0}
    \sup_{(u,t)\in \Omega_s}\left|Y_{(u,t)}\right|\leq 
    \left|Y_{(u_0,t_0)}\right| + \sup_{(u,t),(u',t')\in \Omega_s}\left|Y_{(u,t)}-Y_{(u',t')}\right|.
\end{equation}
It is clear that, for each $j\in[n]$, we have $|\tW(u_0,u_j)X(u_j,t)|\leq c_{\max}C_{\max}$. Therefore, by Hoeffding's inequality, with probability at least $1-\gamma/2$, it holds that
\begin{equation}\label{in the proof estimate Yu0t0}
    \left|Y_{(u_0,t_0)}\right|\lesssim \frac{c_{\max}C_{\max}\sqrt{\log(4/\gamma)}}{\sqrt{n}}.
\end{equation}
In addition, note that for each $(u,t)\in\Omega$, $|Y_{(u,t)}|\leq 2c_{\max}C_{\max}$. This implies that the random variable $Y_{(u,t)}$ is bounded, and hence sub-gaussian. For any $(u,t),(u',t')\in \Omega_s$, via triangle inequality of sub-gaussian norm $\|\cdot\|_{\psi_2}$, we have $\norm{Y_{(u,t)}-Y_{(u',t')}}_{\psi_2}\leq \norm{Y_{(u,t)}-Y_{(u',t)}}_{\psi_2}+\norm{Y_{(u',t)}-Y_{(u',t')}}_{\psi_2}.$ Similar to the proof in \cite{keriven2020convergence}, by the properties of sub-gaussian process, we have $\norm{Y_{(u,t)}-Y_{(u',t)}}_{\psi_2}\lesssim\frac{c_{\mathrm{Lip}}}{\sqrt{n}}C_{\max}|u-u'|$, and $\norm{Y_{(u',t)}-Y_{(u',t')}}_{\psi_2}\lesssim\frac{c_{\max}}{\sqrt{n}} C_{\mathrm{Lip}} |t-t'|$. It follows that
$
\norm{Y_{(u,t)}-Y_{(u',t')}}_{\psi_2}\lesssim \frac{\max\left\{c_{\mathrm{Lip}}C_{\max}, c_{\max}C_{\mathrm{Lip}}\right\}}{\sqrt{n}} \norm{(u,t),(u'-t')}_1$. Then we apply Dudley's inequality and obtain that with probability at least $1-\gamma/2$, 
\begin{equation}\label{in the proof sup ut u't' leq Zf sqrt n int 0 infinity}
\begin{aligned}
    &\sup_{(u,t),(u',t')\in\Omega_s}\abs{Y_{(u,t)}-Y_{(u',t')}}\\
    &\lesssim \frac{\max\left\{c_{\mathrm{Lip}}C_{\max}, c_{\max}C_{\mathrm{Lip}}\right\}}{\sqrt{n}}\parens{\int_0^\infty\sqrt{\log N(\Omega_s,\|\cdot\|_1,\epsilon)}d\epsilon+\mathrm{diam}(\Omega_s)\sqrt{\log(4/\gamma)}},
\end{aligned}
\end{equation}
where $N(\Omega_s,\|\cdot\|_1,\epsilon)$ is the covering number of $\Omega_s$ with respect to $\ell_1$ norm and radius $\epsilon$, and $\mathrm{diam}(\Omega_s)$ is the diameter of $\Omega_s$ defined by $\mathrm{diam}(\Omega_s):=\sup\{\|(u,t)-(u',t')\|_2:(u,t),(u',t')\in\Omega_s\}$. Noting that the integral in \eq{in the proof sup ut u't' leq Zf sqrt n int 0 infinity} is finite and only depending on $T$; also $\mathrm{diam}(\Omega_s)$ is bounded above by $\sqrt{T^2+1}$. By $C_T$ we denote a uniform upper bound (only relying on $T$) for the integral in \eq{in the proof sup ut u't' leq Zf sqrt n int 0 infinity} and $\mathrm{diam}(\Omega_s)$. It follows that 
\begin{equation}\label{in the proof sup ut u't' Yut - Yu't'}
    \sup_{(u,t),(u',t')\in\Omega_s}\left|Y_{(u,t)}-Y_{(u',t')}\right|\lesssim\frac{\max\left\{c_{\mathrm{Lip}}C_{\max}, c_{\max}C_{\mathrm{Lip}}\right\}}{\sqrt{n}}\parens{C_T+C_T\sqrt{\log(4/\gamma)}}. 
\end{equation}
With plugging estimates \eq{in the proof estimate Yu0t0} and \eq{in the proof sup ut u't' Yut - Yu't'} into \eq{in the proof Yut leq Yut-Yu't' + Yu0t0}, then with probability at least $1-\gamma$, there holds  
\begin{align*}
    \sup_{(u,t)\in \Omega_s}\abs{Y_{(u,t)}}\lesssim \frac{\sqrt{\log(4/\gamma)}}{\sqrt{n}}\parens{1+ C_T}\max\curlbraket{c_{\max},c_{\mathrm{Lip}}}\max\curlbraket{C_{\max},C_{\mathrm{Lip}}}.
\end{align*}
With applying an union bound for all $s\in[n_I]$, we obtain \eq{sup Yut lesssim sqrt(n)} as desired. 
\end{proof}

Given $U_n=\{u_j:j\in[n]\}$, a set of $n$ distinct points in $I$, we define the empirical degree function $d_{U_n}:I\to I$ by $d_{U_n}(u):=\frac{1}{n}\sum_{j=1}^n \tW(u,u_j)$, $u\in I$, the empirical symmetric normalized adjacency function $\tL_{U_n}:I^2\to I$ by
\begin{equation}\label{def: function LUn}
    \tL_{U_n}(u,v):=\frac{\tW(u,v)}{\sqrt{d_{U_n}(u)d_{U_n}(v)}},\quad u,v\in I,
\end{equation}
and the empirical graphon integral operator $\opLUn:\spaceb\to \spaceb$ by
\begin{equation}\label{def: empirical graphon integral operator LUn}
    \opLUn X:=\frac{1}{n}\sum_{j=1}^n \tL_{U_n}(\cdot,u_j)X(u_j), \quad X \in \spaceb. 
\end{equation}

Following the chaining argument in Lemma~5 of \cite{keriven2020convergence}, we obtain the following two auxiliary lemmas adapted to our setting.
\begin{lemma}\label{lemma: LUn-LP leq sqrt(n)}
Suppose that \dPBoundedBelow\ and \GraphonLipschitz\ hold. Then with probability at least $1-\gamma_1$, 
\begin{equation}\label{eq: sup dP dUn}
\sup_{u\in I}\abs{d_{P}(u)-d_{U_n}(u)}\lesssim \max\curlbraket{c_{\max},c_{\mathrm{Lip}}}\frac{\sqrt{\log(4n_I/\gamma_1)}}{\sqrt{n}}.
\end{equation}
Define a constant 
\begin{equation}\label{def: C1W}
\Csys[1]:=\parens{\frac{\max\curlbraket{c_{\max},c_{\mathrm{Lip}}}}{c_{\min}}}^2. 
\end{equation}
Moreover, if \eq{eq: sup dP dUn} holds and \begin{equation}\label{eq: n large enough for integral operator LP - LUn}
n\gtrsim \Csys[1]\log(4n_I/\gamma_1),    
\end{equation}
then 
\begin{align}
    d_{U_n}(u)&\gtrsim c_{\min},\quad\forall u\in I,\label{in the proof: dUn is bounded below}\\
    \norm{\opLUn}_{\spaceb\to\spaceb}&\lesssim \frac{c_{\max}}{c_{\min}},\label{upper bound of norm of L Un}\\
    |\tL_{U_n}(u,v)-\tL_P(u,v)|&\lesssim \Csys[1]\frac{\sqrt{\log(4n_I/\gamma_1)}}{\sqrt{n}}. \label{eq: LUn-LP leq sqrt(n)}
\end{align}

\end{lemma}

\begin{lemma}\label{lemma: sup LUn - LP X leq 1/sqrt(n)}
Suppose that \dPBoundedBelow\ and \GraphonLipschitz\ hold. Let $\Omega=I\times[0,T]$. Suppose that function $X:\Omega\to\mathbb{R}$ is Lipschitz continuous about $t$, i.e., \eq{f_t is Lipschitz continuous about t}. Suppose that $C_{\max}:=\supT\|X(\cdot,t)\|_{\spaceb}$ exists. Let $u_j$, $j\in[n]$ be independent random variables following a distribution $P$. If \eqs \eqref{eq: sup dP dUn} and \eqref{eq: n large enough for integral operator LP - LUn} hold, then with probability at least $1-\gamma$,
\begin{align*}
    \sup_{(u,t)\in \Omega}\abs{\parens{\parens{\opLUn-\mathcal{L}_{P}}X(\cdot,t)}(u)} &\lesssim \parens{\Csys[1]\frac{\sqrt{\log(4n_I/\gamma_1)}}{\sqrt{n}} + \Csys[2]\frac{\sqrt{\log(4n_I/\gamma)}}{\sqrt{n}}}\max\curlbraket{C_{\max},C_{\mathrm{Lip}}}.    
\end{align*}  
where $\Csys[1]$ is defined in \eqref{def: C1W} and 
\begin{equation}\label{def: C2W}
    \Csys[2]:=\parens{1+ C_T}\max\curlbraket{\frac{c_{\max}}{c_{\min}},\frac{c_{\mathrm{Lip}}}{c_{\min}}\parens{1+\frac{c_{\max}}{2c_{\min}}}}.
\end{equation}
\end{lemma}

The following result is a slight modification of Corollary 1 in \cite{keriven2020convergence}. 
\begin{lemma}\label{lemma: Ln - LWn 2norm}
Suppose that \eq{eq: sup dP dUn} holds, 
\begin{equation}\label{eq: n large enough for matrix L - LUn}
n\gtrsim\Csys[1]\log(4n_I/\gamma_1)+1/\gamma_1,
\end{equation}
and
\begin{equation}\label{assumption: sparsity alpha n large enough}
\alpha_n\gtrsim c_{\max}c_{\min}^{-2}n^{-1}\log n. 
\end{equation}
Then, with probability at least $1-\gamma_1$, there holds
\begin{equation}\label{eq: Ln - LWn 2norm}
    \|\matL-\matLUn\|_2\lesssim \frac{c_{\max}}{c_{\min}^2}\frac{1}{\sqrt{\alpha_n n}},
\end{equation}
where 
\begin{equation}\label{def: matrix LUn}
\matLUn:=\mD_{U_n}^{-1/2}\mW_{U_n}\mD_{U_n}^{-1/2},
\end{equation}
$\mW_{U_n}:=[\tW(u_i,u_j):i,j\in[n]]\in\mathbb{R}^{n\times n}$, and $\mD_{U_n}$ is the degree matrix of $\mW_{U_n}$.
\end{lemma}

The following result is a special case of Theorem 21 in \cite{dragomir2003some}; the backward version is obtained through a standard change of variables.
\begin{lemma}[Generalized Gr\"onwall’s inequality]\label{lemma: generalized Gronwall's inequality}
Let $a,b,T$ be positive constants. Let $u:[0,T]\to\mathbb{R}$ be a non-negative function. 
\begin{align*}
    \text{If for all }t\in[0,T],\ u(t)&\leq \int_{0}^t\parens{a u(s)+b \sqrt{u(s)}} ds,\ \text{then}\ u(t)\leq \parens{\frac{\mathrm{exp}(at/2)-1}{a}b}^2,\ t\in[0,T].\\
    \text{If for all }t\in[0,T],\ u(t)&\leq \int_{t}^T\parens{a u(s)+b \sqrt{u(s)}} ds,\ \text{then}\ u(t)\leq \parens{\frac{\mathrm{exp}(a(T-t)/2)-1}{a}b}^2,\ t\in[0,T].
\end{align*}

\end{lemma}

\begin{lemma}[Convergence for Euler's Method] \label{lemma: Euler's method}
Let $a,b\in\mathbb{R}$ with $a<b$. Consider the ordinary differential equation defined by $\frac{dx}{dt} = f(x,t)$, $t\in [a,b]$ with initial value conition $x(a) = x_0$. Suppose that the solution \( x \) has a bounded second derivative with an upper bound $B$; and \( f \) is \( L \)-Lipschitz continuous in \( x \). For $M\in\mathbb{N}$, let $\kappa:=(b-a)/M$ be the time step. Consider Euler's method: for $m\in[M]$, 
    \begin{align*}
        x^{[m]}&:=x^{[m-1]}+\kappa f(x^{[m-1]},t_{m-1}),\\
        t_m&:=t_{m-1}+\kappa,
    \end{align*}
    with initial value $t_0:=a$ and $x^{[0]}:=x_0$. Then for all $m\in[M]$, 
    \begin{equation*}
    \norm{x(t_m) - x^{[m]}} \leq \frac{\kappa B}{2L} \parens{e^{L(t_m - a)} - 1}.
    \end{equation*}

\end{lemma}

\end{document}